\documentclass{article} 
\usepackage{hyperref}
\usepackage{url}
\usepackage{graphicx,amsmath,amssymb,algorithm,algorithmic,psfrag,graphics,bm,comment}
\usepackage{graphicx} 
\usepackage{subfigure}
\usepackage{amsthm,array}\usepackage{mathrsfs}\usepackage{comment}
\usepackage{cleveref}

\usepackage{cite}
\usepackage{color}
\usepackage[margin=1in]{geometry}
\newcommand{\bX}{\text{\boldmath{$X$}}}

\newcommand{\bw}{\text{\boldmath{$w$}}}

\newcommand{\bB}{\text{\boldmath{$B$}}}

\newcommand{\bz}{\boldsymbol{z}}
\newcommand{\bx}{\boldsymbol{x}}

\newcommand{\ba}{\boldsymbol{a}}
\newcommand{\bb}{\boldsymbol{b}}

\newcommand{\by}{\boldsymbol{y}}
\newcommand{\bY}{\boldsymbol{Y}}

\newcommand{\bU}{\boldsymbol{U}}
\newcommand{\bzero}{\boldsymbol{0}}
\newcommand{\bone}{\boldsymbol{1}}

\newcommand{\bI}{\boldsymbol{I}}
\newcommand{\cE}{\mathcal{E}}
\newcommand{\cD}{\mathcal{D}}
\newcommand{\bh}{\boldsymbol{h}}
\newcommand{\bq}{\boldsymbol{q}}
\newcommand{\bA}{\boldsymbol{A}}
\newcommand{\bM}{\boldsymbol{M}}
\newcommand{\bbP}{\mathbb{P}}
\newcommand{\bbR}{\mathbb{R}}

\newcommand{\bv}{\boldsymbol{v}}

\newcommand{\median}{\mathsf{med}}
\newcommand{\sgn}{\text{sgn}}

\newcommand{\erf}{\text{erf}}
\newcommand{\dist}{\text{dist}}

\newtheorem{definition}{\textbf{Definition}}
\newtheorem{corollary}{\textbf{Corollary}}
\newtheorem{lemma}{\textbf{Lemma}}
\newtheorem{theorem}{\textbf{Theorem}}
\newtheorem{proposition}{\textbf{Proposition}}
\newtheorem{remark}{\textbf{Remark}}

\newcommand{\nn}{\nonumber}

\newcommand{\mE}{\mathbb{E}}
\newcommand{\Var}{\mathsf{Var}}

\newcommand{\cU}{\mathcal{U}}
\newcommand{\cR}{\mathcal{R}}

\newcommand{\cN}{\mathcal{N}}
\newcommand{\cB}{\mathcal{B}}
\newcommand{\cP}{\mathcal{P}}

\newcommand{\cA}{\mathcal{A}}
\newcommand{\cT}{\mathcal{T}}

\newcommand{\cO}{\mathcal{O}}

\DeclareMathAlphabet{\matheuf}{U}{euf}{m}{n}

\title{Median-Truncated Nonconvex Approach \\ for Phase Retrieval with Outliers\footnote{The work of H. Zhang and Y. Liang is supported in part by the grants AFOSR FA9550-16-1-0077 and NSF ECCS 16-09916. The work of Y. Chi is supported in part by the grants NSF ECCS-1650449, AFOSR FA9550-15-1-0205 and ONR N00014-15-1-2387.}
\footnote{The material in this paper was presented in part at the International Conference of Machine Learning (ICML), New York, USA, 2016.}}

\author{
Huishuai Zhang$^s$, Yuejie Chi$^o$ and Yingbin Liang$^s$\\
$^s$Department of EECS, Syracuse University, Syracuse, NY 13244 \\
$^o$Department of ECE, Ohio State University, Columbus, OH 43210}
\begin{document}

\maketitle

\vskip 0.2in

\begin{abstract}
This paper investigates the phase retrieval problem, which aims to recover a signal from the magnitudes of its linear measurements. We develop statistically and computationally efficient algorithms for the situation when the measurements are corrupted by sparse outliers that can take arbitrary values. We propose a novel approach to robustify the gradient descent algorithm by using the sample median as a guide for pruning spurious samples in initialization and local search. Adopting the Poisson loss and the reshaped quadratic loss respectively, we obtain two algorithms termed \emph{median-TWF} and \emph{median-RWF}, both of which provably recover the signal from a near-optimal number of measurements when the measurement vectors are composed of i.i.d. Gaussian entries, up to a logarithmic factor, even when a constant fraction of the measurements are adversarially corrupted. We further show that both algorithms are stable in the presence of additional dense bounded noise. Our analysis is accomplished by developing non-trivial concentration results of median-related quantities, which may be of independent interest. We provide numerical experiments to demonstrate the effectiveness of our approach.
\end{abstract}

\section{Introduction}\label{Introduction}

Phase retrieval is a classical problem in signal processing, optics and machine learning that has a wide range of applications such as X-ray crystallography \cite{drenth2007x}, astronomical imaging, and diffraction imaging. Mathematically, it is formulated as recovering a signal $\bx\in\mathbb{R}^n/\mathbb{C}^n$ from the magnitudes of its linear measurements:
\begin{flalign}
y_i = |\langle \ba_i, \bx\rangle|^2, \quad i=1,\ldots, m, \label{eq:problem}
\end{flalign}
where $m$ is the total number of measurements, and $\ba_i\in \mathbb{R}^n/\mathbb{C}^n$ is the $i$th known measurement vector, $i=1, \ldots, m$. Phase retrieval is known to be notoriously difficult due to the quadratic form of the measurements. Classical methods based on alternating minimization between the signal of interest and the phase information \cite{fienup1982phase}, though computationally simple, are often trapped at local minima and lack rigorous performance guarantees. 

There has been, however, a recent line of work that successfully develops provably accurate algorithms for phase retrieval, in particular for the case when the measurement vectors $\ba_i$'s are composed of \emph{independent and identically distributed} (i.i.d.) Gaussian entries. 
Broadly speaking, two classes of approaches have been proposed based on convex and nonconvex optimization techniques, respectively. Using the lifting trick, the phase retrieval problem can be reformulated as estimating a rank-one positive semidefinite (PSD) matrix $\bX=\bx\bx^T$ from linear measurements \cite{balan2006signal}, for which convex relaxations into semidefinite programs have been studied \cite{waldspurger2015phase,candes2013phaselift,chen2015exact,demanet2012stable,candes2012solving,li2012sparse}. In particular, Phaselift \cite{candes2013phaselift} perfectly recovers the signal with high probability as long as the number of measurements $m$ is on the order of $n$. However, the computational complexity of Phaselift is at least cubic in $n$, which becomes expensive when $n$ is large. Very recently, another convex relaxation has been proposed in the natural parameter space without lifting \cite{goldstein2016phasemax,bahmani2016phase,hand2016elementary}, resulting in a linear program that can handle large problem dimensions as long as $m$ is on the order of $n$.

Another class of approaches aims to find the signal that minimizes a loss function based on certain postulated noise model, which often results in a nonconvex optimization problem due to the quadratic measurements. Despite nonconvexity, it is demonstrated in \cite{candes2015phase,soltanolkotabi2014algorithms} that the so-called Wirtinger flow (WF) algorithm, based on gradient descent, works remarkably well: it converges to the global optima when properly initialized using the spectral method. Several variants of WF have been proposed thereafter to further improve its performance, including the truncated Wirtinger flow (TWF) algorithm \cite{chen2015solving}, 
the reshaped Wirtinger flow (RWF) algorithm \cite{zhang2016reshaped}, 
and the truncated amplitude flow (TAF) algorithm \cite{wang2016solving} 
Notably, TWF, RWF and TAF are shown to converge globally at a linear rate as long as $m$ is on the order of $n$, and attain $\epsilon$-accuracy within $\cO(mn\log(1/\epsilon))$ flops using a constant step size.\footnote{Notation $f(n) = \cO(g(n))$ or $f(n)\lesssim g(n)$ means that there exists a  constant $c>0$ such that $|f(n)|\le c|g(n)|$.} 

\subsection{Outlier-Robust Phase Retrieval}
The aforementioned algorithms are evaluated based on their {\em statistical} and {\em computational} performances: statistically, we wish the sample complexity $m$ to be as small as possible; computationally, we wish the run time to be as fast as possible. As can be seen, existing WF-type algorithms are already near-optimal both statistically and computationally. This paper introduces a third consideration, which is the {\em robustness to outliers}, where we wish the algorithm continues to work well even in the presence of outliers that may take arbitrary magnitudes. This bears great importance in practice, because outliers arise frequently from the phase imaging applications \cite{weller2015undersampled} due to various reasons such as detector failures, recording errors, and missing data. Specifically, suppose the set of $m$ measurements are given as
\begin{flalign}\label{eq:outliermodel}
	y_i=\left|\langle \ba_i,\bx \rangle\right|^2+\eta_i, \quad i=1,\cdots,m,
\end{flalign}
where $\eta_i \in \mathbb{R}/\mathbb{C}$ for $i=1,\ldots,m$ are outliers that can take arbitrary values. We assume that outliers are sparse with no more than $s m$ nonzero values, i.e., $\|\boldsymbol{\eta}\|_0\le s m$, where $\boldsymbol{\eta}=\{\eta_i\}_{i=1}^m\in\mathbb{R}^m/\mathbb{C}^m$. Here, $s$ is a nonzero constant, representing the faction of measurements that are corrupted by outliers.

The goal of this paper is to develop phase retrieval algorithms with both statistical and computational efficiency, and provable robustness to even a constant proportion of outliers. None of the existing algorithms meet all of the three considerations simultaneously. The performance of WF-type algorithms is very sensitive to outliers which introduce anomalous search directions when their values are excessively deviated. While a form of Phaselift \cite{hand2015phaselift} is robust to a constant portion of outliers, it is computationally too expensive.



\subsection{Median-Truncated Gradient Descent}

A natural idea is to recover the signal as a solution to the following loss minimization problem:
\begin{flalign}\label{eq:loss}
  \min_{\bz}  \frac{1}{2m} \sum_{i=1}^m\ell(\bz;y_i)
\end{flalign}
where $\ell(\bz,y_i)$ is postulated using the negative likelihood of Gaussian or Poisson noise model. Since the measurements are quadratic in $\bx$, the objective function is nonconvex. We consider two choices of $\ell(\bz;y_i)$ in this paper. The first one is the Poisson loss function of $|\ba_i^T\bz|^2$ employed in TWF \cite{chen2015solving}, which is given by 
\begin{flalign}
\ell(\bz; y_i)=  |\ba_i^T\bz|^2-y_i\log |\ba_i^T\bz|^2 . \label{eq:TWFloss_sample}
\end{flalign}
The second one is the {\em reshaped}\footnote{It is called ``reshaped'' in order to distinguish it from the quadratic loss of $|\ba_i^T\bz|^2$ used in \cite{candes2015phase}.} quadratic loss of $|\ba_i^T\bz|$ employed in RWF \cite{zhang2016reshaped}, which is given by
\begin{flalign}
\ell(\bz; y_i) =  \left(|\ba_i^T\bz|-\sqrt{y_i}\right)^2. \label{eq:rushloss_sample}
\end{flalign}
Though \eqref{eq:rushloss_sample} is not smooth everywhere, it resembles more closely the least-squares loss as if the phase information is available than the quadratic loss of $|\ba_i^T\bz|^2$, resulting in a more amenable curvature. 

In the presence of outliers, the signal of interest may no longer be the global optima of \eqref{eq:loss}. Therefore, we wish to only include the clean samples that are not corrupted in the optimization \eqref{eq:loss}, which is, however, impossible as we do not assume any {\em a priori} knowledge of the outliers. Our key strategy is to prune the bad samples {adaptively and iteratively}, using a gradient descent procedure that proceeds as follows:
\begin{flalign}\label{eq:TWF_gradient_update}
\bz^{(t+1)}=\bz^{(t)}- \frac{\mu}{m}\sum_{i\in T_{t+1}} \nabla \ell(\bz^{(t)};y_i).
\end{flalign}
where $\bz^{(t)}$ denotes the $t$th iterate of the algorithm, $ \nabla \ell(\bz^{(t)};y_i)$ is the gradient of $ \ell(\bz^{(t)};y_i)$, and $\mu$ is the step size, for $t=0,1,\ldots$. In each iteration, only a subset $T_{t+1}$ of data-dependent and iteration-varying samples  contributes to the search direction. But how to select the set $T_{t+1}$? Note that the gradient of the loss function typically contains the term $\left|y_i-|\ba_i^T\bz^{(t)}|^2\right|$ (for TWF) or $\left|\sqrt{y_i}-|\ba_i^T\bz^{(t)} | \right|$ (for RWF), which measures the residual using the current iterate. With $y_i$ being corrupted by arbitrarily large outliers, the gradient can deviate the search direction from the signal arbitrarily. 

Inspired by the utility of {\em median} to combat outliers in robust statistics \cite{huber2011robust}, we prune samples whose gradient components $\nabla \ell(\bz^{(t)};y_i)$ are much larger than the {\em sample median} to control the search direction of each update. Hiding some technical details, this gives the main ingredient of our {\em median-truncated gradient descent} update rule\footnote{Please see the exact form of the algorithms in Section~\ref{sec:medianApproach}.}, i.e., for each iterate $t\geq 0$:
\begin{align}
T_{t+1} &:= \{i:  |y_i-|\ba_i^T\bz^{(t)}|^2| \lesssim \mathsf{med}(\{|y_i-|\ba_i^T\bz^{(t)}|^2\}_{i=1}^m) \}, \quad \mbox{for TWF}, \\
T_{t+1} &:= \{i:   | \sqrt{y_i}-|\ba_i^T\bz^{(t)} |  | \lesssim \mathsf{med}(\{|\sqrt{y_i}-|\ba_i^T\bz^{(t)} |\}_{i=1}^m) \}, \quad \mbox{for RWF}, 
\end{align}
where $\mathsf{med}(\cdot)$ denotes the sample median. The robust property of median
lies in the fact that the median cannot be arbitrarily perturbed unless the outliers dominate the inliers \cite{huber2011robust}. This is in sharp contrast to the sample mean, which can be made arbitrarily large even by a single outlier. Thus, using the sample median in the truncation rule can effectively remove the impact of outliers. Finally, there still left the question of initialization, which is critical to the success of the algorithm. We use the spectral method, i.e., initialize $\bz^{(0)}$ by a proper rescaling of the top eigenvector of a surrogate matrix
\begin{equation}\label{eq:TWF_initialization}
\bY=\frac{1}{m}\sum_{i\in T_0 } y_i\ba_i \boldsymbol{a}_i^T ,
\end{equation}
where again $T_0$  includes only a subset of samples whose values are not excessively large compared with the sample median of the measurements, given as
\begin{equation}
T_0  =\{ i: y_i \lesssim  \mathsf{med}(\{ y_i \}_{i=1}^m)  \}.
\end{equation}
Putting things together (the update rule \eqref{eq:TWF_gradient_update} and the initialization \eqref{eq:TWF_initialization}), we obtain two new median-truncated gradient descent algorithms, median-TWF and median-RWF, based on applying the median truncation strategy for the loss functions used in TWF and RWF, respectively. The median-TWF and median-RWF algorithms do not assume a priori knowledge of the outliers, such as their existence or the number of outliers, and therefore can be used in an oblivious fashion. Importantly, we establish the following performance guarantees.

\vspace{0.05in}
\textbf{\em Main Result (informal):} For the Gaussian measurement model, with high probability, median-TWF and median-RWF recover all signal $\bx$ up to the global sign at a linear rate of convergence, even with a constant fraction of outliers, as long as the number of measurements $m$ is on the order of $n\log n$. Furthermore, the reconstruction is stable in the presence of additional bounded dense noise.
\vspace{0.05in}
 

Statistically, the sample complexity of both algorithms is near-optimal up to a logarithmic factor, and to reassure, they continue to work even when outliers are absent. Computationally, both algorithms converge linearly, requiring a mere computational cost of $\mathcal{O}(mn\log 1/\epsilon)$ to reach $\epsilon$-accuracy. More importantly, our algorithms now tolerate a constant fraction of arbitrary outliers, without sacrificing performance otherwise. To the best of our knowledge, this is the first application of the median to robustify high-dimensional statistical estimation in the presence of arbitrary outliers with rigorous non-asymptotic performance guarantees. 

To establish the performance guarantees, we first show that the initialization is close enough to the ground truth, and then that within the neighborhood of the ground truth, the gradients satisfy certain {\em Regularity Condition} \cite{candes2015phase,chen2015solving} that guarantees linear convergence of the descent rule, as long as the fraction of outliers is small enough and the sample complexity is large enough. As a nonlinear operator, the sample median is much more difficult to analyze than the sample mean, which is a linear operator and many existing concentration inequalities are readily applicable. Therefore, considerable technical efforts are devoted to develop novel non-asymptotic concentrations of the sample median, and various statistical properties of the sample median related quantities, which may be of independent interest. 

Finally, we note that while median-TWF and median-RWF share similar theoretical performance guarantees, their empirical performances vary under different scenarios, due to the use of different loss functions. Their theoretical analyses also have significant difference that worth separate treatments. While we only consider the loss functions used in TWF and RWF in this paper, we believe the median-truncation technique can be applied to gradient descent algorithms for solving other problems as well.

\subsection{Related Work}

Our work is closely related to the TWF algorithm \cite{chen2015solving}, which is also a truncated gradient descent algorithm for phase retrieval. However, the truncation rule in TWF is based on the sample mean, which is very sensitive to outliers. In \cite{li2017lowrank,hand2015phaselift,weller2015undersampled}, the problem of phase retrieval under outliers is investigated, but the proposed algorithms either lack performance guarantees or are computationally too expensive. 
 

The adoption of median in machine learning is not unfamiliar, for example, $K$-median clustering \cite{chen2006k} and resilient data aggregation for sensor networks \cite{wagner2004resilient}. Our work here further extends the applications of median to robustifying high-dimensional estimation problems with theoretical guarantees. Another popular approach in robust estimation is to use the trimmed mean \cite{huber2011robust}, which has found success in robustifying sparse regression \cite{chen2013robust}, subspace clustering \cite{NIPS2015_5707}, etc. However, using the trimmed mean requires knowledge of an upper bound on the number of outliers, whereas median does not require such information.

Developing non-convex algorithms with provable global convergence guarantees has attracted intensive research interest recently. A partial list of these studies include phase retrieval \cite{sujay2013phase, candes2015phase, chen2015solving, wang2016solving,sun2016geometric}, matrix completion \cite{keshavan2010matrix, jain2013low, sun2015guaranteed, hardt2014understanding, de2015global, zheng2016convergence, jin2016provable, ge2016matrix}, low-rank matrix recovery \cite{bhojanapalli2016global, chen2015fast, tu2015low,zheng2015convergent, park2016provable, wei2015guarantees,li2016nonconvex,li2016symmetry}, robust PCA \cite{netrapalli2014non,yi2016fast}, robust tensor decomposition \cite{anandkumar2015tensor}, dictionary learning \cite{arora2015simple,sun2015complete}, community detection \cite{bandeira2016low}, phase synchronization\cite{boumal2016nonconvex}, blind deconvolution \cite{lee2015blind, li2016rapid}, joint alignment \cite{chen2016projected}, etc. Our algorithm provides a new instance in this list that emphasizes robust high-dimensional signal estimation under minimal assumptions of outliers. 


\subsection{Paper Organization and Notations}
The rest of this paper is organized as follows. Section~\ref{sec:medianApproach} describes the proposed two algorithms, median-TWF and median-RWF, in details and their performance guarantees. Section~\ref{sec:numerical} presents numerical experiments. Section~\ref{sec:proof} provides the preliminaries and the proof road map. Section~\ref{sec:proof:median_twf} provides the proofs for median-TWF and Section~\ref{sec:proof:median_rwf} provides the proofs of median-RWF, respectively. Finally, we conclude in Section~\ref{sec:conclusion}. Supporting proofs are given in the Appendix.

We adopt the following notations in this paper. Given a set of numbers $\{y_i\}_{i=1}^m$, the sample median is denoted as $\mathsf{med}(\{y_i\}_{i=1}^m)$. The indicator function $\bone_{A}=1$ if the event $A$ holds, and $\bone_{A}=0$ otherwise. For a vector $\by$, $\|\by\|$ denotes the $l_2$ norm. For two matrices, $\bA\preceq \bB$ if $\bB-\bA$ is a positive semidefinite matrix.  

\section{Algorithms and Performance Guarantees}\label{sec:medianApproach}



We consider the following model for phase retrieval, where the measurements are corrupted by not only sparse arbitrary outliers but also dense bounded noise. Under such a model, the measurements are given as
\begin{flalign}
y_i=\left|\langle \ba_i,\bx \rangle\right|^2+w_i+\eta_i,  \quad i=1,\cdots,m, \label{eq:twonoisesmodel}
\end{flalign}
where $\bx \in \mathbb{R}^n$ is the unknown signal\footnote{We focus on real signals here, but our analysis can be extended to complex signals.}, $\ba_i \in \mathbb{R}^n$ is the $i$th measurement vector composed of $i.i.d.$ Gaussian entries distributed as $\cN(0,1)$, and $\eta_i \in \mathbb{R}$ for $i=1,\ldots,m$ are outliers with arbitrary values satisfying $\|\boldsymbol{\eta}\|_0\le sm$, where $s$ is the fraction of outliers, and $\bw=\{w_i\}_{i=1}^m$ is the bounded noise satisfying $\|\bw\|_\infty\le c \|\bx\|^2$ for some universal constant $c$.

It is straightforward that changing the sign of the signal does not affect the measurements. The goal is to recover the signal $\bx$, up to a global sign difference, from the measurements $\by=\{y_i\}_{i=1}^m$ and the measurement vectors $\{\ba_i\}_{i=1}^m$. To this end, we define the Euclidean distance between two vectors up to a global sign difference as the performance metric,
\begin{flalign}
\dist(\bz,\bx):=\min \{\|\bz + \bx\|, \|\bz-\bx\|\}. \label{eq:defdis}
\end{flalign}




We propose two median-truncated gradient descent algorithms, median-TWF in  Section~\ref{subsec:medianTWF} and median-RWF in Section~\ref{subsec:medianRWF}, based on different choices of the loss functions. This leads to applying the truncation based on the sample median of $\left\{\left|y_i-|\ba_i^T\bz|^2\right|\right\}_{i=1}^m$ in median-TWF, and the sample median of$\left\{\left|y_i-|\ba_i^T\bz|^2\right|\right\}_{i=1}^m$ in median-RWF. Section~\ref{subsec:guarantees} provides the theoretical performance guarantees of median-TWF and median-RWF, which turn out to be almost the same at the order level except the choice of constants. The empirical comparisons of median-TWF and median-RWF are demonstrated in Section~\ref{sec:numerical}. 



\subsection{Median-TWF Algorithm}\label{subsec:medianTWF}

In median-TWF, we adopt the Poisson loss function of $|\ba_i^T\bz|^2$ employed in TWF \cite{chen2015solving}, given as
\begin{flalign}\label{eq:Poissonloss}
	\ell(\bz):=\frac{1}{2m} \sum_{i=1}^m \left(|\ba_i^T\bz|^2-y_i\log|\ba_i^T\bz|^2\right).
\end{flalign}

\begin{algorithm}[t]
\caption{Median Truncated Wirtinger Flow (Median-TWF)}\label{alg:mtwf}

\textbf{Input}: $\by=\{y_i\}_{i=1}^m$, $\{\ba_i\}_{i=1}^m$;\\
\textbf{Parameters:} thresholds $\alpha_y$, $\alpha_h$, $\alpha_{l}$, and $\alpha_u$, stepsize $\mu$;

\textbf{Initialization}: Let $\bz^{(0)}=\lambda_0 \tilde{\bz}$, where $\lambda_0=\sqrt{\median(\by)/0.455}$ and $\tilde{\bz}$ is the leading eigenvector of
\begin{equation}\label{eq:init_medianTWF}
\bY:=\frac{1}{m}\sum_{i=1}^m y_i\ba_i \boldsymbol{a}_i^T \bone_{\{|y_i|\le \alpha_y^2 \lambda_0^2\}}.
\end{equation}

 \textbf{Gradient loop}: for $t=0:T-1$ do
  \begin{flalign}\label{eq:loop_medianTWF}
		\bz^{(t+1)}=\bz^{(t)}- \frac{\mu}{m}\sum_{i=1}^{m}\frac{|\ba_i^T\bz^{(t)}|^2-y_i}{\ba_i^T\bz^{(t)} }\ba_i \bone_{\mathcal{E}_1^i\cap \mathcal{E}_2^i},
\end{flalign}
where
\begin{flalign}
&\cE_1^i:=\left\{\alpha_l\|\bz^{(t)}\|\le|\ba^T_i \bz^{(t)}|\le \alpha_u \|\bz^{(t)}\|\right\}, \nonumber \\
& \cE_2^i:= \left\{|y_i-|\ba_i^T\bz^{(t)}|^2|\le \alpha_h K_t \frac{|\ba_i^T\bz^{(t)}|}{\|\bz^{(t)}\|}\right\} , \quad \text{and }\quad
K_t:=\median\Big(\{|y_i-|\ba^T_i\bz^{(t)}|^2|\}_{i=1}^m\Big). \nonumber
\end{flalign}

\textbf{Output} $\bz_T$.
\end{algorithm}

The median-TWF algorithm, as described in Algorithm~\ref{alg:mtwf}, gradually eliminates the  influence of outliers on the way of minimizing \eqref{eq:Poissonloss}. Specifically, it comprises an initialization step and a truncated gradient descent step. 

1. \textbf{Initialization}: As in \eqref{eq:init_medianTWF}, we initialize $\bz^{(0)}$ by the spectral method using a truncated set of samples, where the threshold is determined by $\mathsf{med}(\{y_i\}_{i=1}^m)$.
As will be shown in Section~\ref{sec:initialization}, as long as the fraction of outliers is not too large and the sample complexity is large enough, our initialization is guaranteed to be within a small neighborhood of the true signal.


2. \textbf{Gradient loop}: for each iteration $0\le t\leq T-1$, median-TWF uses an iteration-varying truncated gradient given as
\begin{flalign}\label{trimmed_loss}
\nabla \ell_{tr}(\bz^{(t)})=\frac{1}{m} \sum_{i=1}^{m}\frac{|\ba_i^T\bz^{(t)}|^2-y_i}{\ba_i^T\bz^{(t)} }\ba_i \bone_{\mathcal{E}_1^i\cap \mathcal{E}_2^i}.
\end{flalign}
In order to remove the contribution of corrupted samples, from the definition of the set $\cE_2^i$ (see Algorithm~\ref{alg:mtwf}), it is clear that samples are truncated if their measurement residuals evaluated using the current iterate are much larger than the sample median. Samples are also truncated according to the set $\cE_1^i$, which removes the contribution of samples outside some confidence interval to better control the search direction, since $\mathbb{E}[|\ba_i^T\bz|]=\sqrt{2/\pi}\|\bz\|$. The median-TWF algorithm closely resembles the TWF algorithm, except that the truncation is guided by the sample median, rather than the sample mean.

We set the step size in median-TWF to be a fixed small constant, i.e., $\mu=0.4$. The rest of the parameters $\{\alpha_y, \alpha_h, \alpha_{l}, \alpha_u\}$ are set to satisfy
\begin{flalign}
& \zeta_1 :=\max \Big\{\mE\left[\xi^2 \bone_{\left\{|\xi|<\sqrt{1.01}\alpha_l\text{ or } |\xi|> \sqrt{0.99}\alpha_u\right\}}\right], \mE\left[\bone_{\left\{|\xi|<\sqrt{1.01}\alpha_l\text{ or }|\xi|> \sqrt{0.99}\alpha_u\right\}}\right]\Big\},\nn\\
&\zeta_2:=\mE\left[\xi^2\bone_{\{|\xi|>0.248\alpha_h}\}\right],\label{eq:parameters}\\
& 2(\zeta_1+\zeta_2)+\sqrt{8/\pi}\alpha_h^{-1}<1.99\nn\\
& \alpha_y\ge 3, \nn
\end{flalign}
where $\xi\sim \cN(0,1)$. For example, we can set $\alpha_l=0.3, \alpha_u=5, \alpha_y=3$ and  $\alpha_h=12$, and consequently $\zeta_1\approx 0.24$ and $\zeta_2\approx 0.032$.



\begin{algorithm}[t]
	\caption{Median Reshaped Wirtinger Flow (median-RWF)}\label{alg:RWF_mtwf}
	
	\textbf{Input}: $\by=\{y_i\}_{i=1}^m$, $\{\ba_i\}_{i=1}^m$;\\
	\textbf{Parameters:} threshold $\alpha'_h$, and step size $\mu$;
	
	\textbf{Initialization}: Same as median-TWF (see Algorithm \ref{alg:mtwf}).
	
	\textbf{Gradient loop}: for $t=0:T-1$ do
	\begin{flalign}\label{eq:loop_medianRWF}
		\bz^{(t+1)}=\bz^{(t)}- \frac{\mu}{m}\sum_{i=1}^{m}\left(\ba_i^T\bz^{(t)}-\sqrt{y_i}\cdot\frac{\ba_i^T\bz^{(t)}}{|\ba_i^T\bz^{(t)}|}\right)\ba_i \bone_{\mathcal{T}^i},
	\end{flalign}
	where
	\begin{flalign}
		\cT^i:= \left\{\left|\sqrt{y_i}-|\ba_i^T\bz^{(t)}|\right|\le \alpha'_h M_t\right\},\quad \text{and }\quad	M_t:=\median\left(\left\{\left|\sqrt{y_i}-|\ba^T_i\bz^{(t)}|\right|\right\}_{i=1}^m\right). \nonumber
	\end{flalign}
	\textbf{Output} $\bz_T$.
\end{algorithm}

\subsection{Median-RWF Algorithm}\label{subsec:medianRWF}

In median-RWF, we adopt the reshaped quadratic loss function of $|\ba_i^T\bz|$ employed in RWF \cite{zhang2016reshaped}, given as
\begin{flalign}\label{eq:RWFloss}
	\cR(\bz)=\frac{1}{2m}\sum_{i=1}^{m}\left(\sqrt{y_i}-|\ba_i^T\bz|\right)^2,
\end{flalign} 
which has been shown to be advantageous over other loss functions for phase retrieval \cite{zhang2016reshaped}.   

Similarly to median-TWF, the median-RWF algorithm as described in Algorithm~\ref{alg:RWF_mtwf}, gradually eliminates the  influence of outliers on the way of minimizing \eqref{eq:RWFloss}. Specifically, it also comprises an initialization step and a truncated gradient descent step. 

1. \textbf{Initialization}: we initialize in the same manner as in median-TWF (Algorithm \ref{alg:mtwf}). 


2. \textbf{Gradient loop}: for each iteration $0\le t\leq T-1$, median-RWF uses the following iteration-varying truncated gradient:
\begin{flalign}\label{eq:RWFgradient}
	\nabla \cR_{tr}(\bz^{(t)})=\frac{1}{m}\sum_{i=1}^{m}\left(\ba_i^T\bz^{(t)}-\sqrt{y_i}\cdot\frac{\ba_i^T\bz^{(t)}}{|\ba_i^T\bz^{(t)}|}\right)\ba_i \bone_{\mathcal{T}^i},
\end{flalign}
From the definition of the set $\cT^i$ (see Algorithm~\ref{alg:RWF_mtwf}), samples are truncated by the sample median of gradient components evaluated at the current iteration. We set the step size in median-RWF to be a fixed small constant, i.e., $\mu=0.8$. Compared with median-TWF, the truncation rule is much simpler with fewer parameters. We simply set  the truncation threshold $\alpha'_h=5$. It is possible that including $\mathcal{E}_1^i$ may further improves the performance, but we wish to highlight that, in this paper, the simple truncation rule is already sufficient to guarantee both robustness and efficiency of median-RWF.

\subsection{Performance Guarantees}\label{subsec:guarantees}

In this section, we characterize the performance guarantees of median-TWF and median-RWF, which turn out to be very similar though the proofs in fact involve quite different techniques. To avoid repetition, we present the guarantees together for both algorithms. We note that the values of constants in the results can vary for median-TWF and median-RWF. 


We first show that median-TWF/median-RWF performs well for the noise-free model in the following proposition, which lends support to the model with outliers. This also justifies that we can run median-TWF/median-RWF without having to know whether the underlying measurements are corrupted.

\begin{proposition}[\textbf{Exact recovery for the noise-free model}]\label{prop:noisefree}
	Suppose that the measurements are noise-free, i.e., $\eta_i=0$ and $w_i=0$ for $i=1,\cdots,m$ in the model \eqref{eq:twonoisesmodel}.
	There exist constants $\mu_0>0$, $0<\rho,\nu<1$ and $c_0,c_1,c_2>0$ such that if $m\ge c_0 n \log n$ and $\mu \le \mu_0$, then with probability at least $1-c_1\exp (-c_2 m)$, median-TWF/median-RWF yields
	\begin{flalign}
		\dist (\bz^{(t)}, \bx)\le \nu(1-\rho)^t\|\bx\|,\quad \forall t\in \mathbb{N}
	\end{flalign}
	simultaneously for all $\bx\in\bbR^n\backslash\{\bzero\}$.
\end{proposition}

Proposition~\ref{prop:noisefree} suggests that median-TWF/median-RWF allows exact recovery at a linear rate of convergence as long as the sample complexity is on the order of $n\log n$, which is in fact slightly worse, by a logarithmic factor, than existing WF-type algorithms (TWF, RWF and TAF) for the noise-free model. This is a price due to working with the nonlinear operator of median in the proof, and it is not clear whether it is possible to further improve the result. Nonetheless, as the median is quite stable as long as the number of outliers is not so large, the following main theorem indeed establishes that median-TWF/median-RWF still performs well even in the presence of a constant fraction of sparse outliers with the same sample complexity.
\begin{theorem}[\textbf{Exact recovery with sparse arbitrary outliers}]\label{thm:outliers}
Suppose that the measurements are corrupted by sparse outliers, i.e., $w_i=0$ for $i=1,\cdots,m$ in the model \eqref{eq:twonoisesmodel}. There exist constants $\mu_0$, $s_0>0$, $0<\rho,\nu<1$ and $c_0,c_1,c_2>0$ such that if $m\ge c_0 n\log n$, $s<s_0$, $\mu\le \mu_0$, then with probability at least $1-c_1\exp (-c_2 m)$, median-TWF/median-RWF yields
\begin{flalign}
\dist (\bz^{(t)}, \bx)\le \nu(1-\rho)^t\|\bx\|,\quad \forall t\in \mathbb{N}
\end{flalign}
simultaneously for all $\bx\in\bbR^n\backslash\{\bzero\}$.
\end{theorem}

Theorem~\ref{thm:outliers} indicates that median-TWF/median-RWF admits exact recovery for {\em all} signals in the presence of sparse outliers with arbitrary magnitudes even when the number of outliers scales linearly with the number of measurements, as long as the sample complexity satisfies $m \gtrsim n\log n$. Moreover, median-TWF/median-RWF converges at a linear rate using a constant step size, with per-iteration cost $\mathcal{O}(mn)$ (note that the median can be computed in linear time \cite{tibshirani2008fast}). To reach $\epsilon$-accuracy, i.e., $\dist (\bz^{(t)}, \bx) \leq \epsilon$, only $\mathcal{O}(\log 1/\epsilon)$ iterations are needed, yielding the total computational cost as $\mathcal{O}(mn\log 1/\epsilon)$, which is highly efficient. Empirically in the numerical experiments in Section~\ref{sec:numerical}, median-RWF converges faster and tolerates a larger fraction of outliers than median-TWF, which can be due to the use of the reshaped quadratic loss function.


We next consider the model when the measurements are corrupted by both sparse arbitrary outliers and dense bounded noise. Our following theorem characterizes that median-TWF/median-RWF is stable to coexistence of the two types of noises.
\begin{theorem}[\textbf{Stability to sparse arbitrary outliers and dense bounded noises}]\label{thm:twonoises}
Consider the phase retrieval problem given in \eqref{eq:twonoisesmodel} in which measurements are corrupted by both sparse arbitrary and dense bounded noises. There exist constants $\mu_0,s_0 >0$, $0<\rho<1$ and $c_0,c_1,c_2>0$ such that if $m \ge c_0 n\log n$, $s<s_0$, $\mu\le \mu_0$, then with probability at least $1-c_1\exp (-c_2 m)$, median-TWF and median-RWF respectively yield
\begin{flalign}
&\dist (\bz^{(t)}, \bx)\lesssim \frac{\|\bw\|_\infty}{\|\bx\|}+(1-\rho)^t\|\bx\|,\quad \forall t\in \mathbb{N}~\quad\mbox{for median-TWF},\\
&\dist (\bz^{(t)}, \bx)\lesssim \sqrt{\|\bw\|_\infty}+(1-\rho)^t\|\bx\|,\quad \forall t\in \mathbb{N}~\quad\mbox{for median-RWF},
\end{flalign}
simultaneously for all $\bx\in\bbR^n\backslash\{\bzero\}$.
\end{theorem}
Theorem \ref{thm:twonoises} immediately implies the stability of median-TWF/median-RWF when the measurements are only corrupted by dense bounded noise.
\begin{corollary}\label{cor:boundednoise}
Consider the phase retrieval problem in which measurements are corrupted only by dense bounded noises, i.e., $\eta_i=0$  for $i=1,\cdots,m$ in the model \eqref{eq:twonoisesmodel}. There exist constants $\mu_0>0$, $0<\rho<1$ and $c_0,c_1,c_2>0$ such that if $m\ge c_0 n\log n$, $\mu\le \mu_0$, then with probability at least $1-c_1\exp (-c_2 m)$, median-TWF and median-RWF respectively yield
\begin{flalign}
&\dist (\bz^{(t)}, \bx)\lesssim \frac{\|\bw\|_\infty}{\|\bx\|}+(1-\rho)^t\|\bx\|,\quad \forall t\in \mathbb{N}~\quad\mbox{for median-TWF},\\
&\dist (\bz^{(t)}, \bx)\lesssim \sqrt{\|\bw\|_\infty}+(1-\rho)^t\|\bx\|,\quad \forall t\in \mathbb{N}~\quad\mbox{for median-RWF},
\end{flalign}
simultaneously for all $\bx\in\bbR^n\backslash\{\bzero\}$.
\end{corollary}

With both sparse arbitrary outliers and dense bounded noises, Theorem \ref{thm:twonoises} and Corollary \ref{cor:boundednoise} imply that median-TWF/median-RWF achieves the same convergence rate and the same level of estimation error as the model with only bounded noise. In fact, together with Theorem \ref{thm:outliers} and Proposition \ref{prop:noisefree}, it can be seen that applying median-TWF/median-RWF does not require the knowledge of the existence of outliers. When there do exist outliers, median-TWF/median-RWF achieves almost the same performance \emph{as if outliers do not exist}. 

\section{Numerical Experiments}\label{sec:numerical}
In this section, we provide numerical experiments to demonstrate the effectiveness of median-TWF and median-RWF, which corroborate our theoretical findings. 
\subsection{Exact Recovery for Noise-free Data}
We first show that, in the noise-free case, median-TWF and median-RWF provide similar performance as TWF \cite{chen2015solving} and RWF \cite{zhang2016reshaped} for exact recovery. We set the parameters of median-TWF and median-RWF as specified in Section~\ref{subsec:medianTWF} and Section~\ref{subsec:medianRWF}, and those of TWF and RWF as suggested in \cite{chen2015solving} and \cite{zhang2016reshaped}, respectively. Let the signal length $n$ take values from $1000$ to $10000$ by a step size of $1000$, and the ratio of the number of measurements to the signal dimension, $m/n$, take values from $2$ to $6$ by a step size of $0.1$. For each pair of $(n,m/n)$, we generate a signal $\bx\sim \cN(\bzero,\bI_{n\times n})$, and the measurement vectors $\ba_i\sim \cN(\bzero,\bI_{n\times n})$ i.i.d.\ for $i=1,\ldots, m$. For all algorithms, a fixed number of iterations $T=500$ are run, and the trial is declared successful if  $\bz^{(T)}$, the output of the algorithm, satisfies $\dist(\bz^{(T)},\bx)/\|\bx\|\le 10^{-8}$. Figure~\ref{fig:mnscale} shows the number of successful trials out of 20 trials for all algorithms, with respect to $m/n$ and $n$. It can be seen that, as soon as $m$ is above $4n$, exact recovery is achieved for all four algorithms. Around the phase transition boundary, the empirical sample complexity of median-TWF is slightly worse than that of TWF, which is possibly due to the inefficiency of median compared to mean in the noise-free case \cite{huber2011robust}. Interestingly, the empirical sample complexity of median-RWF is slightly better than RWF because the truncation rule used in median-RWF allows sample pruning that improves the performance.\footnote{The original RWF in \cite{zhang2016reshaped} does not have sample truncation.}

\begin{figure}[th]
\centering 
\subfigure[median-TWF]{
\includegraphics[width=0.4\textwidth]{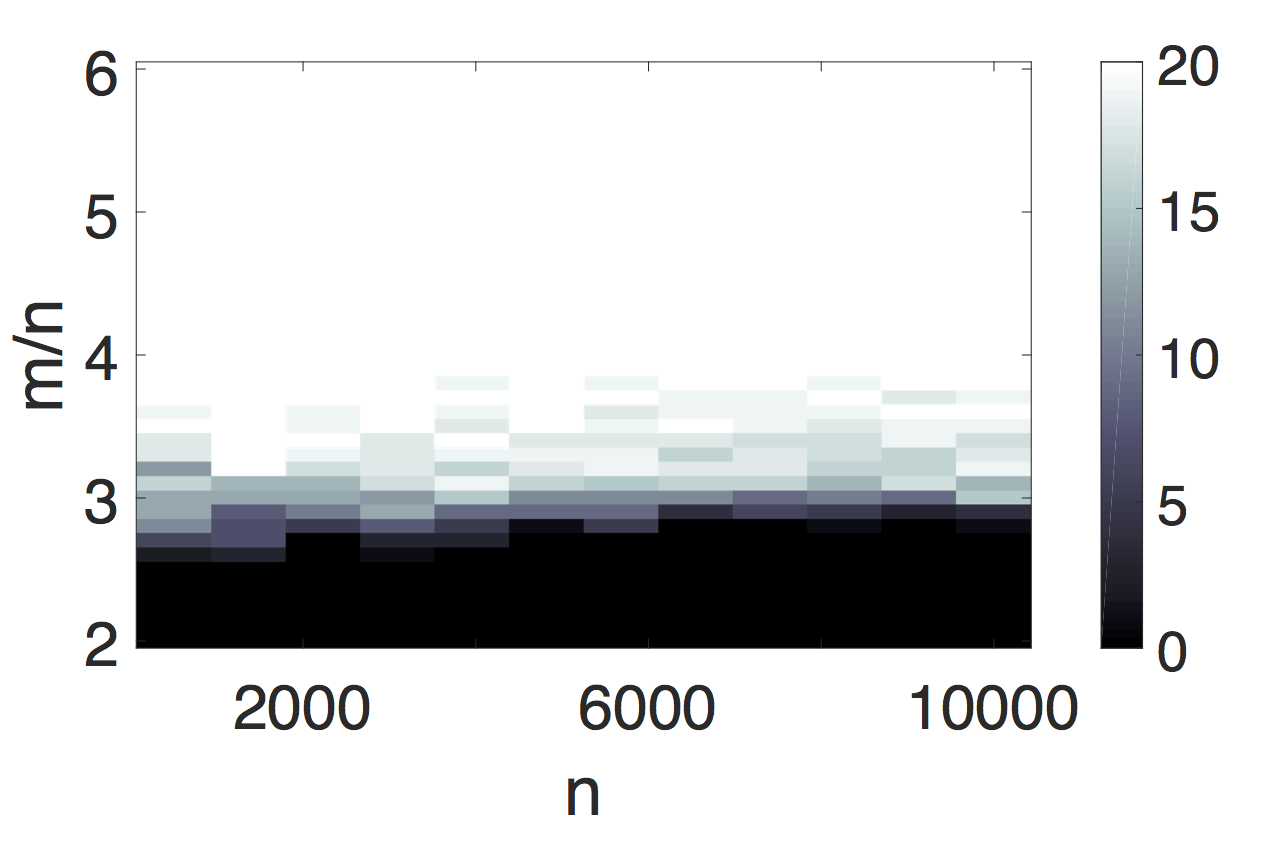}
\label{fig:mnscalemedian}}
\hfil
\subfigure[TWF]{
	\includegraphics[width=0.4\textwidth]{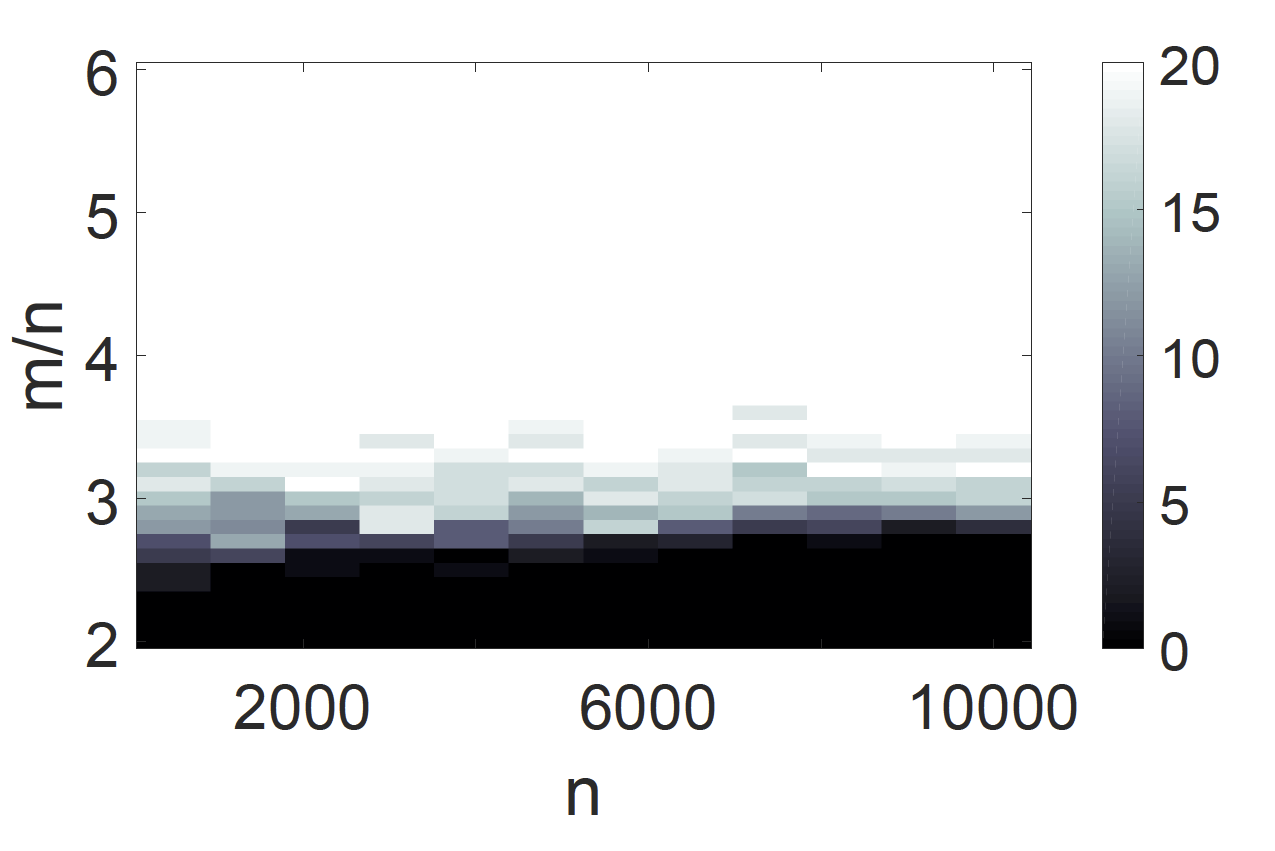}
	\label{fig:mnscaletwf}}
\subfigure[median-RWF]{
	\includegraphics[width=0.4\textwidth]{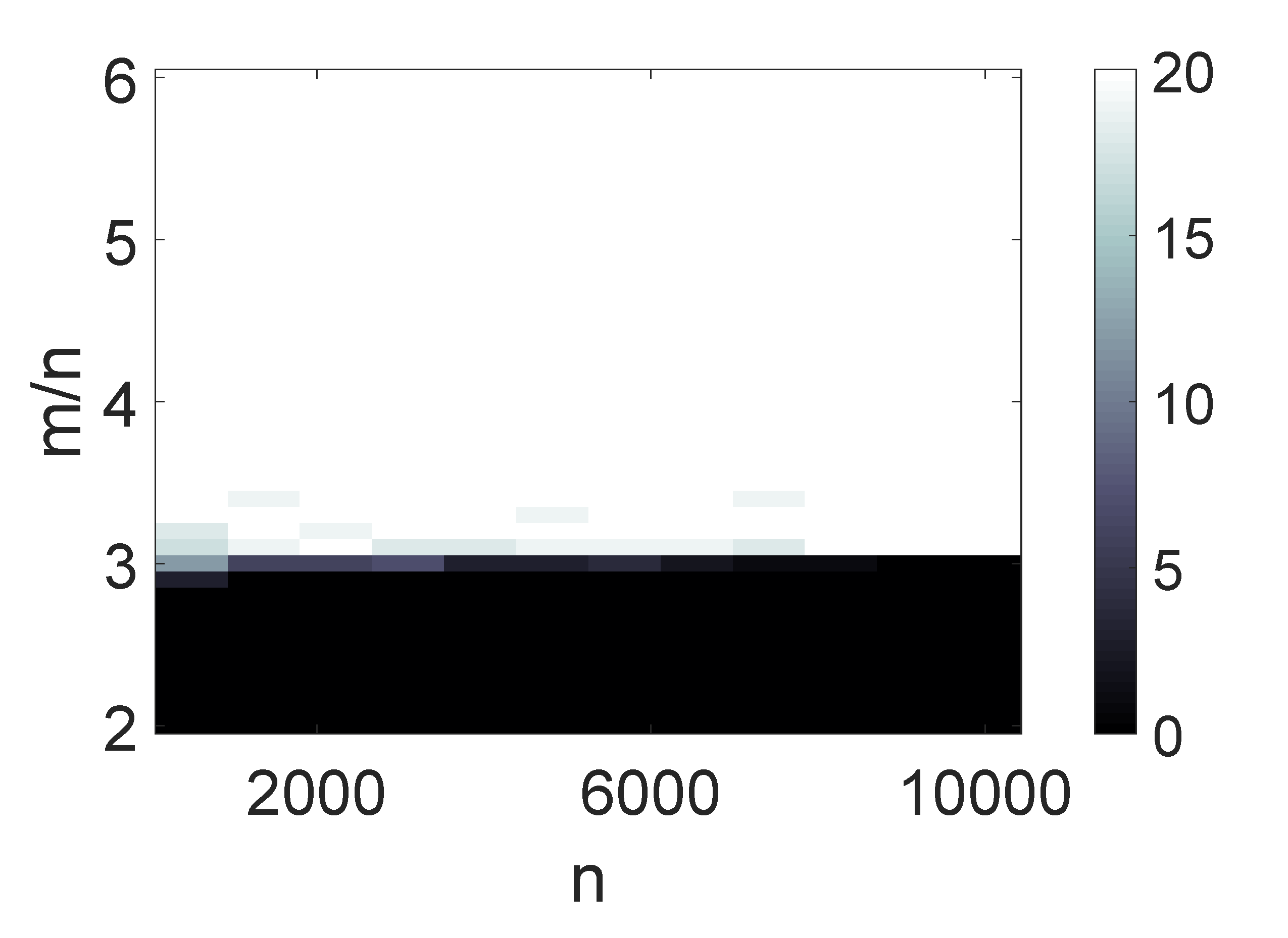}
	\label{fig:mnscalemedianRWF}}
\hfil
\subfigure[RWF]{
	\includegraphics[width=0.4\textwidth]{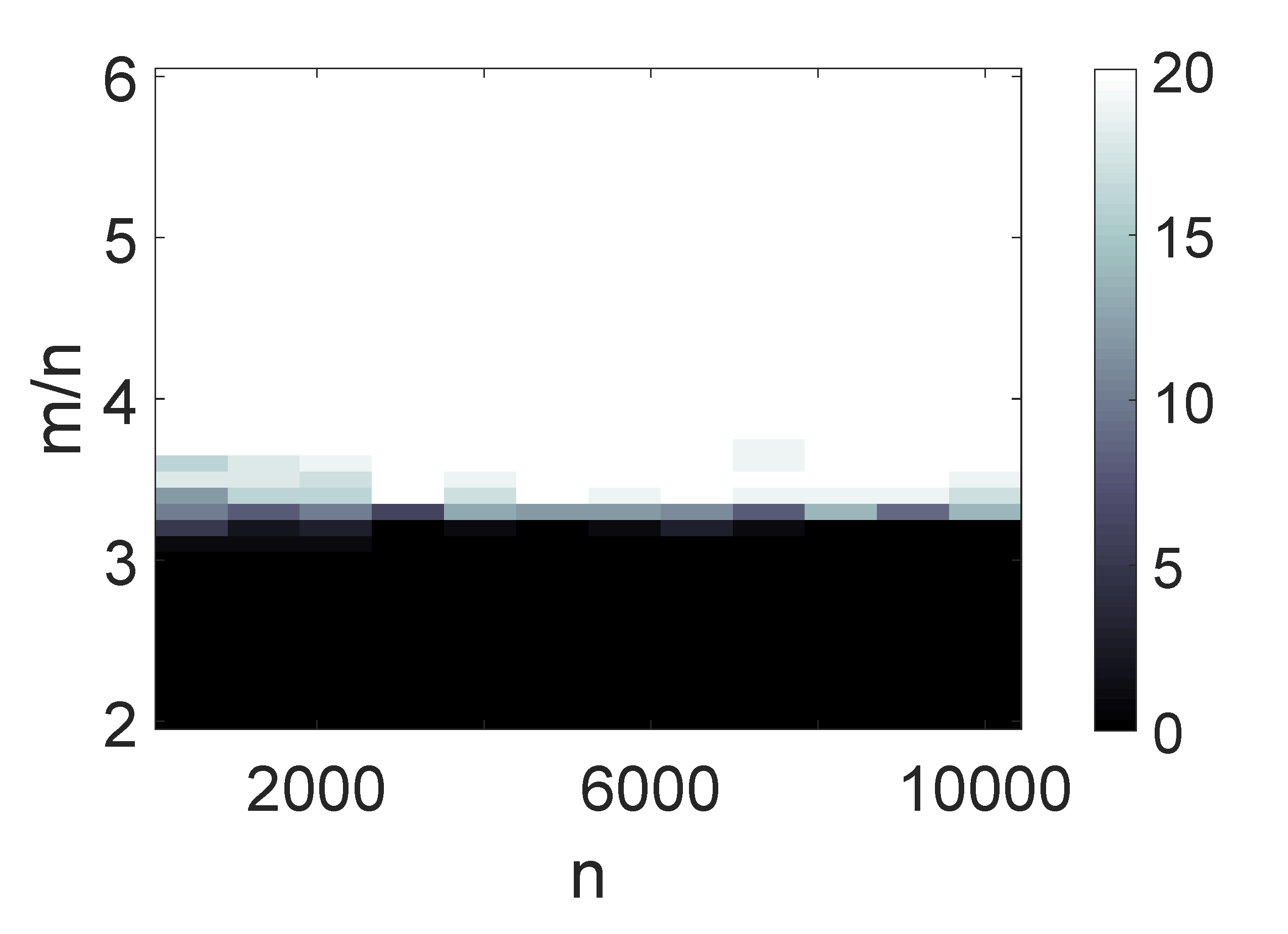}
	\label{fig:mnscaleRWF}}
\caption{Sample complexity of median-TWF, TWF, median-RWF, and RWF for noise-free data: the gray scale of each cell $(n,m/n)$ indicates the number of successful recovery out of $20$ trials. }
\label{fig:mnscale}
\end{figure}

\subsection{Exact Recovery with Sparse Outliers}
We next examine the performance of median-TWF and median-RWF in the presence of sparse outliers. We compare the performance of median-TWF and median-RWF with TWF, and also an alternative which we call {\em trimean-TWF}, based on replacing the sample mean in TWF by the {\em trimmed mean} for truncation purpose. Specifically, trimean-TWF requires knowing the fraction $s$ of outliers so that samples corresponding to $sm$ largest values in the measurements or gradients are first removed, and truncation is then applied based on the mean of the remaining samples similar to TWF.

We fix the signal length $n=1000$ and the number of measurements $m=8000$. Let each measurement $y_i$ be corrupted with probability $s\in [0,0.4]$ independently, where the corruption value $\eta_i \sim \cU(0, \eta_{\max})$ is randomly generated from a uniform distribution. Figure~\ref{fig:outliers} shows the success rate of exact recovery over $100$ trials as a function of $s$ at different levels of outlier magnitudes $\eta_{\max}/ \|\bx\|^2=0.1, 1, 10, 100$, for the four algorithms median-TWF, median-RWF, trimean-TWF and TWF.


\begin{figure*}[tb]
\centering
\subfigure[$\eta_{\max}=0.1\|\bx\|^2$]{
\includegraphics[width=0.45\textwidth]{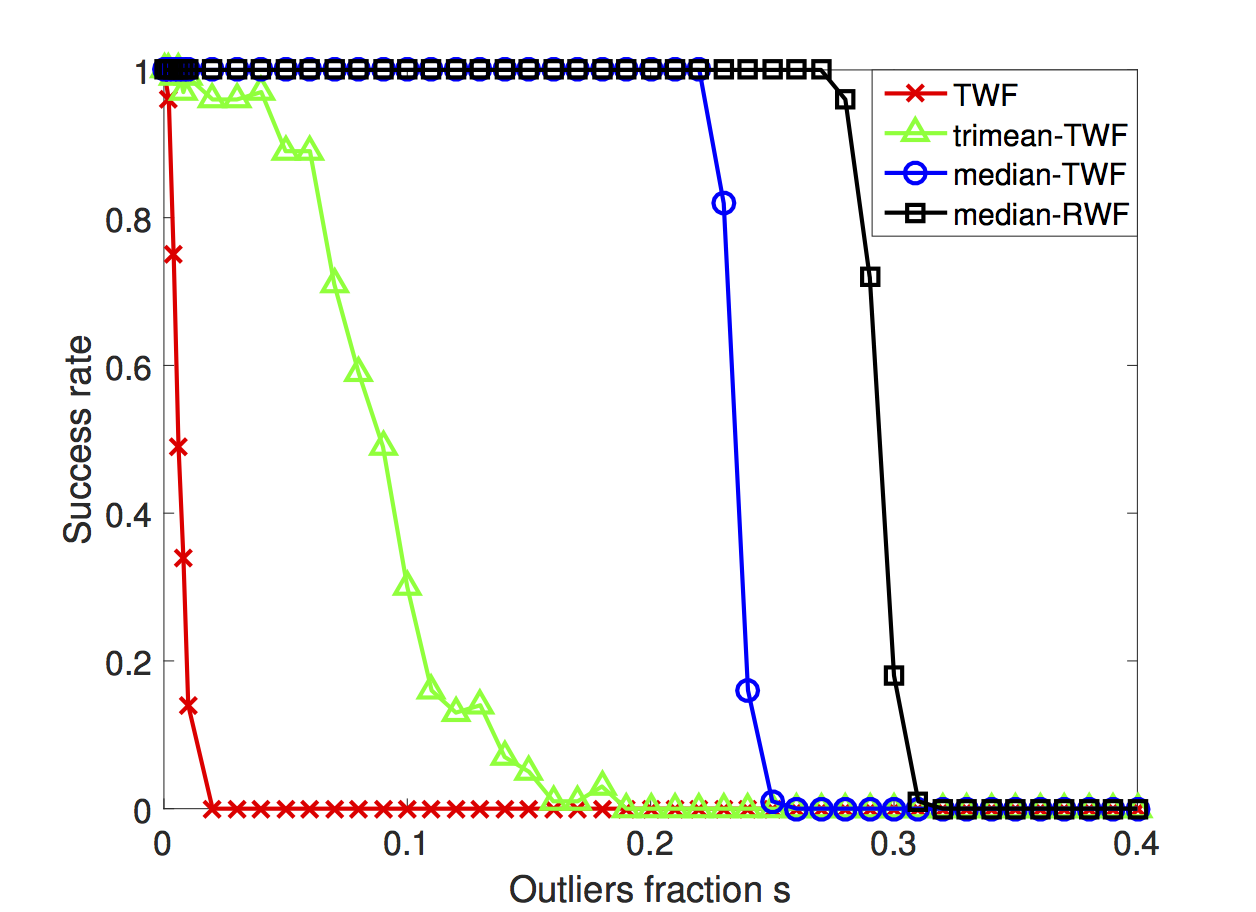}
\label{fig:outliers01}}
\hfil
\subfigure[$\eta_{\max}=\|\bx\|^2$]{
\includegraphics[width=0.45\textwidth]{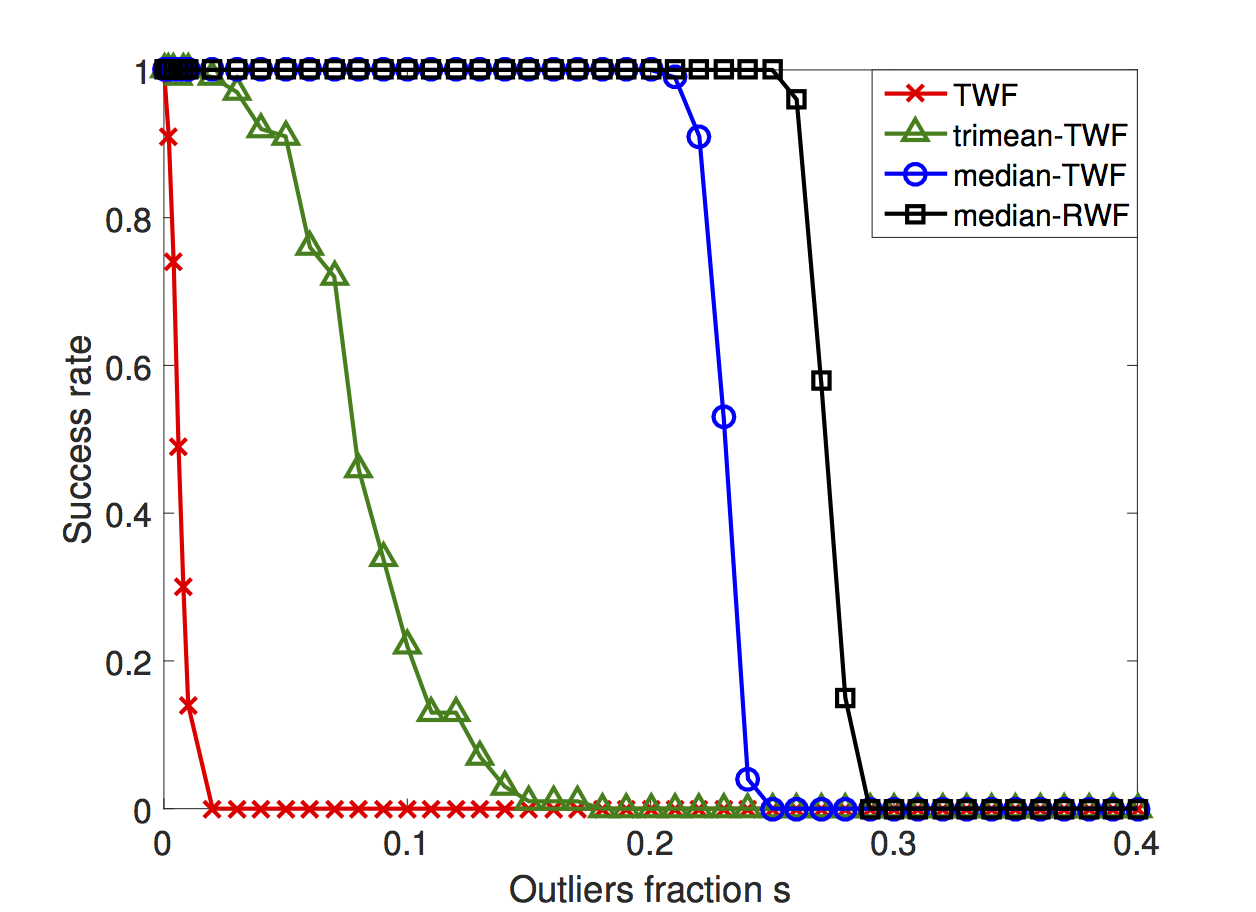}
\label{fig:outliers1}}\\
\subfigure[$\eta_{\max}=10\|\bx\|^2$]{
\includegraphics[width=0.45\textwidth]{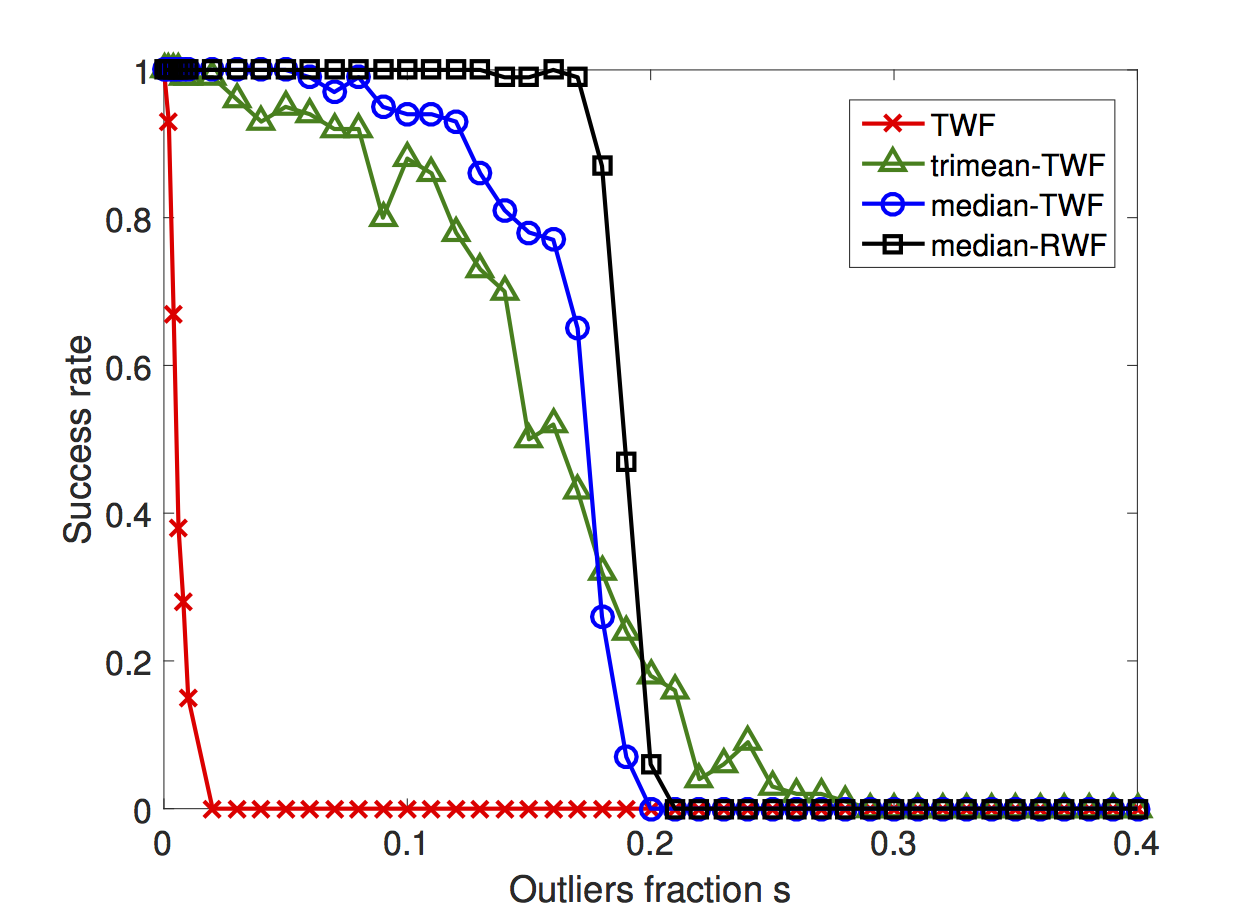}
\label{fig:outliers10}}
\hfil
\subfigure[$\eta_{\max}=100\|\bx\|^2$]{
\includegraphics[width=0.45\textwidth]{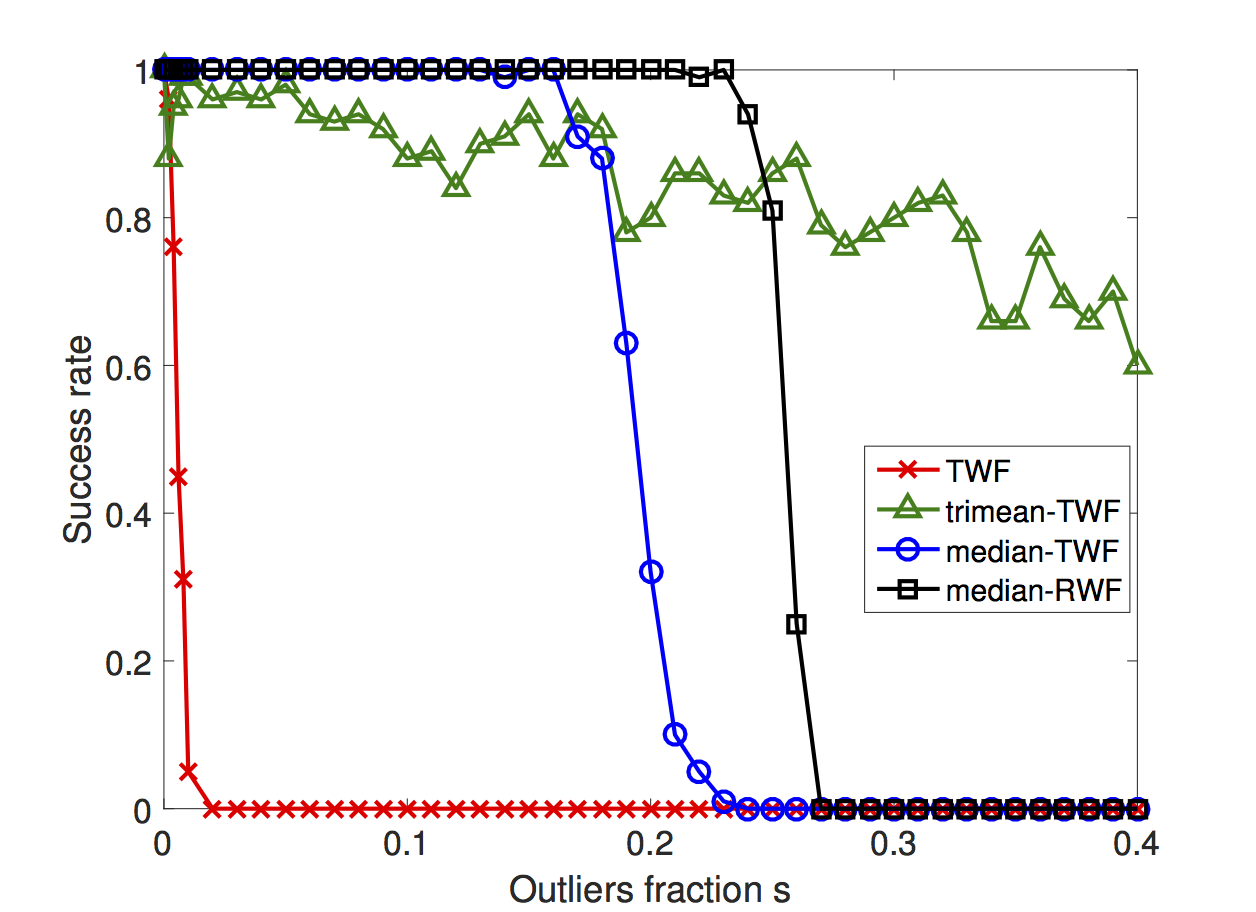}
\label{fig:outliers100}}
\caption{Success rate of exact recovery with respect to the fraction of sparse outliers for median-TWF, median-RWF, trimean-TWF, and TWF at different levels of outlier magnitudes.}
\label{fig:outliers}
\end{figure*}

From Figure~\ref{fig:outliers}, it can be seen that median-TWF and median-RWF allow exact recovery as long as $s$ is not too large for all levels of outlier magnitudes, without assuming any knowledge of the outliers, which validates our theoretical analysis. Empirically, median-RWF can tolerate a larger fraction of outliers than median-TWF. This could be due to the fact that the lower-order objective adopted in median-RWF reduces the variance and allows more stable search direction. Unsurprisingly, TWF fails quickly even with a very small fraction of outliers. No successful instance is observed for TWF when $s\geq 0.02$ irrespective of the value of $\eta_{\max}$. Trimean-TWF, even knowing the number of outliers, still does not exhibit a sharp phase transition, and in general underperforms the proposed median-TWF and median-RWF, except when both $\eta_{\max}$ and $s$ get very large, see Figure~\ref{fig:outliers100}. This is because in this range of large outliers, knowing the fraction $s$ of outliers provides substantial advantage for trimean-TWF to eliminate them, while the sample median can deviate significantly from the true median for large $s$. Moreover, it is worth mentioning that exact recovery is more challenging for median-TWF and median-RWF when the magnitudes of most outliers are comparable to the measurements, as in Figure~\ref{fig:outliers10}. In such a case, the outliers are more difficult to exclude as opposed to the case with very large outlier magnitudes as in Figure~\ref{fig:outliers100}; and meanwhile, the outlier magnitudes in Figure~\ref{fig:outliers10} are large enough to affect the accuracy heavily in contrast to the cases in Figures~\ref{fig:outliers01} and \ref{fig:outliers1} where outliers are less prominent.


 
 \subsection{Stable Recovery with Sparse Outliers and Dense Noise}
We now examine the performance of median-TWF and median-RWF in the presence of both sparse outliers and dense bounded noise. The entries of the dense bounded noise term $\bw$ are generated independently from $\cU(0, w_{\max})$. The entries of the sparse outlier are then generated as $\eta_i\sim \|\bw\|\cdot \mbox{Bernoulli}(0.1)$ independently. Figure~\ref{fig:twonoises1} and Figure~\ref{fig:twonoises2} depict the relative error $\dist(\bz^{(t)},\bx)/\|\bx\|$ with respect to the iteration count $t$, when $w_{\max} /\|\bx\|^2=0.001$ and $0.01$ respectively. In the presence of sparse outliers, it can be seen that both median-TWF and median-RWF clearly outperforms TWF under the same situation, and acts as if the outliers do not exist by achieving almost the same accuracy as TWF without outliers. Moreover, the relative error of the reconstruction using median-TWF or median-RWF has 10 times gain from Figure \ref{fig:twonoises1} to Figure \ref{fig:twonoises2} as $w_{\max}$ shrinks by a factor of $10$, which corroborates Theorem \ref{thm:twonoises} nicely. Furthermore, it can be seen that median-RWF converges faster than the other algorithms, due to the improved curvature of using low-order objectives, corroborating the result in \cite{zhang2016reshaped}. On the other hand, median-TWF returns more accurate estimates, due to employing more delicate truncation rules that may help reduce the noise.

 
 \begin{figure*}[tb]
\begin{center}
\subfigure[$w_{\max}=0.01\|\bx\|^2$]{
\includegraphics[width=0.45\textwidth]{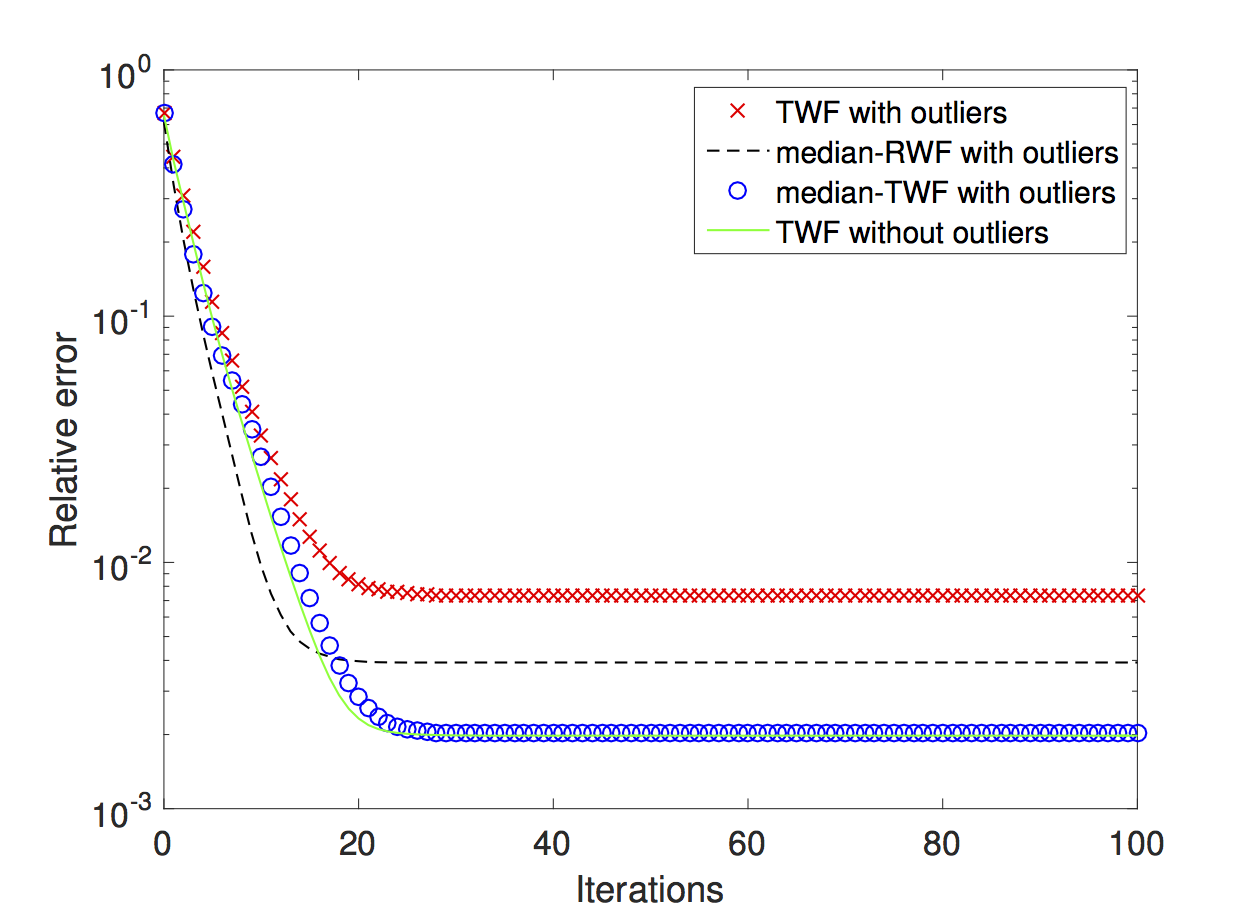}
\label{fig:twonoises1}}
\subfigure[$w_{\max}=0.001\|\bx\|^2$]{
\includegraphics[width=0.45\textwidth]{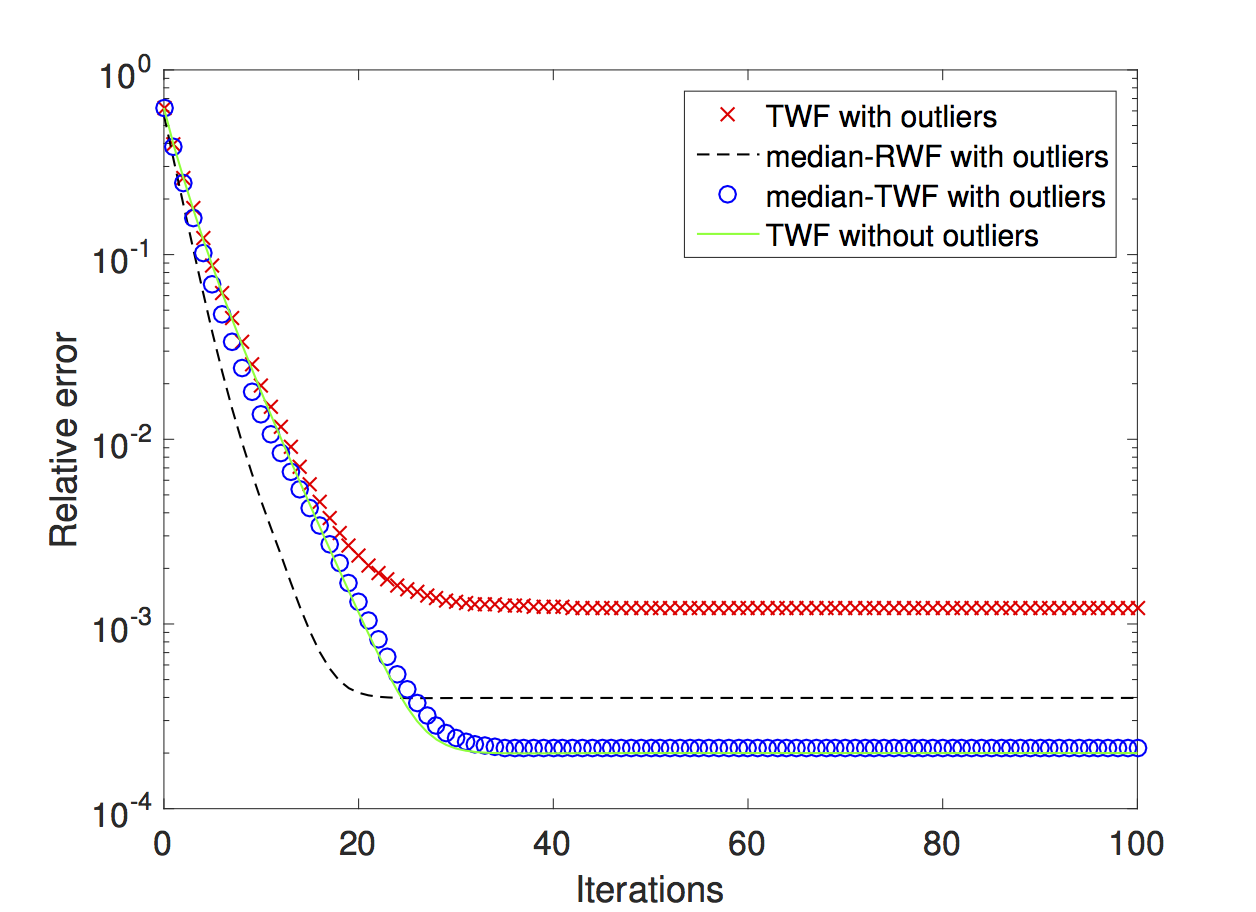}
\label{fig:twonoises2}}
\end{center}
\caption{The relative error with respect to the iteration count for median-TWF, median-RWF and TWF with both dense  noise and sparse outliers, and TWF with only dense noise. In (a) and (b), the dense noise is generated uniformly at different levels.}
\label{fig:twonoises}
\end{figure*}



Finally, we consider the case when the measurements are corrupted by both Poisson noise and outliers, modeling photon detection in optical imaging applications. We generate each measurement as $y_i\sim \text{Poisson}(|\langle\ba_i,\bx\rangle|^2)$, for $i=1,\cdots, m$, which is then corrupted with probability $s=0.1$ by outliers. The entries of the outlier are obtained by first generating $\eta_i\sim \|\bx\|^2 \cdot \cU(0, 1)$ independently, and then rounding it to the nearest integer. Figure \ref{fig:poisson} depicts the relative error $\dist(\bz^{(t)},\bx)/\|\bx\|$ with respect to the iteration count $t$, where median-TWF and median-RWF under both outliers and Poisson noise have almost the same accuracy as, if not better than, TWF under only the Poisson noise. 
\begin{figure}[htb]
\begin{center}
\includegraphics[width=0.5\textwidth]{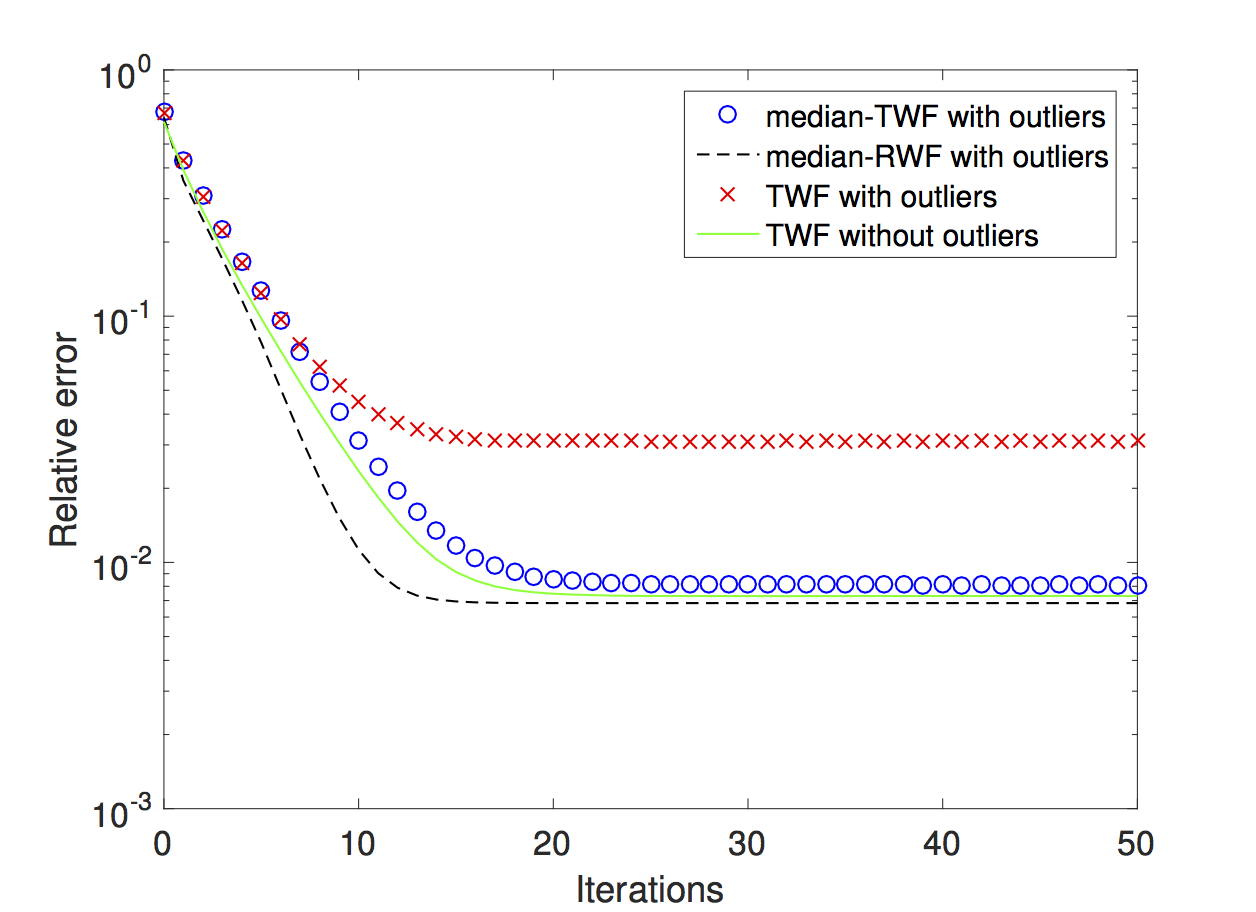}
\end{center}
\caption{The relative error with respect to the iteration count for median-TWF, median-RWF and TWF with both Poisson noise and sparse outliers, and TWF with only Poisson noise.}
\label{fig:poisson}
\end{figure}



\section{Preliminaries and Proof Roadmap}\label{sec:proof}

Broadly speaking, the proofs for median-TWF and median-RWF follow the same roadmap. The crux is to use the statistical properties of the median to show that the median-truncated gradients satisfy the so-called \emph{Regularity Condition} \cite{candes2015phase}, which guarantees the linear convergence of the update rule, provided the initialization provably lands in a small neighborhood of the true signal. 

We first develop a few statistical properties of median that will be useful throughout our analysis in Section~\ref{sec:medianprop}. Section~\ref{sec:initialization} analyzes the initialization that is used in both algorithms. We then state the definition of Regularity Condition in Section~\ref{sec:rc_convergence} and explain how it leads to the linear convergence rate. We provide separate detailed proofs for two algorithms in Section~\ref{sec:proof:median_twf} and Section~\ref{sec:proof:median_rwf}, respectively, because they involve different bounding techniques that may be of independent interest due to different loss functions. 

%
%

\subsection{Properties of Median}\label{sec:medianprop}
We start by the definitions of the quantile of a population distribution and its sample version.
\begin{definition}[Generalized quantile function]
	Let $0<p<1$. For a cumulative distribution function (CDF) $F$, the generalized quantile function is defined as
	\begin{flalign}
	F^{-1}(p)=\inf \{x\in \bbR: F(x)\ge p\}. \label{eq:quantiledef}
	\end{flalign}
	For simplicity, denote $\theta_p(F)=F^{-1}(p)$ as the $p$-\emph{quantile} of $F$. Moreover for a sample sequence $\{X_i\}_{i=1}^m$, the sample $p$-quantile $\theta_p(\{X_i\})$ means $\theta_p(\hat{F})$, where $\hat{F}$ is the empirical distribution of the samples $\{X_i\}_{i=1}^m$ .
\end{definition}

\begin{remark}
We note that the median $\median(\{X_i\})=\theta_{1/2}(\{X_i\})$, and we use both notations interchangeably.
\end{remark}

Next, we show that as long as the sample size is large enough, the sample quantile concentrates around the population quantile (motivated from \cite{charikar2002finding}), as in Lemma~\ref{lem:quantileconcentration}.
\begin{lemma}\label{lem:quantileconcentration}
Suppose $F(\cdot)$ is cumulative distribution function (i.e., non-decreasing and right-continuous) with continuous density function $F'(\cdot)$. Assume the samples $\{X_i\}_{i=1}^m$ are i.i.d. drawn from $F$. Let $0<p<1$. If $l<F'(\theta)<L$ for all $\theta$ in $\{\theta:|\theta-\theta_p|\le \epsilon$\}, then
\begin{flalign}
	|\theta_p(\{X_i\}_{i=1}^m)-\theta_p(F) |<\epsilon
\end{flalign}
holds with probability at least $1-2\exp(-2m\epsilon^2l^2)$.
\end{lemma}
\begin{proof} See Appendix~\ref{proof_lem:quantileconcentration}.
\end{proof}

Lemma~\ref{lem:orderstats} bounds the distance between the median of two sequences.
\begin{lemma}\label{lem:orderstats}
Given a vector $\bX=(X_1,X_2,...,X_n)$, reorder the entries in a non-decreasing manner
\begin{flalign*}
X_{(1)}\le X_{(2)}\le ... \le X_{(n-1)}\le X_{(n)}.
\end{flalign*}
Given another vector $\bY=(Y_1,Y_2,..., Y_n)$, then 
\begin{flalign}
	|X_{(k)}-Y_{(k)}|\le \|\bX-\bY\|_\infty,
\end{flalign}
holds for all $k=1,..., n.$
\end{lemma}
\begin{proof} See Appendix~\ref{proof_lem:orderstats}.
\end{proof}

Lemma~\ref{lem:empiricalmedian}, as a key robustness property of median, suggests that in the presence of outliers, one can bound the sample median from both sides by neighboring quantiles of the corresponding clean samples.

\begin{lemma}\label{lem:empiricalmedian}
Consider \emph{clean samples} $\{\tilde{X}_i\}_{i=1}^m$. If a fraction $s$ ($s<\frac{1}{2}$) of them are corrupted by outliers, one obtains \emph{contaminated samples} $\{X_i\}_{i=1}^m$ which contain $sm$ corrupted samples and $(1-s)m$ clean samples. Then for a quantile $p$ such that $s < p < 1-s$, we have
\begin{flalign}
	\theta_{p-s}(\{\tilde{X}_i\})\le\theta_{p} (\{X_i\}) \le\theta_{p+s}(\{\tilde{X}_i\}). \nn
\end{flalign}
\end{lemma}
\begin{proof} See Appendix~\ref{proof_lem:empiricalmedian}.
\end{proof}

Finally, Lemma~\ref{lem:product} is related to bound the value of the median, as well as the density at the median for the product of two possibly correlated standard Gaussian random variables.
\begin{lemma}\label{lem:product}
Let $u,v \sim \mathcal{N}(0,1)$ which can be correlated with the correlation coefficient $|\rho|\leq 1$. Let $r=|uv|$, and $\psi_{\rho}(x)$ represent the density of $r$. Denote $\theta_{\frac{1}{2}}(\psi_{\rho})$ as the median of $r$, and the value of $\psi_{\rho}(x)$ at the median as $\psi_{\rho}(\theta_{1/2})$. Then for all $\rho$,
\begin{flalign}
& 0.348<\theta_{1/2}(\psi_{\rho})<0.455,\nn \\
&0.47<\psi_{\rho}(\theta_{1/2})<0.76.\nn
\end{flalign}
\end{lemma}
\begin{proof} See Appendix~\ref{proof_lem:product}.
\end{proof}

\subsection{Robust Initialization with Outliers}\label{sec:initialization}

Considering the model that the measurements are corrupted by both bounded noise and sparse outliers given by \eqref{eq:twonoisesmodel}, we show that the initialization provided by the median-truncated spectral method in \eqref{eq:init_medianTWF} is close enough to the ground truth, i.e., $\dist(\bz^{(0)},\bx)\le \delta \|\bx\|$.

\begin{proposition}\label{prop:initialization}
Fix $\delta>0$ and $\bx \in \bbR^n$, and consider the model given by \eqref{eq:twonoisesmodel}. Suppose that $\|\mathbf{w}\|_\infty\le c \|\bx\|^2$ for some sufficiently small constant $c>0$ and that $\|\mathbf{\eta}\|_0\le sm$ for some sufficiently small constant $s$. With probability at least $1-\exp(-\Omega(m))$, the initialization given by the median-truncated spectral method obeys\footnote{Notation $f(n) = \Omega(g(n))$ or $f(n)\gtrsim g(n)$ means that there exists a  constant $c>0$ such that $|f(n)|\ge c|g(n)|$.} 
\begin{flalign}
	dist(\bz^{(0)}, \bx) \le \delta \|\bx\|,
\end{flalign}
provided that $m>c_0 n$ for some constant $c_0>0$.
\end{proposition}

\begin{proof} See Appendix~\ref{supp:initialization}.
\subsection{Regularity Condition} \label{sec:rc_convergence}
Once the initialization is guaranteed to be within a small neighborhood of the ground truth, we only need to show that the truncated gradient \eqref{trimmed_loss} and \eqref{eq:RWFgradient} satisfy the {\em Regularity Condition} ($\mathsf{RC}$) \cite{candes2015phase,chen2015solving}, which guarantees the geometric convergence of median-TWF/median-RWF once the initialization lands into this neighborhood.

\begin{definition}  \label{def:RC}
The gradient $\nabla \ell(\bz)$ is said to satisfy the Regularity Condition $\mathsf{RC}(\mu,\lambda,c)$ if
\begin{flalign}
\left\langle \nabla \ell(\bz), \bz-\bx \right\rangle \ge \frac{\mu}{2} \left\|\nabla \ell(\bz)\right\|^2+\frac{\lambda}{2} \|\bz-\bx\|^2 \label{eq:defRC}
\end{flalign}
for all $\bz$ obeying $\|\bz-\bx\|\le c \|\bx\|$.
\end{definition}
The above $\mathsf{RC}$ guarantees that the gradient descent update $\bz^{(t+1)}=\bz^{(t)}-\mu \nabla \ell(\bz)$ converges to the true signal $\bx$ geometrically \cite{chen2015solving} if $\mu\lambda<1$. We repeat this argument below for completeness.
\begin{flalign*}
\dist^2(\bz-\mu \nabla \ell(\bz),\bx) &\le \left\|\bz-\mu \nabla \ell(\bz)-\bx\right\|^2\\
&=\|\bz-\bx\|^2+\|\mu\nabla\ell(\bz)\|^2-2\mu\left\langle\bz-\bx, \nabla\ell(\bz)\right\rangle\\
&\le \|\bz-\bx\|^2+\|\mu\nabla\ell(\bz)\|^2-\mu^2\|\nabla\ell(\bz)\|^2-\mu\lambda\left\|\bz-\bx\right\|^2\\
&=(1-\mu\lambda)\dist^2(\bz,\bx).
\end{flalign*}

\section{Proofs for Median-TWF}\label{sec:proof:median_twf}
We first show that $\nabla \ell_{tr}(\bz)$ in \eqref{trimmed_loss} satisfies the $\mathsf{RC}$ for the noise-free case in Section~\ref{sec:proof_prop:noisefree}, and then extend it to the model with only sparse outliers in Section~\ref{sec:proofTWFoutliers}, thus together with Proposition \ref{prop:initialization} establishing the global convergence of median-TWF in both cases. Section~\ref{app:thm:twonoises} proves Theorem~\ref{thm:twonoises} in the presence of both sparse outliers and dense bounded noise.

\subsection{Proof of Proposition \ref{prop:noisefree}} \label{sec:proof_prop:noisefree}


We consider the noise-free model. The central step to establish the $\mathsf{RC}$ is to show that the sample median used in the truncation rule of median-TWF concentrates at the level $\|\bz-\bx\|\|\bz\|$ as stated in the following proposition.
\begin{proposition}\label{prop:mediansandwich}
If $m>c_0 n\log n  $, then with probability at least $1-c_1\exp(-c_2  m)$, 
\begin{flalign}
0.6\|\bz\|\|\bz-\bx\|  \le \theta_{0.49}, \theta_{0.5},\theta_{0.51}(\left\{\left||\ba_i^T\bx|^2-|\ba_i^T\bz|^2\right|\right\}_{i=1}^m) \le \|\bz\|\|\bz-\bx\|, \label{eq:median_sandwichTWF}
\end{flalign}
holds for all $\bz, \bx$ satisfying $\|\bz-\bx\|<1/11\|\bz\|$.
\end{proposition}
\begin{proof}
	Detailed proof is provided in Appendix \ref{app:SamMedCon}.
\end{proof}

We note that a similar property for the sample mean has been shown in \cite{chen2015solving} as long as the number $m$ of measurements is on the order of $n$. In fact, the sample median is much more challenging to bound due to its non-linearity, which also causes slightly more measurements compared to the sample mean.

Then we can establish that $\left\langle \nabla \ell_{tr}(\bz), \bz-\bx \right\rangle$ is lower bounded on the order of $\|\bz-\bx\|^2$, as in Proposition~\ref{prop:strongconvex}, and that $\left\|\nabla \ell_{tr}(\bz)\right\|$ is upper bounded on the order of $\|\bz-\bx\|$, as in Proposition~\ref{prop:localsmooth}.
\begin{proposition}[Adapted version of Proposition 2 of \cite{chen2015solving}]\label{prop:strongconvex}
Consider the noise-free case $y_i=|\ba_i^T\bx|^2$ for $i=1, \cdots, m$, and any fixed constant $\epsilon>0$. Under the condition \eqref{eq:parameters}, if $m>c_0 n\log n$, then with probability at least $1-c_1\exp(-c_2\epsilon^{-2}m)$,
\begin{flalign}
\left\langle \nabla \ell_{tr}(\bz), \bz-\bx \right\rangle\ge \left\{1.99-2(\zeta_1+\zeta_2)-\sqrt{8/\pi}\alpha_h^{-1}-\epsilon\right\}\|\bz-\bx\|^2
\end{flalign}
holds uniformly over all $\bx, \bz\in \bbR^n$ satisfying
\begin{flalign}
	\frac{\|\bz-\bx\|}{\|\bz\|}\le \min\left\{\frac{1}{11}, \frac{\alpha_l}{\alpha_h}, \frac{\alpha_l}{6}, \frac{\sqrt{98/3}(\alpha_l)^2}{2\alpha_u+\alpha_l}\right\},
\end{flalign}
where $c_0, c_1, c_2>0$ are some universal constants, and $\zeta_1, \zeta_2, \alpha_l, \alpha_u$ and $\alpha_h$ are defined in \eqref{eq:parameters}.
\end{proposition}
The proof of Proposition \ref{prop:strongconvex} adapts the proof of Proposition 2 of \cite{chen2015solving}, by properly setting parameters based on the properties of sample median. For completeness, we include a short outline of the proof in Appendix~\ref{app:prop:strongconvex}.


\begin{proposition}[Lemma 7 of \cite{chen2015solving}]\label{prop:localsmooth}
Under the same condition as in Proposition \ref{prop:strongconvex}, if $m>c_0 n$, then there exist some constants $c_1, c_2>0$ such that with probability at least $1-c_1\exp(-c_2 m)$,
\begin{flalign}
\left\|\nabla \ell_{tr}(\bz)\right\| \le (1+\delta)\cdot 2\sqrt{1.02+2/\alpha_h}\|\bz-\bx\|
\end{flalign}
holds uniformly over all $\bx, \bz\in \bbR^n$ satisfying
\begin{flalign}
	\frac{\|\bz-\bx\|}{\|\bz\|}\le \min\left\{\frac{1}{11}, \frac{\alpha_l}{\alpha_h}, \frac{\alpha_l}{6}, \frac{\sqrt{98/3}(\alpha_l)^2}{2\alpha_u+\alpha_l}\right\},
\end{flalign}
where $\delta$ can be arbitrarily small as long as $m/n$ sufficiently large, and  $\alpha_l, \alpha_u$ and $\alpha_h$ are given in \eqref{eq:parameters}.
\end{proposition}
\begin{proof}
See the proof of Lemma 7 in \cite{chen2015solving}.
\end{proof}

With these two propositions and \eqref{eq:parameters}, $\mathsf{RC}$ is guaranteed by setting 
\begin{flalign*}
&\mu<\mu_0:=\frac{(1.99-2(\zeta_1+\zeta_2)-\sqrt{8/\pi}\alpha_h^{-1}}{2(1+\delta)^2\cdot(1.02+2/\alpha_h)},\\
&\lambda+\mu\cdot 4(1+\delta)^2\cdot(1.02+2/\alpha_h)<2\left\{1.99-2(\zeta_1+\zeta_2)-\sqrt{8/\pi}\alpha_h^{-1}-\epsilon\right\}.
\end{flalign*}

\subsection{Proof of Theorem \ref{thm:outliers} }\label{sec:proofTWFoutliers}

We next consider the model \eqref{eq:twonoisesmodel} with only sparse outliers. It suffices to show that  $\nabla \ell_{tr}(\bz)$ continues to satisfy the $\mathsf{RC}$. The critical step is to bound the sample median of the corrupted measurements. Lemma~\ref{lem:empiricalmedian} yields
\begin{flalign}
\theta_{\frac{1}{2}-s}(\{|(\ba_i^T\bx)^2-(\ba_i^T\bz)^2|\})\le \theta_{\frac{1}{2}}(\{|y_i-(\ba_i^T\bz)^2|\}) \le \theta_{\frac{1}{2}+s}(\{|(\ba_i^T\bx)^2-(\ba_i^T\bz)^2|\}.
\end{flalign}
For simplicity of notation, we let $\bh:=\bz-\bx$. Then for the instance of $s=0.01$, by Proposition \ref{prop:mediansandwich}, we have with probability at least $1-2\exp(-\Omega(m))$,
\begin{flalign}
0.6\|\bz\|\|\bh\|\le \theta_{\frac{1}{2}}(\{|y_i-(\ba_i^T\bz)^2|\}) \le \|\bz\|\|\bh\|.
\end{flalign}

To differentiate from $\cE_2^i$, we define $\tilde{\cE}_2^i:=\left\{\left|(\ba_i^T\bx)^2-(\ba_i^T\bz)^2\right|\le \alpha_h \median \left\{\left|y_i - (\ba_i^T\bz)^2\right|\right\}\frac{|\ba_i^T\bz|}{\|\bz\|}\right\}$. We then have
\begin{flalign}
\nabla \ell_{tr}(\bz) &= \frac{1}{m}\sum_{i=1}^m \frac{(\ba_i^T\bz)^2-y_i}{\ba_i^T\bz} \ba_i \bone_{\cE_1^i\cap \cE_2^i} \nn\\
&= \underbrace{\frac{1}{m}\sum_{i=1}^m \frac{(\ba_i^T\bz)^2-(\ba_i^T\bx)^2}{\ba_i^T\bz} \ba_i \bone_{\cE_1^i\cap \tilde{\cE}_2^i} }_{\nabla^{clean}\ell_{tr}(\bz)} + \underbrace{\frac{1}{m}\sum_{i\in S} \left( \frac{(\ba_i^T\bz)^2-y_i}{\ba_i^T\bz}\bone_{\cE_1^i\cap \cE_2^i}-\frac{(\ba_i^T\bz)^2-(\ba_i^T\bx)^2}{\ba_i^T\bz} \bone_{\cE_1^i\cap \tilde{\cE}_2^i} \right)\ba_i}_{\nabla^{extra}\ell_{tr}(\bz)}  \nn.
\end{flalign}

Choosing $\epsilon$ small enough, it is easy to verify that Propositions \ref{prop:strongconvex} and \ref{prop:localsmooth} are still valid on $\nabla^{clean}\ell_{tr}(\bz)$. Thus, one has
\begin{flalign*}
&\langle \nabla^{clean} \ell_{tr}(\bz), \bh\rangle \ge \left\{1.99-2(\zeta_1+\zeta_2)-\sqrt{8/\pi}\alpha_h^{-1}-\epsilon\right\}\|\bh\|^2,\\
&\left\|\nabla^{clean}\ell_{tr}(\bz)\right\|\le (1+\delta)\cdot 2\sqrt{1.02+2/\alpha_h} \|\bh\|.
\end{flalign*}

We next bound the contribution of $\nabla^{extra}\ell_{tr}(\bz)$. Introduce $\bq=[q_1,\ldots,q_m]^T$, where
\begin{flalign*}
q_i:=\left( \frac{(\ba_i^T\bz)^2-y_i}{\ba_i^T\bz}\bone_{\cE_1^i\cap \cE_2^i}-\frac{(\ba_i^T\bz)^2-(\ba_i^T\bx)^2}{\ba_i^T\bz} \bone_{\cE_1^i\cap \tilde{\cE}_2^i} \right)\bone_{\{i \in S\}}.
\end{flalign*}
It can be seen that $|q_i|\le 2\alpha_h\|\bh\|$. Thus $\|\bq\|\le \sqrt{sm}\cdot 2\alpha_h\|\bh\|$, and
\begin{flalign*}
\left\|  \nabla^{extra} \ell_{tr}(\bz)\right\| &=\frac{1}{m}\left\|\bA^T\bq\right\|\le 2(1+\delta)\sqrt{s}\alpha_h\|\bh\|,\\
\left|\left\langle \nabla^{extra} \ell_{tr}(\bz), \bh\right\rangle\right| &\le \|\bh\| \cdot \left\| \frac{1}{m} \nabla^{extra} \ell_{tr}(\bz)\right\|\le 2(1+\delta)\sqrt{s}\alpha_h\|\bh\|^2,
\end{flalign*}
where $\bA = [\ba_1,\ldots, \ba_m]^T$. Then, we have
\begin{flalign*}
-\left\langle \nabla \ell_{tr}(\bz), \bh \right\rangle & \ge \left\langle \nabla^{clean}\ell_{tr}(\bz),\bh \right\rangle- \left|\left\langle \nabla^{extra} \ell_{tr}(\bz), \bh\right\rangle\right|\\
&\ge \left(1.99-2(\zeta_1+\zeta_2)-\sqrt{8/\pi}\alpha_h^{-1}-\epsilon-2(1+\delta)\sqrt{s}\alpha_h\right)\|\bh\|^2,
\end{flalign*}
and
\begin{flalign}
\left\|\nabla \ell_{tr}(\bz)\right\|&\le \left\|\nabla^{clean} \ell_{tr}(\bz)\right\|+\left\| \nabla^{extra} \ell_{tr}(\bz)\right\| \nonumber \\
&\le 2(1+\delta)\left(\sqrt{1.02+2/\alpha_h}+\sqrt{s}\alpha_h\right) \|\bh\|.
\end{flalign}
Therefore, the $\mathsf{RC}$ is guaranteed if $\mu,\lambda,\epsilon$ are chosen properly and $s$ is sufficiently small.

\subsection{Proof of Theorem~\ref{thm:twonoises}}\label{app:thm:twonoises}

We consider the model \eqref{eq:twonoisesmodel}, and split our analysis of the gradient loop into two regimes.

$\bullet$ \textbf{Regime 1}: $c_4\|\bz\|\ge \|\bh\|\ge c_3\frac{\|\bw\|_\infty}{\|\bz\|}$. In this regime, error contraction by each gradient step is given by
\begin{flalign*}
\dist\left(\bz-\mu\nabla \ell_{tr}(\bz), \bx\right)\le (1-\rho) \dist(\bz,\bx).
\end{flalign*}
It suffices to justify that $\nabla \ell_{tr}(\bz)$ satisfies the $\mathsf{RC}$.  
Denote $\tilde{y}_i:=(\ba_i^T\bx)^2+w_i$. Then by Lemma \ref{lem:empiricalmedian}, we have
\begin{flalign*}
\theta_{\frac{1}{2}-s}\left\{\left|\tilde{y}_i-(\ba_i^T\bz)^2\right|\right\}\le\median \left\{\left|y_i-(\ba_i^T\bz)^2\right|\right\}\le \theta_{\frac{1}{2}+s}\left\{\left|\tilde{y}_i-(\ba_i^T\bz)^2\right|\right\}.
\end{flalign*}
Moreover, by Lemma \ref{lem:orderstats} we have
\begin{flalign*}
&\left|\theta_{\frac{1}{2}+s}\left\{\left|\tilde{y}_i-(\ba_i^T\bz)^2\right|\right\}- \theta_{\frac{1}{2}+s}\left\{\left|(\ba_i^T\bx)^2-(\ba_i^T\bz)^2\right|\right\}\right|\le\|\bw\|_\infty,\\
&\left|\theta_{\frac{1}{2}-s}\left\{\left|\tilde{y}_i-(\ba_i^T\bz)^2\right|\right\}- \theta_{\frac{1}{2}-s}\left\{\left|(\ba_i^T\bx)^2-(\ba_i^T\bz)^2\right|\right\}\right|\le\|\bw\|_\infty.
\end{flalign*}
Assume that $s=0.01$. By Proposition \ref{prop:mediansandwich}, if $c_3$ is sufficiently large (i.e., $c_3>100)$, we still shave
\begin{flalign}
0.6\|\bx-\bz\|\|\bz\|\le\median \left\{\left|y_i-(\ba_i^T\bz)^2\right|\right\}\le \|\bx-\bz\|\|\bz\| .\label{eq:medianboundtwonoise}
\end{flalign}

Furthermore, recall $\tilde{\cE}_2^i:=\left\{\left|(\ba_i^T\bx)^2-(\ba_i^T\bz)^2\right|\le \alpha_h \median \left\{\left|(\ba_i^T\bz)^2-y_i\right|\right\}\frac{|\ba_i^T\bz|}{\|\bz\|}\right\}$. Then,
\begin{flalign}
\nabla \ell_{tr}(\bz) &= \frac{1}{m}\sum_{i=1}^m \frac{(\ba_i^T\bz)^2-y_i}{\ba_i^T\bz} \ba_i \bone_{\cE_1^i\cap \cE_2^i}\nn\\
& = \underbrace{\frac{1}{m}\left(\sum_{i\notin S} \frac{(\ba_i^T\bz)^2-(\ba_i^T\bx)^2}{\ba_i^T\bz} \ba_i \bone_{\cE_1^i\cap \cE_2^i}+\sum_{i\in S} \frac{(\ba_i^T\bz)^2-(\ba_i^T\bx)^2}{\ba_i^T\bz} \ba_i \bone_{\cE_1^i\cap \tilde{\cE}_2^i}\right)}_{\nabla^{clean}\ell_{tr}(\bz)} \nn\\
&\quad-\underbrace{\frac{1}{m}\sum_{i\notin S}\frac{\bw_i}{\ba_i^T\bz} \ba_i \bone_{\cE_1^i\cap \cE_2^i}}_{\nabla^{noise}\ell_{tr}(\bz)}+\underbrace{\frac{1}{m}\sum_{i\in S} \left( \frac{(\ba_i^T\bz)^2-y_i}{\ba_i^T\bz}\bone_{\cE_1^i\cap \cE_2^i}-\frac{(\ba_i^T\bz)^2-(\ba_i^T\bx)^2}{\ba_i^T\bz} \bone_{\cE_1^i\cap \tilde{\cE}_2^i} \right)\ba_i }_{\nabla^{extra}\ell_{tr}(\bz)}.\nn
\end{flalign}

For $i\notin S$, the inclusion property (i.e. $\cE_3^i\subseteq \cE_2^i \subseteq \cE_4^i$) holds because
\begin{flalign*}
\left|y_i-(\ba_i^T\bz)^2\right|\in \left|(\ba_i^T\bx)^2-(\ba_i^T\bz)^2\right|\pm |w_i|
\end{flalign*}
and $|w_i|\le \frac{1}{c_3}\|\bh\|\|\bz\|$ for some sufficient large $c_3$. For $i\in S$, the inclusion $\cE_3^i\subseteq \tilde{\cE}_2^i \subseteq \cE_4^i$ holds because of \eqref{eq:medianboundtwonoise}. All the proof arguments for Propositions \ref{prop:strongconvex} and \ref{prop:localsmooth} are also valid for $\nabla^{clean} \ell_{tr}(\bz)$, and thus we have
\begin{flalign*}
&\langle \nabla^{clean} \ell_{tr}(\bz), \bh\rangle \ge \left\{1.99-2(\zeta_1+\zeta_2)-\sqrt{8/\pi}\alpha_h^{-1}-\epsilon\right\}\|\bh\|^2,\\
&\left\|\nabla^{clean}\ell_{tr}(\bz)\right\|\le (1+\delta)\cdot 2\sqrt{1.02+2/\alpha_h} \|\bh\|.
\end{flalign*}

Next, we turn to control the contribution of the noise. Let $\tilde{w}_i=\frac{w_i}{\ba_i^T\bz}\bone_{\cE_1^i\cap \cE_2^i}$, and then we have
\begin{flalign*}
\|\nabla^{noise}\ell_{tr}(\bz)\|=\left\|\frac{1}{m}\bA^T\tilde{\bw}\right\|\le \left\|\frac{1}{\sqrt{m}}\bA^T\right\|\left\|\frac{\tilde{\bw}}{\sqrt{m}}\right\|\le (1+\delta)\|\tilde{\bw}\|_\infty\le (1+\delta)\frac{\|\bw\|_\infty}{\alpha_l\|\bz\|},
\end{flalign*}
when $m/n$ is sufficiently large. Given the regime condition $\|\bh\|\ge c_3\frac{\|\bw\|_\infty}{\|\bz\|}$, we further have
\begin{flalign*}
&\|\nabla^{noise}\ell_{tr}(\bz)\|\le  \frac{(1+\delta)}{c_3\alpha_l}\|\bh\|,\\
& \left|\left\langle \nabla^{noise} \ell_{tr}(\bz), \bh\right\rangle\right| \le  \left\| \nabla^{noise} \ell_{tr}(\bz)\right\|\cdot\|\bh\| \le \frac{(1+\delta)}{c_3\alpha_l}\|\bh\|^2.
\end{flalign*}
We next bound the contribution of $\nabla^{extra}\ell_{tr}(\bz)$. Introduce $\bq=[q_1,\ldots,q_m]^T$, where
\begin{flalign*}
q_i:=\left( \frac{(\ba_i^T\bz)^2-y_i}{\ba_i^T\bz}\bone_{\cE_1^i\cap \cE_2^i}-\frac{(\ba_i^T\bz)^2-(\ba_i^T\bx)^2}{\ba_i^T\bz} \bone_{\cE_1^i\cap \tilde{\cE}_2^i} \right)\bone_{\{i \in S\}}.
\end{flalign*}
Then $|q_i|\le 2\alpha_h\|\bh\|$, and $\|\bq\|\le \sqrt{sm}\cdot 2\alpha_h\|\bh\|$. We thus have
\begin{flalign*}
&\left\|\nabla^{extra} \ell_{tr}(\bz)\right\|=\frac{1}{m}\left\|\bA^T\bq\right\|\le 2(1+\delta)\sqrt{s}\alpha_h\|\bh\|,\\
&\left|\left\langle  \nabla^{extra} \ell_{tr}(\bz), \bh\right\rangle\right| \le \|\bh\| \cdot \left\|\nabla^{extra} \ell_{tr}(\bz)\right\|\le 2(1+\delta)\sqrt{s}\alpha_h\|\bh\|^2.
\end{flalign*}

Putting these together, one has
\begin{flalign}
\left\langle \nabla \ell_{tr}(\bz), \bh \right\rangle & \ge \left\langle \nabla^{clean}\ell_{tr}(\bz),\bh \right\rangle- \left|\left\langle  \nabla^{noise} \ell_{tr}(\bz), \bh\right\rangle\right|- \left|\left\langle  \nabla^{extra} \ell_{tr}(\bz), \bh\right\rangle\right|\nn\\
&\ge \left(1.99-2(\zeta_1+\zeta_2)-\sqrt{8/\pi}\alpha_h^{-1}-\epsilon-(1+\delta)(1/(c_3\alpha_z^l)+2\sqrt{s}\alpha_h)\right)\|\bh\|^2,
\end{flalign}
and
\begin{flalign}
\left\|\nabla \ell_{tr}(\bz)\right\|&\le \left\|\nabla^{clean} \ell_{tr}(\bz)\right\|+\left\|\nabla^{noise} \ell_{tr}(\bz)\right\|+\left\|\nabla^{extra} \ell_{tr}(\bz)\right\| \nn \\
&\le (1+\delta)\left(2\sqrt{1.02+2/\alpha_h}+1/(c_3\alpha_z^l)+2\sqrt{s}\alpha_h\right) \|\bh\|.
\end{flalign}

The $\mathsf{RC}$ is guaranteed if $\mu,\lambda,\epsilon$ are chosen properly, $c_3$ is sufficiently large and $s$ is sufficiently small.

$\bullet$ \textbf{Regime 2}: Once the iterate enters this regime with $\|\bh\|\le \frac{c_3\|\bw\|_\infty}{\|\bz\|}$, each gradient iterate may not reduce the estimation error. However, in this regime each move size $\mu\nabla \ell_{tr}(\bz)$ is at most $\mathcal{O}(\|\bw\|_\infty/\|\bz\|)$. Then the estimation error cannot increase by more than $\frac{\|\bw\|_\infty}{\|\bz\|}$ with a constant factor. Thus one has
\begin{flalign*}
\dist\left(\bz-\mu\nabla \ell_{tr}(\bz),\bx\right)\le c_5\frac{\|\bw\|_\infty}{\|\bx\|}
\end{flalign*}
for some constant $c_5$. As long as $\|\bw\|_\infty/\|\bx\|^2$ is sufficiently small, it is guaranteed that $c_5\frac{\|\bw\|_\infty}{\|\bx\|}\le c_4\|\bx\|$. If the iterate jumps out of \emph{Regime 2}, it falls into \emph{Regime 1}.

\section{Proofs for Median-RWF}\label{sec:proof:median_rwf}

We first show that $\nabla \cR_{tr}(\bz)$ in \eqref{eq:RWFgradient} satisfies the $\mathsf{RC}$ for the noise-free case in Section~\ref{sec:proof_prop:noisefree_RWF}, and then extend it to the model with only sparse outliers in Section~\ref{sec:proofRWFoutliers}, thus together with Proposition \ref{prop:initialization} establishing the global convergence of median-RWF in both cases. Section~\ref{app:thm:twonoisesRWF} proves Theorem~\ref{thm:twonoises} in the presence of both sparse outliers and dense bounded noise.

\subsection{Proof of Proposition \ref{prop:noisefree}} \label{sec:proof_prop:noisefree_RWF}

The central step to establish the $\mathsf{RC}$ is to show that the sample median used in the truncation rule of median-RWF concentrates on the order of $\|\bz-\bx\|$ as stated in the following proposition.
\begin{proposition}\label{prop:mediansandwichRWF}
If $m>c_0 n\log n$, then with probability at least $1-c_1\exp(-c_2  m)$,
\begin{flalign}
0.5\|\bz-\bx\|  \le \theta_{0.49},\theta_{1/2}, \theta_{0.51}\left(\left\{\left||\ba_i^T\bz|-|\ba_i^T\bx|\right|\right\}_{i=1}^m\right) \le 0.8\|\bz-\bx\|, \label{eq:median_sandwichRWF}
\end{flalign}
holds for all $\bz, \bx$ satisfying $\|\bz-\bx\|<1/11\|\bz\|$.
\end{proposition}

\begin{proof}
	See Appendix \ref{app:prop:mediansandwichRWF}.
\end{proof}

Next we give a bound on the left hand side of $\mathsf{RC}$.
\begin{proposition}[Adapted version of Proposition 2 of \cite{chen2015solving}]\label{prop:strongconvexRWF}
Consider the noise-free measurements $y_i=|\ba_i^T\bx|$ and any fixed constant $\epsilon>0$. If $m>c_0 n\log n$, then with probability at least $1-c_1\exp(-c_2m)$,
\begin{flalign}
\left\langle \nabla \cR_{tr}(\bz), \bz-\bx \right\rangle\ge \left\{0.88-\zeta'_1-\zeta'_2-\epsilon\right\}\|\bz-\bx\|^2
\end{flalign}
holds uniformly over all $\bx, \bz\in \bbR^n$ satisfying $\frac{\|\bz-\bx\|}{\|\bz\|}\le \frac{1}{20}$, 
where $c_0, c_1, c_2>0$ are some universal constants, and $\zeta'_1, \zeta'_2$ are given by
\begin{flalign*}
&\zeta'_1:=1-\min\left\{\mE\left[\xi^2\bone_{\{\xi\ge 0.5\sqrt{1.01}\alpha'_h\frac{\|\bz-\bx\|}{\|\bx\|}\}}\right], \mE\left[\bone_{\{\xi\ge 0.5\sqrt{1.01}\alpha'_h\frac{\|\bz-\bx\|}{\|\bx\|}\}}\right]\right\}\\
&\zeta'_2:=\mE\left[\xi^2\bone_{\{|\xi|>0.5\sqrt{0.99}\alpha'_h\}}\right]
\end{flalign*} for some $\xi\sim \cN(0,1)$ and $\alpha'_h=5$.
\end{proposition}
\begin{proof}
	See Appendix \ref{app:prop:strongconvexRWF}.
\end{proof}
Proposition \ref{prop:strongconvexRWF} indicates that $\left\langle \nabla \cR_{tr}(\bz), \bz-\bx \right\rangle$ is lower bounded by $\|\bz-\bx\|^2$ with some positive constant coefficient. In order to prove the $\mathsf{RC}$, it suffices to show that $\|\nabla \cR_{tr}(\bz)\|$ is upper bounded by the order of $\|\bz-\bx\|$ when $\bz$ is within the neighborhood of true signal $\bx$.
\begin{proposition}[Lemma 7 of \cite{chen2015solving}]\label{prop:localsmoothRWF}
If $m>c_0 n$, then there exist some constants $c_1, c_2>0$ such that with probability at least $1-c_1\exp(-c_2 m)$,
\begin{flalign}
\left\|\nabla \cR_{tr}(\bz)\right\| \le (1.8+\delta)\|\bz-\bx\|
\end{flalign}
holds uniformly over all $\bx, \bz\in \bbR^n$ satisfying $\|\bx-\bz\|\le \frac{1}{11}\|\bx\|$
where $\delta$ can be arbitrarily small as long as $c_0$ sufficiently large.
\end{proposition}
\begin{proof}
	See Appendix \ref{app:prop:localsmoothRWF}.
\end{proof}

With the above two propositions, $\mathsf{RC}$ is guaranteed by setting $\mu<\mu_0:=\frac{2(0.88-\zeta'_1-\zeta'_2-\epsilon)}{(1.8+\delta)^2}$ and $\lambda+\mu\cdot (1.8+
\delta)^2<2(0.88-\zeta'_1-\zeta'_2-\epsilon)$.

\subsection{Proof of Theorem \ref{thm:outliers}}\label{sec:proofRWFoutliers}
We consider the model \eqref{eq:twonoisesmodel} with only outliers, i.e., $y_i=\left|\langle \ba_i,\bx \rangle\right|^2+\eta_i$ for $i=1,\cdots,m$. It suffices to show that  $\nabla \cR_{tr}(z)$ satisfies the $\mathsf{RC}$. The critical step is to lower and upper bound the sample median of the corrupted measurements. Lemma \ref{lem:empiricalmedian} yields
\begin{flalign}
\theta_{\frac{1}{2}-s}(\{||\ba_i^T\bx|-|\ba_i^T\bz||\})\le \theta_{\frac{1}{2}}(\{|\sqrt{y_i}-|\ba_i^T\bz||\}) \le \theta_{\frac{1}{2}+s}(\{||\ba_i^T\bx|-|\ba_i^T\bz||\}.
\end{flalign}
For the simplicity of notation, we let $\bh:=\bz-\bx$. Then for the instance of $s=0.01$, Proposition \ref{prop:mediansandwichRWF} yields that if $m>c_0 n\log n$, then
\begin{flalign}
0.5\|\bh\|\le \theta_{\frac{1}{2}}(\{|\sqrt{y_i}-|\ba_i^T\bz||\}) \le 0.8\|\bh\| \label{eq:RWFmedianboundoutliers}
\end{flalign}
holds with probability at least $1-2\exp(-\Omega(m))$.

To differentiate from $\cT^i$, we define $\tilde{\cT}^i:=\left\{\left||\ba_i^T\bx|-|\ba_i^T\bz|\right|\le \alpha'_h \median \left\{\left|\sqrt{y_i} - |\ba_i^T\bz|\right|\right\}\right\}$. We then have
\begin{flalign}
\nabla \cR_{tr}(\bz) &= \frac{1}{m}\sum_{i=1}^m \left(|\ba_i^T\bz|-\sqrt{y_i}\right) \ba_i \bone_{\cT^i} \nn\\
&= \underbrace{\frac{1}{m}\sum_{i=1}^m \left(|\ba_i^T\bz|-\sqrt{y_i}\right) \ba_i \bone_{\tilde{\cT}^i} }_{\nabla^{clean}\cR_{tr}(\bz)} + \underbrace{\frac{1}{m}\sum_{i\in S} \left( \left(|\ba_i^T\bz|-\sqrt{y_i}\right)  \bone_{\cT^i}- \left(|\ba_i^T\bz|-|\ba_i^T\bx|\right) \bone_{\tilde{\cT}^i} \right)\ba_i}_{\nabla^{extra}\cR_{tr}(\bz)}  \nn.
\end{flalign}

Under the condition \eqref{eq:RWFmedianboundoutliers}, the inclusion property (i.e., $\cT_1^i\subseteq \tilde{\cT}^i \subseteq \cT_2^i$) holds, and all the proof arguments for Propositions \ref{prop:strongconvexRWF} and \ref{prop:localsmoothRWF}	are also valid to $\nabla^{clean}\cR_{tr}(\bz)$. Thus, one has
\begin{flalign*}
&\left\langle \nabla^{clean} \cR_{tr}(\bz), \bh \right\rangle \ge \left(0.88-\zeta'_1-\zeta'_2-\epsilon\right)\|\bh\|^2\\
&\left\|\nabla^{clean}\cR_{tr}(\bz)\right\|\le (1.8+\delta)\|\bh\|.
\end{flalign*}

We next bound the contribution of $\nabla^{extra}\cR_{tr}(\bz)$. Introduce $\bq=[q_1,\ldots,q_m]^T$, where
\begin{flalign*}
q_i:=\left( \left(|\ba_i^T\bz|-\sqrt{y_i}\right)  \bone_{\cT^i}- \left(|\ba_i^T\bz|-|\ba_i^T\bx|\right) \bone_{\tilde{\cT}^i}  \right)\bone_{\{i \in S\}},
\end{flalign*}
and then $|q_i|\le 1.6 \alpha'_h\|\bh\|$. Thus, $\|\bq\|\le \sqrt{sm}\cdot 1.6 \alpha'_h\|\bh\|$, and
\begin{flalign*}
\left\| \nabla^{extra} \cR_{tr}(\bz)\right\| &=\frac{1}{m}\left\|\bA^T\bq\right\|\le 1.6 (1+\delta)\sqrt{s}\alpha'_h\|\bh\|,\\
\left|\left\langle  \nabla^{extra} \cR_{tr}(\bz), \bh\right\rangle\right| &\le \|\bh\| \cdot \left\|  \nabla^{extra} \cR_{tr}(\bz)\right\|\le 1.6(1+\delta)\sqrt{s}\alpha'_h\|\bh\|^2,
\end{flalign*}
where $\bA = [\ba_1,\ldots, \ba_m]^T$. Then, we have
\begin{flalign*}
\left\langle \nabla \cR_{tr}(\bz), \bh \right\rangle & \ge \left\langle \nabla^{clean}\cR_{tr}(\bz),\bh \right\rangle- \left|\left\langle  \nabla^{extra} \cR_{tr}(\bz), \bh\right\rangle\right|\\
&\ge \left(0.88-\zeta'_1-\zeta'_2-\epsilon-1.6(1+\delta)\sqrt{s}\alpha'_h\right)\|\bh\|^2,
\end{flalign*}
and
\begin{flalign*}
\left\|\nabla \cR_{tr}(\bz)\right\|&\le \left\|\nabla^{clean} \cR_{tr}(\bz)\right\|+\left\|  \nabla^{extra} \cR_{tr}(\bz)\right\|\\
&\le \left(1.8+\delta+1.6(1+\delta)\sqrt{s}\alpha'_h\right) \|\bh\|.
\end{flalign*}
Therefore the $\mathsf{RC}$ is guaranteed if $\mu,\lambda$ are chosen properly, $\delta$ is chosen sufficiently small and $s$ is sufficiently small.

\subsection{Proof of Theorem~\ref{thm:twonoises}}\label{app:thm:twonoisesRWF}

We consider the model \eqref{eq:twonoisesmodel} with outliers and bounded noise. We split our analysis of the gradient loop into two regimes.

$\bullet$ \textbf{Regime 1}: $c_4\|\bz\|\ge \|\bh\|\ge c_3\sqrt{\|\bw\|_\infty}$. In this regime, error contraction by each gradient step is given by
\begin{flalign}
	\dist\left(\bz-\mu\nabla \cR_{tr}(\bz), \bx\right)\le (1-\rho) \dist(\bz,\bx).
\end{flalign}
It suffices to justify that $\nabla \cR_{tr}(\bz)$ satisfies the $\mathsf{RC}$.  
Denote $\tilde{y}_i:=(\ba_i^T\bx)^2+w_i$. Then by Lemma \ref{lem:empiricalmedian}, we have
\begin{flalign*}
\theta_{\frac{1}{2}-s}\left\{\left|\sqrt{\tilde{y}_i}-|\ba_i^T\bz|\right|\right\}\le\median \left\{\left|\sqrt{y_i}-|\ba_i^T\bz|\right|\right\}\le \theta_{\frac{1}{2}+s}\left\{\left|\sqrt{\tilde{y}_i}-|\ba_i^T\bz|\right|\right\}.
\end{flalign*}
Moreover, by Lemma \ref{lem:orderstats} we have
\begin{flalign*}
&\left|\theta_{\frac{1}{2}+s}\left\{\left|\sqrt{\tilde{y}_i}-|\ba_i^T\bz|\right|\right\}- \theta_{\frac{1}{2}+s}\left\{\left||\ba_i^T\bx|-|\ba_i^T\bz|\right|\right\}\right|\le \sqrt{\|\bw\|_\infty},\\
&\left|\theta_{\frac{1}{2}-s}\left\{\left|\sqrt{\tilde{y}_i}-|\ba_i^T\bz|\right|\right\}- \theta_{\frac{1}{2}-s}\left\{\left||\ba_i^T\bx|-|\ba_i^T\bz|\right|\right\}\right|\le\sqrt{\|\bw\|_\infty}.
\end{flalign*}
Assume that $s=0.01$. By Proposition \ref{prop:mediansandwichRWF}, if $c_3$ is sufficiently large (i.e., $c_3>100)$, we still have
\begin{flalign}
0.5\|\bh\|\le\median \left\{\left|\sqrt{y_i}-|\ba_i^T\bz|\right|\right\}\le 0.8\|\bh\|.\label{eq:RWFmedianboundtwonoise}
\end{flalign}

Furthermore, recall $\tilde{\cT}^i:=\left\{\left||\ba_i^T\bx|-|\ba_i^T\bz|\right|\le \alpha'_h \median \left\{\left||\ba_i^T\bz|-\sqrt{y_i}\right|\right\}\right\}$. Then,
\begin{flalign}
\nabla \cR_{tr}(\bz) &= \frac{1}{m}\sum_{i=1}^m \left(|\ba_i^T\bz|-\sqrt{y_i}\right) \ba_i \bone_{\cT^i}\nn\\
& = \underbrace{\frac{1}{m}\left(\sum_{i\notin S} \left(|\ba_i^T\bz|-|\ba_i^T\bx|\right) \ba_i \bone_{\cT^i}+\sum_{i\in S} \left(|\ba_i^T\bz|-|\ba_i^T\bx|\right) \ba_i \bone_{ \tilde{\cT}^i}\right)}_{\nabla^{clean}\cR_{tr}(\bz)} \nn\\
&\quad-\underbrace{\frac{1}{m}\sum_{i\notin S}(\sqrt{y_i}-|\ba_i^T\bx|) \ba_i \bone_{\cT^i}}_{\nabla^{noise}\cR_{tr}(\bz)}+\underbrace{\frac{1}{m}\sum_{i\in S} \left( \left(|\ba_i^T\bz|-\sqrt{y_i}\right)\bone_{\cT^i}-\left(|\ba_i^T\bz|-|\ba_i^T\bx|\right) \bone_{ \tilde{\cT}^i} \right)\ba_i }_{\nabla^{extra}\cR_{tr}(\bz)}.\nn
\end{flalign}

 For $i\notin S$, the inclusion property (i.e. $\cT_1^i\subseteq \cT^i \subseteq \cT_2^i$) holds because
 \begin{flalign*}
	\left|\sqrt{y_i}-|\ba_i^T\bz|\right|\in \left||\ba_i^T\bx|-|\ba_i^T\bz|\right|\pm \sqrt{|w_i|}
\end{flalign*}
and $\sqrt{|w_i|}\le \frac{1}{c_3}\|\bh\|$ for some sufficient large $c_3$. For $i\in S$, the inclusion $\cT_1^i\subseteq \tilde{\cT}^i \subseteq \cT_2^i$ holds because of \eqref{eq:RWFmedianboundtwonoise}. All the proof arguments for Propositions \ref{prop:strongconvexRWF} and \ref{prop:localsmoothRWF} are also valid for $\nabla^{clean} \cR_{tr}(\bz)$, and thus we have
\begin{flalign*}
	&\left\langle \nabla^{clean} \cR_{tr}(\bz), \bh \right\rangle \ge \left(0.88-\zeta'_1-\zeta'_2-\epsilon\right)\|\bh\|^2,\\
	&\left\|\nabla^{clean}\cR_{tr}(\bz)\right\|\le  (1.8+\delta)\|\bh\|.
\end{flalign*}

Next, we turn to control the contribution of the noise. Let $\tilde{w}_i=(\sqrt{y_i}-|\ba_i^T\bx|)\bone_{\cT^i}$. Then $|\tilde{w}_i|<\sqrt{|w_i|}$ and we have
\begin{flalign*}
\|\nabla^{noise}\cR_{tr}(\bz)\|=\left\|\frac{1}{m}\bA^T\tilde{\bw}\right\|\le \left\|\frac{1}{\sqrt{m}}\bA^T\right\|\left\|\frac{\tilde{\bw}}{\sqrt{m}}\right\|\le (1+\delta)\|\tilde{\bw}\|_\infty\le (1+\delta)\sqrt{\|\bw\|_\infty},
\end{flalign*}
when $m/n$ is sufficiently large. Given the regime condition $\|\bh\|\ge c_3\sqrt{\|\bw\|_\infty}$, we further have
\begin{flalign*}
&\|\nabla^{noise}\cR_{tr}(\bz)\|\le  \frac{(1+\delta)}{c_3}\|\bh\|,\\
& \left|\left\langle \nabla^{noise} \cR_{tr}(\bz), \bh\right\rangle\right| \le   \left\| \nabla^{noise} \cR_{tr}(\bz)\right\|\cdot\|\bh\| \le \frac{(1+\delta)}{c_3}\|\bh\|^2.
\end{flalign*}
We next bound the contribution of $\nabla^{extra}\cR_{tr}(\bz)$. Introduce $\bq=[q_1,\ldots,q_m]^T$, where
\begin{flalign*}
	q_i:=\left( (|\ba_i^T\bz|-\sqrt{y_i})\bone_{\cT^i}-(|\ba_i^T\bz|-|\ba_i^T\bx|) \bone_{\tilde{\cT}^i} \right)\bone_{\{i \in S\}}.
\end{flalign*}
Then $|q_i|\le 1.6\alpha'_h\|\bh\|$, and $\|\bq\|\le \sqrt{sm}\cdot 1.6\alpha'_h\|\bh\|$. We thus have
\begin{flalign*}
	&\left\|  \nabla^{extra} \cR_{tr}(\bz)\right\|=\frac{1}{m}\left\|\bA^T\bq\right\|\le 1.6(1+\delta)\sqrt{s}\alpha'_h\|\bh\|,\\
	&\left|\left\langle \nabla^{extra} \cR_{tr}(\bz), \bh\right\rangle\right| \le \|\bh\| \cdot \left\| \nabla^{extra} \cR_{tr}(\bz)\right\|\le 1.6(1+\delta)\sqrt{s}\alpha'_h\|\bh\|^2.
\end{flalign*}

Putting these together, one has
\begin{flalign*}
	\left\langle \nabla \cR_{tr}(\bz), \bh \right\rangle & \ge \left\langle \nabla^{clean}\cR_{tr}(\bz),\bh \right\rangle- \left|\left\langle  \nabla^{noise} \cR_{tr}(\bz), \bh\right\rangle\right|- \left|\left\langle  \nabla^{extra} \cR_{tr}(\bz), \bh\right\rangle\right|\nn\\
	&\ge \left(0.88-\zeta'_1-\zeta'_2-\epsilon-(1+\delta)(1/c_3-1.6\sqrt{s}\alpha'_h)\right)\|\bh\|^2,
\end{flalign*}
and
\begin{flalign}
\left\|\nabla \cR_{tr}(\bz)\right\|&\le \left\|\nabla^{clean} \cR_{tr}(\bz)\right\|+\left\|\nabla^{noise} \cR_{tr}(\bz)\right\|+\left\|\nabla^{extra} \cR_{tr}(\bz)\right\| \nn \\
&\le \left(1.8+\delta+(1+\delta)\cdot(1/c_3+1.6\sqrt{s}\alpha'_h)\right) \|\bh\|.
\end{flalign}

Thus, the $\mathsf{RC}$ is guaranteed if $\mu,\lambda,\epsilon$ are chosen properly, $c_0, c_3$ are sufficiently large and $s$ is sufficiently small.

$\bullet$ \textbf{Regime 2}: Once the iterate enters this regime with $\|\bh\|\le c_3\sqrt{\|\bw\|_\infty}$, each gradient iterate may not reduce the estimation error. However, in this regime each move size $\mu\nabla \cR_{tr}(\bz)$ is at most $\mathcal{O}(\sqrt{\|\bw\|_\infty})$. Then the estimation error cannot increase by more than $\sqrt{\|\bw\|_\infty}$ with a constant factor. Thus one has
\begin{flalign}
	\dist\left(\bz-\mu\nabla \cR_{tr}(\bz),\bx\right)\le c_5\sqrt{\|\bw\|_\infty}
\end{flalign}
for some constant $c_5$. As long as $\sqrt{\|\bw\|_\infty}$ is sufficiently small, it is guaranteed that $c_5\sqrt{\|\bw\|_\infty}\le c_4\|\bx\|$. If the iterate jumps out of \emph{Regime 2}, it falls into \emph{Regime 1}.

\section{Conclusions}\label{sec:conclusion}
In this paper, we propose provably effective approaches, median-TWF and median-RWF, for phase retrieval when the measurements are corrupted by sparse outliers that can take arbitrary values. Our strategy is to apply gradient descent with respect to carefully chosen loss functions, where both the initialization and the search directions are pruned guided by the sample median. We show that both algorithms allow exact recovery even with a constant proportion of arbitrary outliers for robust phase retrieval using a near-optimal number of measurements up to a logarithmic factor. Our algorithm performs well for phase retrieval problem under sparse corruptions. We anticipate that the technique developed in this paper will be useful for designing provably robust algorithms for other inference problems under sparse corruptions.  

\vspace{0.5in}
%
%
%
%
%
\appendix
%
\noindent {\Large \textbf{Appendix}}

\section{Proof of Properties of Median}\label{supp:medianlemmas}

\subsection{Proof of Lemma \ref{lem:quantileconcentration}} \label{proof_lem:quantileconcentration}
For simplicity, denote $\theta_p:=\theta_p(F)$ and $\hat{\theta}_{p}:=\theta_p(\{X_i\}_{i=1}^m)$.
Since $F'$ is continuous and positive, for an $\epsilon$, there exists a constant $\delta_1$ such that $\bbP(X\le\theta_p-\epsilon)=p-\delta_1$, where $\delta_1\in (\epsilon l, \epsilon L)$. Then one has
\begin{flalign*}
\bbP\left(\hat{\theta}_{p}<\theta_p-\epsilon\right)&\stackrel{(a)}{=} \bbP\left(\sum_{i=1}^{m}\bone_{\{X_i\le \theta_p-\epsilon\}}\ge pm\right)=\bbP\left(\frac{1}{m} \sum_{i=1}^{m}\bone_{\{X_i\le \theta_p-\epsilon\}}\ge(p-\delta_1)+\delta_1\right)\\
&\stackrel{(b)}{\le} \exp(-2m\delta_1^2)\le \exp(-2m\epsilon^2l^2),
\end{flalign*}
where (a) is due to the definition of the quantile function in \eqref{eq:quantiledef} and (b) is due to the fact that $\bone_{\{X_i\le \theta_p-\epsilon\}}\sim \mbox{Bernoulli}(p-\delta_1)$ i.i.d., followed by the Hoeffding inequality.
Similarly, one can show for some $\delta_2\in (\epsilon l, \epsilon L)$,
\begin{flalign*}
\bbP\left(\hat{\theta}_{p}>\theta_p+\epsilon\right)\le \exp(-2m\delta_2^2)\le \exp(-2m\epsilon^2l^2).
\end{flalign*}
Combining these two inequalities, one has the conclusion.

\subsection{Proof of Lemma \ref{lem:orderstats}} \label{proof_lem:orderstats}

It suffices to show that
\begin{equation}
	|X_{(k)}-Y_{(k)}|\le\max_l|X_l-Y_l|, \quad \forall k=1,\cdots, n.
\end{equation}

%
%
%
%
%

Case 1: $k=n$, suppose $X_{(n)}= X_i$ and $Y_{(n)}= Y_j$, i.e., $X_i$ is the largest among $\{X_l\}_{l=1}^n$ and $Y_j$ is the largest among  $\{Y_l\}_{l=1}^n$. Then we have either $X_j\le X_i\le Y_j$ or $Y_i\le Y_j \le X_i$. Hence,
\begin{flalign*}
	|X_{(n)}-Y_{(n)}|=|X_i-Y_j|\le\max\{|X_i-Y_i|, |X_j-Y_j|\}.
\end{flalign*}

Case 2: $k=1$, suppose that $X_{(1)}= X_i$ and $Y_{(1)}= Y_j$. Similarly
\begin{flalign*}
	|X_{(1)}-Y_{(1)}|=|X_i-Y_j|\le\max\{|X_i-Y_i|, |X_j-Y_j|\}.
\end{flalign*}

Case 3: $1<k<n$,  suppose that  $X_{(k)}= X_i$, $Y_{(k)}= Y_j$, and without loss of generality assume that $X_i<Y_j$ (if $X_i=Y_j$, $	0= |X_{(k)}-Y_{(k)}|\le\max_l|X_l-Y_l|$ holds trivially). We show the conclusion by contradiction.

Assume $|X_{(k)}-Y_{(k)}|>\max_l |X_l-Y_l|$. Then one must have $Y_i<Y_j$ and $X_j>X_i$ and $i\neq j$.
 Moreover for any $p<k$ and $q>k$, the index of $X_{(p)}$ cannot be equal to the index of $Y_{(q)}$; otherwise the assumption is violated.

Thus, all $Y_{(q)}$ for $q>k$ must share the same index set with $X_{(p)}$ for $p>k$. However, $X_j$, which is larger than $X_{i}$ (thus if $X_j=X_{(k')}$, then $k'>k$), shares the same index with $Y_{j}$, where $Y_j=Y_{(k)}$. This yields contradiction.

\subsection{Proof of Lemma \ref{lem:empiricalmedian}}\label{proof_lem:empiricalmedian}
Assume that $sm$ is an integer. Since there are $sm$ corrupted samples in total, one can select at least $\left\lceil(p-s)m\right\rceil$ clean samples from the left $p$ portion of ordered contaminated samples $\{\theta_{1/m}(\{X_i\}), \theta_{2/m}(\{X_i\}),\cdots,\theta_{p}(\{X_i\})\}$. Thus one has the left inequality. Furthermore, one can also select out at least $\left\lceil(1-p-s)m\right\rceil$ clean samples from the right $1-p$ portion of ordered contaminated samples $\{\theta_{p}(\{X_i\}),\cdots,\theta_{1}(\{X_i\})\}$. One has the right inequality. 

\subsection{Proof of Lemma \ref{lem:product}} \label{proof_lem:product}

First we introduce some general facts for the distribution of the product of two correlated standard Gaussian random variables \cite{donahue1964products}. Let $u\sim\mathcal{N}(0,1)$, $v\sim\mathcal{N}(0,1)$, and their correlation coefficient be $\rho\in[-1,1]$. Then the density of $uv$ is given by
\begin{flalign*}
	\phi_{\rho}(x)=\frac{1}{\pi \sqrt{1-\rho^2}}\exp\left(\frac{\rho x}{1-\rho^2}\right) K_0\left(\frac{|x|}{1-\rho^2}\right), \quad x\neq 0,
\end{flalign*}
where $K_0(\cdot)$ is the modified Bessel function of the second kind.
Thus the density of $r=|uv|$ is
\begin{flalign}\label{eq:psi_rho}
	\psi_{\rho}(x)=\frac{1}{\pi \sqrt{1-\rho^2}}\left[\exp\left(\frac{\rho x}{1-\rho^2}\right) +\exp\left(-\frac{\rho x}{1-\rho^2}\right) \right]K_0\left(\frac{|x|}{1-\rho^2}\right),\quad  x> 0,
\end{flalign}
for $|\rho|<1$.
If $|\rho|=1$, $r$ becomes a $\chi_1^2$ random variable, with the density
\begin{flalign*}
	\psi_{|\rho|=1}(x)=\frac{1}{\sqrt{2\pi}}x^{-1/2}\exp(-x/2), \quad x> 0.
\end{flalign*}
It can be seen from \eqref{eq:psi_rho} that the density of $r$ only relates to the correlation coefficient $\rho\in[-1,1]$. 

Let $\theta_{1/2}(\psi_\rho)$ be the $1/2$ quantile (median) of the distribution $\psi_\rho(x)$, and $\psi_\rho(\theta_{1/2})$ be the value of the function $\psi_\rho$ at the point $\theta_{1/2}(\psi_\rho)$. Although it is difficult to derive the analytical expressions of $\theta_{1/2}(\psi_\rho)$ and $\psi_\rho(\theta_{1/2})$ due to the complicated form of $\psi_\rho$ in \eqref{eq:psi_rho}, due to the continuity of $\psi_\rho(x)$ and $\theta_{1/2}(\psi_\rho)$,  we can calculate them numerically, as illustrated in Figure~\ref{fig:quantiles}.
\begin{figure}[thb]
\begin{center}
\includegraphics[width=0.45\textwidth]{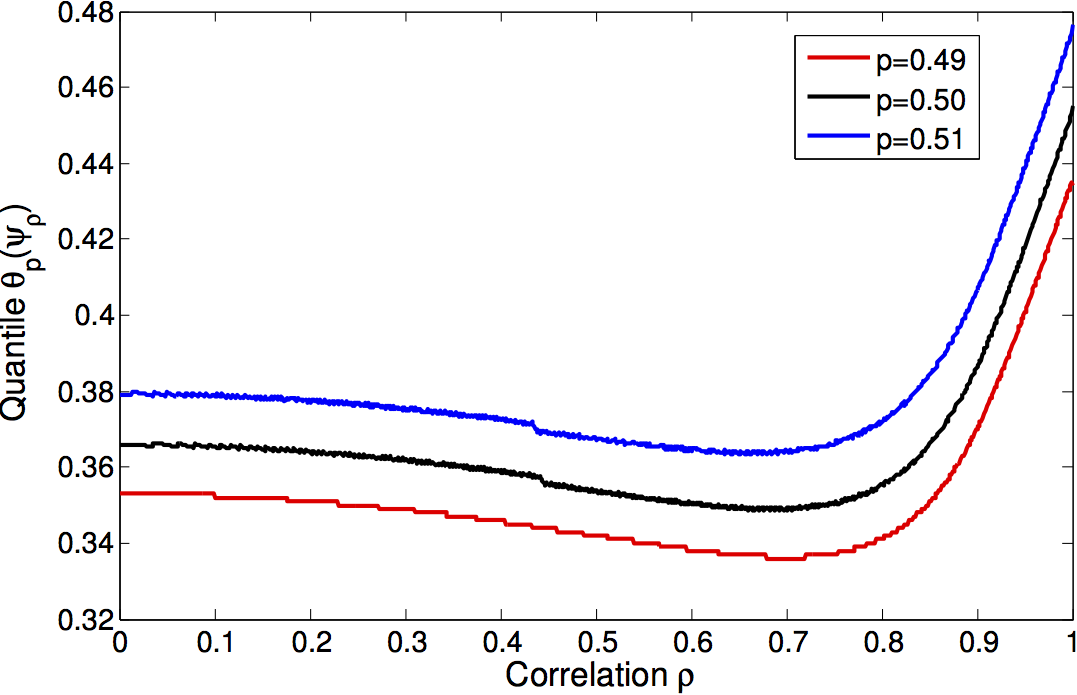}
\includegraphics[width=0.47\textwidth]{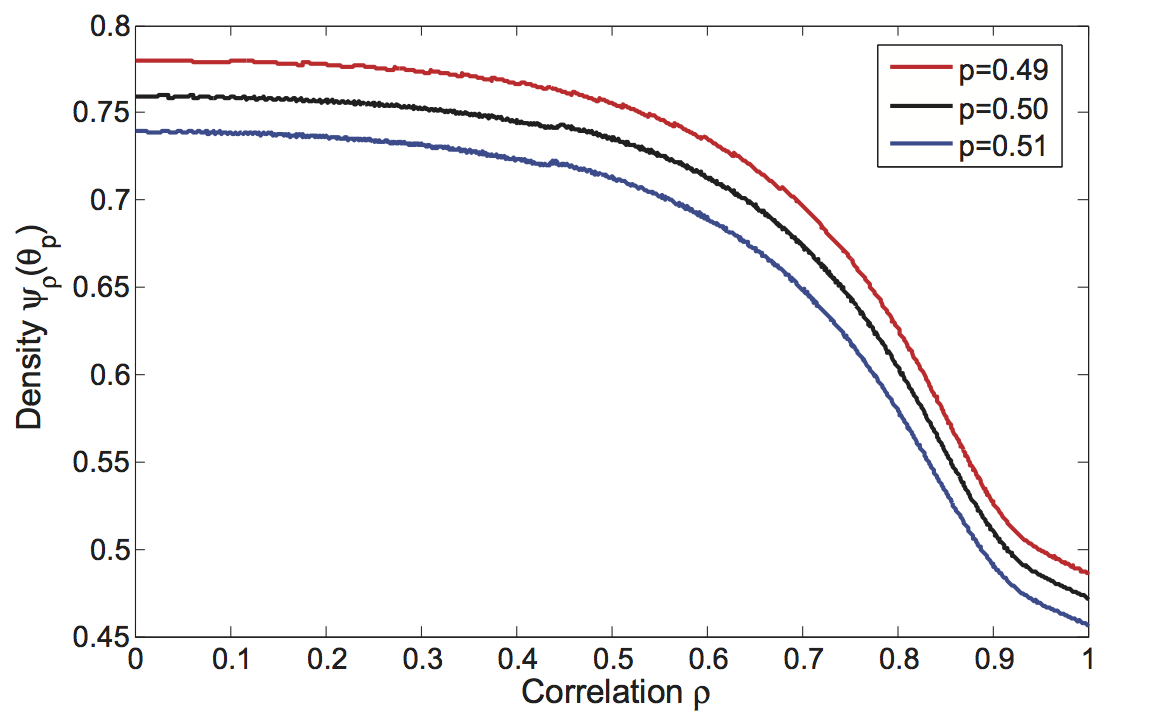}
\caption{Quantiles and density at quantiles of $\psi_\rho(x)$ across $\rho$.}
\label{fig:quantiles}
\end{center}
\end{figure}
From the numerical calculation, one can see that both $\psi_{\rho}(\theta_{1/2})$ and  $\theta_{1/2}(\psi_{\rho})$ are bounded from below and above for all $\rho\in[0,1]$ ($\psi_{\rho}(\cdot)$ is symmetric over $\rho$, hence it is sufficient to consider $\rho\in [0,1]$), satisfying
\begin{flalign}
	0.348<\theta_{1/2}(\psi_{\rho})<0.455, \quad 0.47<\psi_{\rho}(\theta_{1/2})<0.76.
\end{flalign}


\section{Proof of Proposition~\ref{prop:initialization}}\label{supp:initialization}

Denote $\tilde{y}_i:=|\ba_i^T\bx|^2+w_i$ for convenience. We first bound the concentration of $\median(\{y_i\})$, also denoted by $\theta_{\frac{1}{2}}(\{y_i\})$.
 Lemma \ref{lem:empiricalmedian} yields
\begin{flalign}
	\theta_{\frac{1}{2}-s}(\{\tilde{y}_i\})<\theta_{\frac{1}{2}} (\{y_i\}) <\theta_{\frac{1}{2}+s}(\{\tilde{y}_i\}).
\end{flalign}
Moreover, Lemma \ref{lem:orderstats} indicates that
\begin{flalign}
	&\theta_{\frac{1}{2}-s} (\{\tilde{y}_i\})\ge\theta_{\frac{1}{2}-s}(\{|\ba_i^T\bx|^2\})-\|\bw\|_\infty, \\
	&\theta_{\frac{1}{2}+s}(\{\tilde{y}_i\})\le\theta_{\frac{1}{2}+s}(\{|\ba_i^T\bx|^2\})+\|\bw\|_\infty.
\end{flalign}

Observe that $\ba_i^T\bx=\tilde{a}^2_{i1}\|\bx\|^2$, where $\tilde{a}_{i1}=\ba_i^T\bx/\|\bx\|$ is a standard Gaussian random variable. Thus $|\tilde{a}_{i1}|^2$ is a $\chi_1^2$ random variable, whose cumulative distribution function is denoted as $K(x)$.
Moreover by Lemma \ref{lem:quantileconcentration},  for a small $\epsilon$, one has	$\left|\theta_{\frac{1}{2}-s}(\{|\tilde{a}_{i1}|^2\})- \theta_{\frac{1}{2}-s}(K)\right|<\epsilon$ and $\left|\theta_{\frac{1}{2}+s}(\{|\tilde{a}_{i1}|^2\}) -\theta_{\frac{1}{2}+s}(K)\right|<\epsilon$ with probability $1-2\exp(-cm\epsilon^2)$ and $c$ is a constant around $2\times 0.47^2$ (see Figure~\ref{fig:quantiles}). We note that $\theta_{\frac{1}{2}}(K)=0.455$ and both $\theta_{\frac{1}{2}-s}(K)$ and $\theta_{\frac{1}{2}+s}(K)$ can be arbitrarily close to $\theta_{\frac{1}{2}}(K)$ simultaneously as long as $s$ is small enough (independent of $n$). Thus, one has
\begin{flalign}
\left(\theta_{\frac{1}{2}-s}(K)-\epsilon-c\right)\|\bx\|^2 <\theta_{\frac{1}{2}} (\{y_i\})<\left(\theta_{\frac{1}{2}+s}(K)+\epsilon+c\right)\|\bx\|^2, \label{eq:samplemedianbound}
\end{flalign}
with probability at least $1-\exp(-cm\epsilon^2)$. For the sake of simplicity, we introduce two new notations $\zeta_s:=\theta_{\frac{1}{2}-s}(K)$ and $\zeta^s:=\theta_{\frac{1}{2}+s}(K)$.
Specifically for the instance of $s=0.01$, one has $\zeta_s=0.434$ and $\zeta^s=0.477$. It is easy to see that $\zeta^s-\zeta_s$ can be arbitrarily small if $s$ is small enough.

We next estimate the direction of $\bx$, assuming $\|\bx\|=1$. On the event that \eqref{eq:samplemedianbound} holds, the truncation function has the following bounds,
\begin{flalign*}
&\bone_{\{y_i\le \alpha_y^2 \theta_{1/2} (\{y_i\})/0.455\}}\le \bone_{\left\{y_i\le \alpha_y^2 \left(\zeta^s+\epsilon\right)/0.455\right\}}\le \bone_{\left\{(\ba_i^T\bx)^2\le \alpha_y^2 \left(\zeta^s+\epsilon+c\right)/0.455\right\}}\\
&\bone_{\{y_i\le \alpha_y^2 \theta_{1/2} (\{y_i\})/0.455\}}\ge \bone_{\left\{y_i\le \alpha_y^2 \left(\zeta_s-\epsilon\right)/0.455\right\}}\ge \bone_{\left\{(\ba_i^T\bx)^2\le \alpha_y^2 \left(\zeta_s-\epsilon-c\right)/0.455\right\}}.
\end{flalign*}

On the other hand, denote the support of the outliers as $S$, and we have
\begin{flalign*}
	\bY&= \frac{1}{m} \sum_{i\notin S} \ba_i\ba_i^T \tilde{y}_i \bone_{\{(\ba_i^T\bx)^2\le \alpha_y^2 \theta_{1/2} (\{y_i\})/0.455\}} +\frac{1}{m} \sum_{i\in S} \ba_i\ba_i^T y_i \bone_{\{y_i\le \alpha_y^2 \theta_{1/2} (\{y_i\})/0.455\}}.
\end{flalign*}

Consequently, one can bound $\bY$ as
\begin{flalign*}
\bY_1:=&\frac{1}{m} \sum_{i\notin S} \ba_i\ba_i^T (\ba_i^T\bx)^2 \bone_{\{(\ba_i^T\bx)^2\le \alpha_y^2 (\zeta_s-\epsilon-c)/0.455\}}-c\cdot \frac{1}{m} \sum_{i\notin S} \ba_i\ba_i^T
\preceq \bY\\
 \preceq &\frac{1}{m} \sum_{i\notin S} \ba_i\ba_i^T (\ba_i^T\bx)^2 \bone_{\{(\ba_i^T\bx)^2\le \alpha_y^2 (\zeta^s+\epsilon+c)/0.455\}} +c\cdot \frac{1}{m} \sum_{i\notin S} \ba_i\ba_i^T +\frac{1}{m} \sum_{i\in S} \ba_i\ba_i^T  \alpha_y^2 (\zeta^s+\epsilon+c)/0.455=:\bY_2,
\end{flalign*}
where we have
\begin{flalign}
	\mE[\bY_1]=(1-s)(\beta_1 \bx\bx^T+\beta_2 \bI-c\bI), \quad \mE[\bY_2]=(1-s)(\beta_3 \bx\bx^T +\beta_4 \bI+c \bI)+s\alpha_y^2\frac{(\zeta^s+\epsilon)}{0.455} \bI,
\end{flalign}
with 
\begin{flalign*}
&\beta_1:=\mE\left[\xi^4\bone_{\left\{|\xi|\le \alpha_y\sqrt{ (\zeta_s-\epsilon-c)/0.455}\right\}}\right]-\mE\left[\xi^2\bone_{\left\{|\xi|\le \alpha_y\sqrt{ (\zeta_s-\epsilon-c)/0.455}\right\}}\right]\\
& \beta_2:=\mE\left[\xi^2\bone_{\left\{|\xi|\le \alpha_y\sqrt{ (\zeta_s-\epsilon-c)/0.455}\right\}}\right]\\
&\beta_3:=\mE\left[\xi^4\bone_{\left\{|\xi|\le \alpha_y\sqrt{ (\zeta^s+\epsilon+c)/0.455}\right\}}\right]-\mE\left[\xi^2\bone_{\left\{|\xi|\le \alpha_y\sqrt{ (\zeta^s+\epsilon+c)/0.455}\right\}}\right]\\
&\beta_4:=\mE\left[\xi^2\bone_{\left\{|\xi|\le \alpha_y\sqrt{ (\zeta^s+\epsilon+c)/0.455}\right\}}\right]
\end{flalign*}
where $\xi\sim\cN(0,1)$.


Applying standard results on random matrices with non-isotropic sub-Gaussian rows \cite[equation (5.26)]{Vershynin2012} and noticing that $\ba_i\ba_i^T(\ba_i^T\bx)^2\bone_{\{|\ba_i^T\bx|\le c\}}$ can be rewritten as $\bb_i\bb_i^T$ where $\bb_i:=\ba_i(\ba_i^T\bx)\bone_{\{|\ba_i^T\bx|\le c\}}$ is sub-Gaussian, one can obtain
\begin{flalign}
	\|\bY_1-\mE[\bY_1]\|\le \delta, \quad \|\bY_2-\mE[\bY_2]\|\le \delta
\end{flalign}
with probability $1-\exp(-\Omega(m))$, provided that $m/n$ exceeds some large constant. Furthermore, when $\epsilon$, $c$ and $s$ are sufficiently small, one further has $\|\mE[\bY_1]-\mE[\bY_2]\|\le \delta$.
Putting these together, one has
\begin{flalign}
\|\bY- (1-s)(\beta_1 \bx\bx^T+\beta_2 \bI-c\bI) \|\le 3\delta.
\end{flalign}
Let $\tilde{\bz}^{(0)}$ be the normalized leading eigenvector of $\bY$.
Repeating the same argument as in \cite[Section 7.8]{candes2015phase} and taking $\delta, \epsilon$ to be sufficiently small, one has
\begin{flalign}
	\dist(\tilde{\bz}^{(0)},\bx)\le \tilde{\delta},
\end{flalign}
for a given $\tilde{\delta}>0$, as long as $m/n$ exceeds some large constant.

Furthermore let $\bz^{(0)}=\sqrt{\median{\{y_i\}}/0.455}\tilde{\bz}^{(0)}$ to handle cases $\|\bx\|\neq 1$. By the bound \eqref{eq:samplemedianbound}, one has
\begin{flalign}
\left|\frac{\median(\{y_i\})}{0.455}-\|\bx\|^2\right|\le\max\left\{\left|\frac{\zeta_s-\epsilon-c}{0.455}-1\right|, \left|\frac{\zeta^s+\epsilon+c}{0.455}-1\right|\right\}\|\bx\|^2\le \frac{\zeta^s-\zeta_s+2\epsilon+2c}{0.455}\|\bx\|^2.
\end{flalign}
Thus
\begin{flalign*}
 \dist(\bz^{(0)},\bx)\le \frac{\zeta^s-\zeta_s+2\epsilon+2c}{0.455}\|\bx\|+\tilde{\delta}\|\bx\| \leq \frac{1}{11}\|\bx\|
\end{flalign*}
as long as $s$ and $c$ are small enough constants.
\end{proof}

\section{Supporting Proofs for median-TWF}\label{app:cor:noisefree}


\subsection{Proof of Proposition \ref{prop:mediansandwich}}\label{app:SamMedCon}

We show that the sample median used in the truncation rule concentrates at the level $\|\bz-\bx\|\|\bz\|$. Along the way, we also establish that the sample quantiles around the median are also concentrated at the level $\|\bz-\bx\|\|\bz\|$.

We first show that for a fixed pair $\bz$ and $\bx$, \eqref{eq:median_sandwichTWF} holds with high probability. For simplicity of notation, we let $\bh:=\bz-\bx$. Let $r_i=| (\ba_i^T\bx)^2 - (\ba_i^T\bz)^2 |$. Then $r_i$'s are i.i.d. copies of a random variable $r$, where $r=|(\ba^T\bx)^2- (\ba^T\bz)^2  |$ with the entries of $\ba$ composed of i.i.d. standard Gaussian random variables. Note that the distribution of $r$ is fixed once given $\bh$ and $\bz$. Let $\bx(1)$ denote the first element of a generic vector $\bx$, and $\bx_{-1}$ denote the remaining vector of $\bx$ after eliminating the first element. Let $\bU_h$ be an orthonormal matrix with first row being $\bh^T/\|\bh\|$, $\tilde{\ba}=\bU_h \ba$, and $\tilde{\bz}=\bU_h \bz$. Similarly, define $\bU_{\tilde{z}_{-1}}$ and let $\tilde{\bb}=\bU_{\tilde{z}_{-1}}\tilde{\ba}_{-1}$. Then $\tilde{\ba}(1)$ and $\tilde{\bb}(1)$ are independent standard normal random variables. We further express $r$ as follows.
  \begin{flalign*}
	r&=|(\ba^T\bz)^2-(\ba^T\bx)^2|\\
	&=|(2\ba^T\bz-\ba^T\bh)(\ba^T\bh)|\\
	&=|(2\tilde{\ba}^T\tilde{\bz}-\tilde{\ba}(1)\|\bh\|)(\tilde{\ba}(1)\|\bh\|)|\\
	&=|(2\bh^T\bz-\|\bh\|^2)\tilde{\ba}(1)^2+2(\tilde{\ba}_{-1}^T\tilde{\bz}_{-1})(\tilde{\ba}(1)\|\bh\|)|\\
	&=|(2\bh^T\bz-\|\bh\|^2)\tilde{\ba}(1)^2+2\tilde{\bb}(1)\|\tilde{\bz}_{-1}\|\tilde{\ba}(1)\|\bh\||\\
	&=|(2\bh^T\bz-\|\bh\|^2)\tilde{\ba}(1)^2+2\sqrt{\|\bz\|^2-\tilde{\bz}(1)^2}\tilde{\ba}(1)\tilde{\bb}(1)\|\bh\||\\
	 &=\left|\left(2\frac{\bh^T\bz}{\|\bh\|\|\bz\|}-\frac{\|\bh\|}{\|\bz\|}\right)\tilde{\ba}(1)^2+2\sqrt{1-\left(\frac{\bh^T\bz}{\|\bh\|\|\bz\|}\right)^2}\tilde{\ba}(1)\tilde{\bb}(1)\right|\cdot \|\bh\|\|\bz\|\\
	&=:\left|(2\cos(\omega)-t)\tilde{\ba}(1)^2+2\sqrt{1-\cos^2(\omega)}\tilde{\ba}(1)\tilde{\bb}(1)\right|\cdot\|\bh\|\|\bz\|\\
	&=:|u\tilde{v}|\cdot\|\bh\|\|\bz\|
\end{flalign*}
where $\omega$ is the angle between $\bh$ and $\bz$, and $t=\|\bh\|/\|\bz\|<1/11$. Consequently, $u=\tilde{\ba}(1)\sim\mathcal{N}(0,1)$ and $\tilde{v}=(2\cos(\omega)-t)\tilde{\ba}(1)+2|\sin (\omega)|\tilde{\bb}(1)$ is also a Gaussian random variable with variance $3.6<\Var(\tilde{v})<4$ under the assumption $t<1/11$.

Let $v=\tilde{v}/\sqrt{\Var(\tilde{v})}$, and then $v\sim \mathcal{N}(0,1)$. Furthermore, let $r'=|uv|$. Denote the density function of $r'$ as $\psi_{\rho}(\cdot)$ and the $1/2$-quantile point of $r'$ as $\theta_{1/2}(\psi_{\rho})$. By Lemma \ref{lem:product}, we have
\begin{flalign}
	0.47<\psi_{\rho}(\theta_{1/2})<0.76,\quad 0.348<\theta_{1/2}(\psi_{\rho})<0.455.
\end{flalign}

By Lemma \ref{lem:quantileconcentration}, we have with probability at least $1-2\exp(-c m \epsilon^2)$ (here $c$ is around $2\times0.47^2$),
\begin{flalign}
	0.348-\epsilon<\median (\{r'_i\}_{i=1}^m)<0.455+\epsilon.
\end{flalign}
The same arguments carry over to other quantiles $\theta_{0.49}(\{r'_i\})$ and $\theta_{0.51}(\{r'_i\})$. From Figure. \ref{fig:quantiles}, we observe that for $\rho\in[0,1]$
\begin{flalign}
	0.45<\psi_{\rho}(\theta_{0.49}), \psi_{\rho}(\theta_{0.51})<0.78,\quad 0.336<\theta_{0.49}(\psi_{\rho}),\theta_{0.51}(\psi_{\rho})<0.477
\end{flalign}
and then we have with probability at least  $1-2\exp(-cm\epsilon^2)$ (here $c$ is around $2\times 0.45^2$),
\begin{flalign}
	0.336-\epsilon<\theta_{0.49} (\{r'_m\}),\theta_{0.51}(\{r'_m\})<0.477+\epsilon.
\end{flalign}
Hence, by multiplying by $\sqrt{\Var(\tilde{v})}$,
we have with probability $1-2\exp(-cm\epsilon^2)$,
\begin{flalign}
	(0.65-\epsilon) \|\bz-\bx\|\|\bz\|\le \median\left( \left\{|(\ba_i^T\bz)^2-(\ba_i^T\bx)^2|\right\}\right)\le  (0.91+\epsilon)\|\bz-\bx\|\|\bz\|,\label{eq:medianbeforenet}\\
	(0.63-\epsilon) \|\bz-\bx\|\|\bz\|\le \theta_{0.49},\theta_{0.51} \left(\left\{|(\ba_i^T\bz)^2-(\ba_i^T\bx)^2|\right\}\right)\le (0.95+\epsilon) \|\bz-\bx\|\|\bz\|.\label{eq:quantilebeforenet}
\end{flalign}
We note that, to keep notation simple, $c$ and $\epsilon$ may vary line by line within constant factors.

Up to now, we prove that for any fixed  $\bz$ and $\bx$, the median or neighboring quantiles of $\left\{|(\ba_i^T\bz)^2-(\ba_i^T\bx)^2|\right\}$ are upper and lower bounded by $\|\bz-\bx\|\|\bz\|$ times constant factors. To prove \eqref{eq:median_sandwichTWF} for all $\bz$ and $\bx$ with $\|\bz-\bx\|\le \frac{1}{11} \|\bz\|$, we use the net covering argument. Still we argue for median first and  the same arguments carry over to other quantiles.

To proceed, we restate \eqref{eq:medianbeforenet} as
\begin{flalign}
	(0.65-\epsilon)\le \median\left( \left\{\left|\left(\frac{2(\ba_i^T\bz)}{\|\bz\|}-\frac{\ba_i^T\bh}{\|\bh\|}\frac{\|\bh\|}{\|\bz\|}\right)\frac{\ba_i^T\bh}{\|\bh\|}\right|\right\}\right)\le  (0.91+\epsilon)
\end{flalign}
holds with probability at least $1-2\exp(-cm\epsilon^2)$ for a given pair $\bh, \bz$ satisfying $\|\bh\|/\|\bz\|\le 1/11$.

Let $\tau=\epsilon/(6n+6m)$, let $\mathcal{S}_\tau$ be a $\tau$-net covering the unit sphere, $\mathcal{L}_\tau$ be a $\tau$-net covering a line with length $1/11$, and set
\begin{flalign}
	\mathcal{N}_\tau=\{(\bz_0,\bh_0, t_0): (\bz_0,\bh_0, t_0)\in \mathcal{S}_\tau\times \mathcal{S}_\tau\times \mathcal{L}_\tau\}.
\end{flalign}
One has cardinality bound (i.e., the upper bound on the covering number) $|\mathcal{N}_\tau|\le (1+2/\tau)^{2n}/(11\tau)<(1+2/\tau)^{2n+1}$. Taking the union bound, we have
\begin{flalign}
	(0.65-\epsilon)\le \median\left( \left\{|2(\ba_i^T\bz_0)-(\ba_i^T\bh_0)t_0||\ba_i^T\bh_0|\right\}\right)\le  (0.91+\epsilon), \quad \forall (\bz_0, \bh_0, t_0)\in\cN_\epsilon \label{eq:medianonnet}
\end{flalign}
with probability at least $1-(1+2/\tau)^{2n+1}\exp(-cm\epsilon^2)$.  

We next argue that \eqref{eq:medianonnet} holds with probability $1-c_1\exp(-c_2 m\epsilon^2)$ for some constants $c_1, c_2$ as long as $m\ge c_0 (\epsilon^{-2}\log\epsilon^{-1}) n\log n$ for sufficiently large constant $c_0$. To prove this claim, we first observe
\begin{flalign*}
(1+2/\tau)^{2n+1}\asymp \exp(2n(\log (n+m) +\log 12+\log (1/\epsilon)))\asymp \exp(2n(\log m)).
\end{flalign*}

We note that once $\epsilon$ is chosen, it is fixed in the whole proof and does not scale with $m$ or $n$. For simplicity, assume that $\epsilon<1/e$.  Fix some positive constant $c'<c-c_2$.  It then suffices to show that there exists a large constant $c_0$ such that if $m\ge c_0 (\epsilon^{-2}\log\epsilon^{-1}) n\log n$, then
\begin{flalign}\label{eq:logm}
2n\log m<c'm\epsilon^2.
\end{flalign}
For any fixed $n$, if \eqref{eq:logm} holds for some $m$ and $m>(2/c')\epsilon^{-2}n$, then \eqref{eq:logm} always holds for larger $m$, because
\begin{flalign*}
2n\log (m+1)&=2n\log m+2n(\log (m+1)-\log m)=2n \log m+\frac{2n}{m}\log (1+\frac{1}{m})^m\\
&\le 2n\log m+\frac{2n}{m}
\le c'm\epsilon^2+c'\epsilon^2=c'(m+1)\epsilon^2.
\end{flalign*}
Next, for any $n$, we can always find a constant $c_0$ such that \eqref{eq:logm} holds for $m=c_0(\epsilon^{-2}\log\epsilon^{-1}) n\log n$. Such $c_0$ can be easily found for large $n$. For example, $c_0=4/c'$ is a valid option if 
\begin{equation}\label{eq:largen}
(4/c')(\epsilon^{-2}\log\epsilon^{-1}) n\log n<n^2.
\end{equation}
Moreover, since the number of $n$ that violates \eqref{eq:largen} is finite, the maximum over all such $c_0$ serves the purpose.

Next, one needs to bound $$\left|\median\left( \left\{|2(\ba_i^T\bz_0)-(\ba_i^T\bh_0)t_0||\ba_i^T\bh_0|\right\}\right) - \median\left( \left\{|2(\ba_i^T\bz)-(\ba_i^T\bh)t||\ba_i^T\bh|\right\}\right)\right|$$
for any $\|\bz-\bz_0\|<\tau, \|\bz-\bz_0\|<\tau$ and $\|t-t_0\|<\tau$.

By Lemma \ref{lem:orderstats} and  the inequality $\big| |x|-|y|\big|\le|x-y|$, we have
\begin{flalign*}
&\left|\median\left( \left\{|2(\ba_i^T\bz_0)-(\ba_i^T\bh_0)t_0||\ba_i^T\bh_0|\right\}\right) - \median\left( \left\{|2(\ba_i^T\bz)-(\ba_i^T\bh)t||\ba_i^T\bh|\right\}\right)\right|\\
&\le \max_{i\in [m]} \left|\left( 2(\ba_i^T\bz_0)-(\ba_i^T\bh_0)t_0\right)(\ba_i^T\bh_0) - \left(2(\ba_i^T\bz)-(\ba_i^T\bh)t\right)(\ba_i^T\bh)\right|\\
&\le \max_{i\in [m]} \left|\left( 2(\ba_i^T\bz_0)-(\ba_i^T\bh_0)t_0\right)(\ba_i^T\bh_0) - \left(2(\ba_i^T\bz)-(\ba_i^T\bh)t\right)(\ba_i^T\bh_0)\right|\\
&\quad+ \max_{i\in [m]} \left|\left( 2(\ba_i^T\bz)-(\ba_i^T\bh)t\right)(\ba_i^T\bh_0) - \left(2(\ba_i^T\bz)-(\ba_i^T\bh)t\right)(\ba_i^T\bh)\right|\\
&\le \max_{i\in [m]} \left(\left|2\ba_i^T(\bz_0-\bz)\right|+\left|(\ba_i^T\bh_0)t_0-(\ba_i^T\bh)t\right|\right)\left|\ba_i^T\bh_0\right|+ \max_{i\in [m]} \left|2(\ba_i^T\bz)-(\ba_i^T\bh)t\right||\ba_i^T(\bh_0-\bh)|\\
&\le \max_{i\in[m]}\|\ba_i\|^2(3+t)\tau+\max_{i\in[m]}\|\ba_i\|^2(2+t)\tau\\
&\le \max_{i\in[m]}\|\ba_i\|^2(5+2t)\tau
\end{flalign*}

On the event $E_1:=\left\{\max_{i\in [m]} \|\ba_i\|^2\le m+n\right\}$, one can show that
\begin{flalign}
\left|\median\left( \left\{|2(\ba_i^T\bz_0)-(\ba_i^T\bh_0)t_0||\ba_i^T\bh_0|\right\}\right) - \median\left( \left\{|2(\ba_i^T\bz)-(\ba_i^T\bh)t||\ba_i^T\bh|\right\}\right)\right|<6(m+n)\tau<\epsilon.
\end{flalign}
We claim that $E_1$ holds with probability at least $1-m\exp(-m/8)$ if $m>n$. This can be argued as follows. Note that $\|\ba_i\|^2=\sum_{j=1}^n \ba_i(j)^2$, where $\ba_i(j)$ is the $j$-th element of $\ba_i$. Hence, $\|\ba_i\|^2$ is a sum of $n$ i.i.d. $\chi_1^2$ random variables. Applying the Bernstein-type inequality \cite[Corollary 5.17]{Vershynin2012} and observing that the sub-exponential norm of $\chi_1^2$ is smaller than 2, we have
\begin{flalign}
\bbP\left\{\|\ba_i\|^2\ge m+n\right\} \le \exp(-m/8).
\end{flalign}
Then a union bound concludes the claim.

Further note that \eqref{eq:medianonnet} holds on an event $E_2$, which has probability $1- c_1\exp(-c_2 m \epsilon^2)$ as long as $m\ge c_0(\epsilon^{-2}\log \frac{1}{\epsilon})n\log n$. On the intersection of $E_1$ and $E_2$, inequality for $\theta_{\frac{1}{2}}$ (i.e., median) in \eqref{eq:median_sandwichTWF} holds. 
Such net covering arguments can also carry over to show that inequalities of $\theta_{0.49}$ and $\theta_{0.51}$ in \eqref{eq:median_sandwichTWF} also hold for all $\bx$ and $\bz$ obeying $\|\bx-\bz\|\le \frac{1}{11}\|\bz\|$.

\subsection{Proof of Proposition~\ref{prop:strongconvex}}\label{app:prop:strongconvex}

The proof adapts that of \cite[Proposition 2]{chen2015solving}. We outline the main steps for completeness. Observe that for the noise-free case, $y_i=(\ba_i^T\bx)^2$. We obtain
\begin{flalign}
\nabla \ell_{tr}(\bz)&=\frac{1}{m}\sum_{i=1}^m \frac{(\ba_i^T\bz)^2-(\ba_i^T\bx)^2}{ \ba_i^T\bz}\ba_i \bone_{\cE_1^i\cap \cE_2^i}\nn\\
&=\frac{1}{m}\sum_{i=1}^m 2(\ba_i^T\bh) \ba_i \bone_{\cE_1^i\cap \cE_2^i}-\frac{1}{m}\sum_{i=1}^m\frac{(\ba_i^T\bh)^2}{ \ba_i^T\bz} \ba_i \bone_{\cE_1^i\cap \cE_2^i}.\label{eq:gradtwoterms}
\end{flalign}
One expects the contribution of the second term in \eqref{eq:gradtwoterms} to be small as $\|\bh\|/\|\bz\|$ decreases.

For each $i$, we introduce two new events
\begin{flalign*}
	&\cE_3^i := \{\left|(\ba_i^T\bx)^2-(\ba_i^T\bz)^2\right|\le 0.6 \alpha_h \|\bh\| \cdot |\ba_i^T\bz|\},\\
	&\cE_4^i := \{\left|(\ba_i^T\bx)^2-(\ba_i^T\bz)^2\right|\le 1.0 \alpha_h \|\bh\| \cdot |\ba_i^T\bz|\}.
\end{flalign*}
One the event that Proposition \ref{prop:mediansandwich} holds, the following inclusion property
\begin{flalign}
	\cE_3^i \subseteq \cE_2^i \subseteq \cE_4^i \label{eq:simpleinclusion}
\end{flalign}
is true for all $i$, where $\cE_2^i$ is defined in Algorithm \ref{alg:mtwf}. It is easier to work with these new events because $\cE_3^i$'s (resp. $\cE_4^i$'s) are statistically independent across $i$ for any fixed $\bx$ and $\bz$. To further decouple the quadratic inequalities in $\cE_3^i$ and $\cE_4^i$ into linear inequalities, we introduce two more events and state their properties in the following lemma.
\begin{lemma}[Lemma 3 in \cite{chen2015solving}]\label{le:Devents}
For any $\gamma>0$, define
\begin{flalign}
	\cD_\gamma^i &:= \{\left|(\ba_i^T\bx)^2-(\ba_i^T\bz)^2\right|\le \gamma \|\bh\| |\ba_i^T\bz|\}, \label{eq:eventDi}\\
	\cD_\gamma^{i,1} &:= \left\{ \frac{|\ba_i^T\bh|}{\|\bh\|}\le \gamma  \right\},\label{eq:eventDi1}\\
	\cD_\gamma^{i,2} &:= \left\{\left| \frac{\ba_i^T\bh}{\|\bh\|}-\frac{2\ba_i^T\bz}{\|\bh\|}\right|\le \gamma\right\}. \label{eq:eventDi2}
\end{flalign}
On the event $\cE_{1}^i$ defined in Algorithm \ref{alg:mtwf}, the quadratic inequality specifying $\cD_\gamma^i$ implicates that $\ba_i^T\bh$ belongs to two intervals centered around $0$ and $2\ba_i^T\bz$, respectively, i.e., $\cD_\gamma^{i,1}$ and $\cD_\gamma^{i,2}$. 
The following inclusion property holds
\begin{flalign}
\left(\cD_{\frac{\gamma}{1+\sqrt{2}}}^{i,1}\cap \cE_{1}^i\right)\cup \left(\cD_{\frac{\gamma}{1+\sqrt{2}}}^{i,2}\cap \cE_{1}^i\right)\subseteq \cD_\gamma^i \cap \cE_1^i \subseteq\left(\cD_\gamma^{i,1}\cap \cE_1^i\right)\cup \left(\cD_\gamma^{i,2}\cap \cE_1^i\right). 
\label{eq:maininclusion}
\end{flalign}
\end{lemma}

Specifically, following the two inclusion properties \eqref{eq:simpleinclusion} and \eqref{eq:maininclusion}, we have
\begin{flalign}
\cD_{\gamma_3}^{i,1}\cap \cE_{1,\gamma_3}^i\subseteq \cE_3^i\cap\cE_1^i\subseteq \cE_2^i\cap \cE_1^i \subseteq \cE_4^i\cap\cE_1^i\subseteq (\cD_{\gamma_4}^{i,1}\cup \cD_{\gamma_4}^{i,2})\cap \cE_1^i
\end{flalign}
where the parameters $\gamma_3, \gamma_4$ are given by
\begin{flalign*}
	\gamma_3:= 0.248 \alpha_h, \quad \text{and } \quad \gamma_4:=\alpha_h.
\end{flalign*}

Further using the identity \eqref{eq:gradtwoterms}, we have the following lower bound
\begin{flalign}
\left\langle\nabla \ell_{tr}(\bz),\bh\right\rangle & \ge \frac{2}{m}\sum_{i=1}^m (\ba_i^T\bh)^2 \bone_{\cE_{1}^i\cap \cD_{\gamma_3}^{i,1}} -\frac{1}{m}\sum_{i=1}^m \frac{|\ba_i^T\bh|^3}{|\ba_i^T\bz|} \bone_{\cD_{\gamma_4}^{i,1}\cap \cE_1^i}-\frac{1}{m}\sum_{i=1}^m \frac{|\ba_i^T\bh|^3}{|\ba_i^T\bz|} \bone_{\cD_{\gamma_4}^{i,2}\cap \cE_1^i}.
\label{eq:strongconvex3terms}
\end{flalign}
The three terms in \eqref{eq:strongconvex3terms} can be bounded following Lemmas 4, 5, and 6 in \cite{chen2015solving}, which concludes the proof.

\section{Supporting Proofs for Median-RWF}\label{app:cor:noisefreeRWF}

\subsection{Proof of Proposition~\ref{prop:mediansandwichRWF}}\label{app:prop:mediansandwichRWF}
%

Observe that
\begin{flalign*}
||\ba_i^T\bx|- |\ba_i^T\bz| | =\begin{cases}
|\ba_i^T\bh|, &\quad \text{if } (\ba_i^T\bx)(\ba_i^T\bz)\ge 0;\\
|2\ba_i^T\bx+\ba_i^T\bh|, &\quad \text{if }  (\ba_i^T\bx)(\ba_i^T\bz)<0.
\end{cases}
\end{flalign*}
The following lemma states that $\{(\ba_i^T\bx)(\ba_i^T\bz)<0\}$ are rare events when $\|\bx-\bz\|$ is small. Hence, $\median\left(\left\{\left||\ba_i^T\bx|-|\ba_i^T\bz|\right|\right\}_{i=1}^m\right)$ can be viewed as $\median (\{|\ba_i^T\bh|\}_{i=1}^m)$ with a small perturbation. 
\begin{lemma}\label{lem:rareEvent}
	If $m>c_0 n  $, then with probability at least $1-c_1\exp(-c_2  m)$, 
	\begin{flalign}
	\frac{1}{m}\sum_{i=1}^m \bone_{\{(\ba_i^T\bx)(\ba_i^T\bz)<0\}}<0.05
	\end{flalign}
	holds for all $\bz, \bx$ satisfying $\|\bz-\bx\|<\frac{1}{11}\|\bx\|$.
\end{lemma}
\begin{proof}
	See Appendix \ref{app:lem:rareEvent}.
\end{proof}
By Lemma \ref{lem:empiricalmedian} and Lemma \ref{lem:rareEvent}, we have
\begin{flalign}
\theta_{p-0.05}\left(\{|\ba_i^T\bh|\}\right) \le \theta_p\left( \left\{\left||\ba_i^T\bx|-|\ba_i^T\bz|\right|\right\}\right) \le  \theta_{p+0.05}\left(\{|\ba_i^T\bh|\}\right)   \label{eq:RWFmedianbound}
\end{flalign}
for all $\bx$ and $\bz$ satisfying $\|\bx-\bz\|\le \frac{1}{11}\|\bz\|$ with high probability.

For the model \eqref{eq:outliermodel} with a fraction $s$ of outliers, due to Lemma \ref{lem:empiricalmedian}, we have that
\begin{flalign}\label{eq:RWF_median_outliercase}
\theta_{\frac{1}{2}-s}(\{\left||\ba_i^T\bx|-|\ba_i^T\bz|\right|\})\le \theta_{\frac{1}{2}}(\{|\sqrt{y_i}-|\ba_i^T\bz||\})  \le \theta_{\frac{1}{2}+s}(\{\left||\ba_i^T\bx|-|\ba_i^T\bz|\right|\}).
\end{flalign}
Combining with \eqref{eq:RWFmedianbound}, we obtain that
\begin{flalign}\label{eq:RWF_median_outliercase2}
\theta_{0.45-s}(\{|\ba_i^T\bh|\})\le \theta_{\frac{1}{2}}(\{|\sqrt{y_i}-|\ba_i^T\bz||\})  \le \theta_{0.55+s}(\{|\ba_i^T\bh|\}).
\end{flalign}
Next it suffices to show that $\theta_{0.45-s},\theta_{0.55+s}(\{|\ba_i^T\bh|\})$  are on the order of $ \|\bh\|$ for small $s$.

Let $\tilde{a}_i=|\ba_i^T\bh|/\|\bh\|$. Then $\tilde{a}_i$'s are i.i.d. copies of a \emph{folded standard Gaussian} random variable (i.e., $|\xi|$ where $\xi\sim \cN(0,1))$. We use $\phi(\cdot)$ to denote the density of folded standard Gaussian distribution.

For $s=0.01$, we calculate that 
\begin{flalign}
& \phi(\theta_{0.44})=0.67, \quad \phi(\theta_{0.45})=0.67,  \quad \phi(\theta_{0.55})=0.60, \quad \phi(\theta_{0.56})=0.59\\
	&\theta_{0.44}(\phi)=0.58,\quad \theta_{0.45}(\phi)=0.6, \quad \theta_{0.55}(\phi)=0.76, \quad\theta_{0.56}(\phi)=0.78 .
\end{flalign}
By Lemma \ref{lem:quantileconcentration}, the sample quantiles concentrate on population quantiles. Thus, for any fixed pair $(\bx,\bz)$, 
\begin{flalign}
(0.6-\epsilon)\|\bh\|\le \theta_{1/2} (\{\left||\ba_i^T\bx|-|\ba_i^T\bz|\right|\}_{i=1}^m)\le (0.76+\epsilon)\|\bh\|,
\end{flalign}
holds with probability at least $1-2\exp(-c m \epsilon^{-2})$.

Following the argument of net covering similarly to that in Appendix \ref{app:SamMedCon}, the proposition is proved.
\subsection{Proof of Proposition \ref{prop:strongconvexRWF}}\label{app:prop:strongconvexRWF}
The proof adapts the proof of Proposition 2 in \cite{chen2015solving}. We outline the main steps for completeness. Observe that for the noise-free case, $y_i=|\ba_i^T\bx|$. We obtain
\begin{flalign}
\nabla \cR_{tr}(\bz)&=\frac{1}{m}\sum_{i=1}^m\left( (\ba_i^T\bz)-|\ba_i^T\bx|\cdot\frac{ \ba_i^T\bz}{ |\ba_i^T\bz|}\right)\ba_i \bone_{\cT^i}=\frac{1}{m}\sum_{i\notin \cB} (\ba_i^T\bh) \ba_i \bone_{\cT^i}+\frac{1}{m}\sum_{i\in \cB}(\ba_i^T\bz+\ba_i^T\bx) \ba_i \bone_{\cT^i},\label{eq:RWF_gradtwoterms}
\end{flalign}
where $\cB:=\{i:(\ba_i^T\bx)(\ba_i^T\bz)<0\}$. If $\|\bh\|/\|\bx\|$ is small enough, the cardinality of $\cB$ is small and thus one expects the contribution of the second term in \eqref{eq:RWF_gradtwoterms} to be negligible.

We note that events $\cT^i$ are not statistically independent. To remove such dependency, we introduce two new series of events
\begin{flalign}
&\cT_1^i := \{\left||\ba_i^T\bx|-|\ba_i^T\bz|\right|\le 0.5 \alpha'_h \|\bh\| \},\\
&\cT_2^i :=\{\left||\ba_i^T\bx|-|\ba_i^T\bz|\right|\le 0.8 \alpha'_h \|\bh\| \}.
\end{flalign}
Due to Proposition \ref{prop:mediansandwichRWF}, the following inclusion property
\begin{flalign}
\cT_1^i \subseteq \cT^i \subseteq \cT_2^i \label{eq:RWF_simpleinclusion}
\end{flalign}
holds for all $i$, where $\cT^i$ is defined in Algorithm \ref{alg:RWF_mtwf}. It is easier to work with these new events because $\cT_1^i$'s (resp. $\cT_2^i$'s) are statistically independent for any fixed $\bx$ and $\bz$. Because of the inclusion property \eqref{eq:RWF_simpleinclusion}, we have 
\begin{flalign}
\left\langle\nabla \cR_{tr}(\bz),\bh\right\rangle & \ge \frac{1}{m}\sum_{i\notin \cB}(\ba_i^T\bh)^2 \bone_{\cT_{1}^i} -\frac{1}{m}\sum_{i\in \cB}|\ba_i^T\bz+\ba_i^T\bx|\cdot|\ba_i^T\bh| \bone_{\cT_2^i}.
\label{eq:RWF_strongconvex2terms}
\end{flalign}
Under the condition $i \notin \cB$, we have $\cT_{1}^i= \{\left|\ba_i^T\bh\right|\le 0.5 \alpha'_h \|\bh\| \}$. Under the condition $i \in \cB$, we have $\cT_{2}^i= \{\left|\ba_i^T\bx+\ba_i^T\bz\right|\le 0.8 \alpha'_h \|\bh\| \}$. For convenience, we introduce two parameters $\gamma_1=0.5\alpha'_h$ and $\gamma_2=0.8\alpha'_h$.

We next bound the two terms in \eqref{eq:RWF_strongconvex2terms} respectively. For the first term, because of the inclusion $\cB\subseteq\{i: |\ba_i^T\bx|<|\ba_i^T\bh|\}$, we have 
\begin{flalign*}
\frac{1}{m}\sum_{i\notin \cB}(\ba_i^T\bh)^2 \bone_{\cT_{1}^i}&=\frac{1}{m}\sum_{i\notin \cB}(\ba_i^T\bh)^2 \bone_{\{|\ba_i^T\bh|\le \gamma_1\|\bh\|\}}\\
&\ge \frac{1}{m} \sum_{i=1}^m (\ba_i^T\bh)^2 \bone_{\{|\ba_i^T\bh|\le \gamma_1\|\bh\|\}}\bone_{\{|\ba_i^T\bx|\ge |\ba_i^T\bh|\}}\\
&\ge \frac{1}{m} \sum_{i=1}^m (\ba_i^T\bh)^2 \bone_{\{|\ba_i^T\bh|\le \gamma_1\|\bh\|\}}\bone_{\{|\ba_i^T\bx|\ge \gamma_1\|\bh\|\}}.
\end{flalign*}
A simpler version of Lemma 4 in \cite{chen2015solving} gives that if $m>c_0 n$, with probability at least $1-c_1\exp(-c_2 m\epsilon^2)$
\begin{flalign}
\frac{1}{m} \sum_{i=1}^m (\ba_i^T\bh)^2 \bone_{\{|\ba_i^T\bh|\le \gamma_1\|\bh\|\}}\bone_{\{|\ba_i^T\bx|\ge \gamma_1\|\bh\|\}}\ge (1-\zeta'_1-\zeta'_2-\epsilon)\|\bh\|^2 \label{eq:proplower1RWF}
\end{flalign}
holds for all $\bh\in \bbR^n$, where $\zeta'_1:=1-\min\left\{\mE\left[\xi^2\bone_{\{\xi\ge \sqrt{1.01}\gamma_1\frac{\|\bh\|}{\|\bx\|}\}}\right], \mE\left[\bone_{\{\xi\ge \sqrt{1.01}\gamma_1\frac{\|\bh\|}{\|\bx\|}\}}\right]\right\}$ and $\zeta'_2:=\mE\left[\xi^2\bone_{\{|\xi|>\sqrt{0.99}\gamma_1\}}\right]$ for $\xi\sim \cN(0,1)$.

For the second term, we have 
\begin{flalign}
\frac{1}{m}\sum_{i\in \cB}|\ba_i^T\bz+\ba_i^T\bx|\cdot|\ba_i^T\bh| \bone_{\cT_2^i}&\le\gamma_2\|\bh\|\frac{1}{m}\sum_{i\in \cB}|\ba_i^T\bh|\le\gamma_2\|\bh\|\frac{1}{m}\sum_{i=1}^m|\ba_i^T\bh|\bone_{\{|\ba_i^T\bx|<|\ba_i^T\bh|\}},\label{eq:proplower2RWF}
\end{flalign}
where the second inequality is due to the inclusion property $\cB\subseteq \{i: |\ba_i^T\bx|<|\ba_i^T\bh|\}$.
\begin{lemma} \label{lem:RWF_RCupper}
	For any $\epsilon>0$, if $m>c_0 n \epsilon^{-2}\log \epsilon^{-1}$, then with probability at least $1-C \exp(-c_1 \epsilon^2 m)$,
	\begin{flalign}
	\frac{1}{m}\sum_{i=1}^m |\ba_i^T\bh|\cdot \bone_{\{|\ba_i^T\bx|<|\ba_i^T\bh|\}}\le \left(0.12+\epsilon\right)\|\bh\| \label{eq:lemma2upper}
	\end{flalign}
	holds for all non-zero vectors $\bx, \bh\in \bbR^n$ satisfying $\|\bh\|\le \frac{1}{20}\|\bx\|$. Here, $c_0, c_1, C>0$ are some universal constants.
\end{lemma}
\begin{proof}
	See Appendix \ref{app:lem:RWF_RCupper}.
\end{proof} 
Thus, putting together \eqref{eq:proplower1RWF}, \eqref{eq:proplower2RWF} and Lemma \ref{lem:RWF_RCupper} concludes the proof.

\subsection{Proof of Proposition \ref{prop:localsmoothRWF} } \label{app:prop:localsmoothRWF}
This proof adapts the proof of Lemma 7 in \cite{chen2015solving}.
Denote $v_i:=\left(\ba_i^T\bz-|\ba_i^T\bx|\sgn(\ba_i^T\bz)\right)\bone_{\cT^i}$. Then 
\begin{flalign*}
\nabla\cR_{tr}(\bz)=\frac{1}{m}\bA^T\bv,
\end{flalign*}
where $\bA$ is a matrix with each row being $\ba_i^T$ and $\bv$ is a $m-$dimensional vector with each entry being $v_i$. Thus, for sufficiently large $m/n$, we have
\begin{flalign*}
\left\|\nabla \cR_{tr} (\bz)\right\|=\left\|\frac{1}{m}\bA^T\bv\right\|\le \frac{1}{m}\|\bA\|\cdot \|\bv\|\le (1+\delta)\frac{\|\bv\|}{\sqrt{m}}
\end{flalign*}
where the last inequality is due to the spectral norm bound $\|\bA\|\le \sqrt{m}(1+\delta)$ following from \cite[Theorem 5.32]{Vershynin2012}.

We next bound $\|\bv\|$. Let $\bv=\bv^{(1)}+\bv^{(2)}$, where $v_i^{(1)}=\ba_i^T\bh\bone_{\cT^i\backslash\cB^i}$ and $v_i^{(2)}=(\ba_i^T\bx+\ba_i^T\bz)\bone_{\cT^i\cap \cB^i}$, where $B^{i}:=\{(\ba_i^T\bx)(\ba_i^T\bz)<0\}$. By triangle inequality, we have
$\|\bv\|\le \|\bv^{(1)}\|+\|\bv^{(2)}\|$. Furthermore, given $m>c_0n$, by \cite[Lemma 3.1]{candes2013phaselift} with probability $1-\exp(-cm)$, we have
\begin{flalign*}
\frac{1}{m}\|\bv^{(1)}\|^2=\frac{1}{m}\sum_{i=1}^{m} (\ba_i^T\bh)^2\le(1+\delta)\|\bh\|^2.
\end{flalign*}
By Lemma \ref{lem:rareEvent}, we have with probability $1-C\exp(-c_1 m)$
\begin{flalign*}
\frac{1}{m}\|\bv^{(2)}\|^2\le (0.8 \alpha'_h \|\bh\|)^2\cdot\left(\frac{1}{m}\sum_{i=1}^m \bone_{\{(\ba_i^T\bx)(\ba_i^T\bz)<0\}}\right)\le 0.8\|\bh\|^2
\end{flalign*}
holds, where the last inequality is due to Lemma \ref{lem:rareEvent}.
Hence,
\begin{flalign*}
\frac{\|\bv\|}{\sqrt{m}}\le \left(\sqrt{1+\delta}+\sqrt{0.8}\right) \|\bh\|.
\end{flalign*}
This concludes the proof. 

\subsection{Proof of Lemma \ref{lem:rareEvent}}\label{app:lem:rareEvent}

Denote  correlation $\rho:=\frac{\bz^T\bx}{\|\bz\|\|\bx\|}$. Under the condition $\|\bz-\bx\|\le \frac{1}{11}\|\bx\|$, simple calculation yields $0.995<\rho\le1$. It suffices to show that the result holds with high probability for all $\bx$ and $\bz$ satisfying $\rho>0.995$. Since now  the claim is invariant with the norms of $\bx$ and $\bz$, we assume that both $\bx$ and $\bz$ have unit length without loss of generality.

We first establish the result for any fixed $\bx$ and $\bz$ and then develop a uniform bound by covering net argument in the end. We introduce a Lipschitz function to approximate the indicator function. Define
\begin{flalign*}
	\chi(t):=\begin{cases}
1, &\text{if } t<0;\\
-\frac{1}{\delta}\cdot t+1, & \text{if } 0\le t\le \delta;\\
0, & \text{else};
\end{cases}
\end{flalign*}
and then $\chi(t)$ is a Lipschitz function with Lipschitz constant $\frac{1}{\delta}$. In the following proof, we set $\delta=0.001$. We further have
\begin{flalign}
	\bone_{\{(\ba_i^T\bx)(\ba_i^T\bz)<0\}}\le \chi\left((\ba_i^T\bx)(\ba_i^T\bz)\right) \le  \bone_{\{(\ba_i^T\bx)(\ba_i^T\bz)<\delta\}}.
\end{flalign}
For convenience, we denote $b_{i}:=\ba_i^T\bx$ and $ \tilde{b}_i:=\ba_i^T\bz$. Then $(b_i, \tilde{b}_i)$ takes the jointly Gaussian distribution with mean $\mu=(0,0)^T$ and correlation $\rho$ ($b_{i}$ and $ \tilde{b}_i$ have unit variance). We next estimate the expectation of $\bone_{\{(\ba_i^T\bx)(\ba_i^T\bz)<\delta\}}$ as follows.
\begin{flalign}
	\mE [\bone_{\{(\ba_i^T\bx)(\ba_i^T\bz)<\delta\}}]=\bbP\left\{(\ba_i^T\bx)(\ba_i^T\bz)<\delta\right\} =\iint_{\tau_1\cdot \tau_2<\delta}  f(\tau_1,\tau_2) d\tau_1 d\tau_2 ,
	\end{flalign}
where  $f(\tau_1,\tau_2)$  is the density of the jointly Gaussian random variables $(b_i,\tilde{b}_i)$.  Note that $\mE [\bone_{\{(\ba_i^T\bx)(\ba_i^T\bz)<\delta\}}]$ is  decreasing on $\rho$ and for the case $\rho=0.995$ we calculate $\mE [\bone_{\{(\ba_i^T\bx)(\ba_i^T\bz)<\delta\}}]= 0.045$ numerically. This implies that 
$$\mE [\chi\left((\ba_i^T\bx)(\ba_i^T\bz)\right)]\le 0.045 $$
for $\delta=0.001$.
Furthermore, $\chi\left((\ba_i^T\bx)(\ba_i^T\bz)\right)$ for all $i$ are bounded and hence sub-Gaussian.
By Hoeffding type inequality for sub-Gaussian tail \cite{Vershynin2012}, we have
\begin{flalign}
	\bbP\left[\frac{1}{m}\sum_{i=1}^m\chi\left((\ba_i^T\bx)(\ba_i^T\bz)\right) >\left(0.045+\epsilon\right) \right]<\exp(-cm\epsilon^2),
\end{flalign}
for some universal constant $c$, as long as $\rho\ge 0.995$. 

We have proved so far that the claim holds for fixed $\bx$ and $\bz$. 
We next obtain a uniform bound over all $\bx$ and $\bz$ with unit length. 
Let $\mathcal{N}'_\epsilon$ be an $\epsilon$-net covering the unit sphere  in $\bbR^n$ and set
\begin{flalign}
	\mathcal{N}_\epsilon=\{(\bx_0,\bz_0): (\bx_0,\bz_0)\in \mathcal{N}'_\epsilon\times \mathcal{N}'_\epsilon\}.
\end{flalign}
One has cardinality bound (i.e., the upper bound on the covering number) $|\mathcal{N}_\epsilon|\le (1+2/\epsilon)^{2n}$. Then for any pair $(\bx,\bz)$ with $\|\bx\|=\|\bz\|=1$, there exists a pair $(\bx_0,\bz_0)\in \cN_{\epsilon}$ such that $\|\bx-\bx_0\|\le \epsilon$ and $\|\bz-\bz_0\|\le \epsilon$. Taking the union bound for all the points on the net, we claim that
\begin{flalign}
	\frac{1}{m}\sum_{i=1}^m \chi\left((\ba_i^T\bx_0)(\ba_i^T\bz_0)\right) \le 0.045+\epsilon, \quad \forall (\bx_0, \bz_0)\in \cN_{\epsilon} \label{eq:RWF_Proponnet}
\end{flalign}
holds with probability at least $1-(1+2/\epsilon)^{2n}\exp(-cm\epsilon^2)$.

Since $\chi(t)$ is Lipschitz with constant $1/\delta$, we have
\begin{flalign}
	\left|\chi\left((\ba_i^T\bx)(\ba_i^T\bz)\right)-\chi\left((\ba_i^T\bx_0)(\ba_i^T\bz_0)\right)\right|\le \frac{1}{\delta}\left|(\ba_i^T\bx)(\ba_i^T\bz)-(\ba_i^T\bx_0)(\ba_i^T\bz_0)\right|. \label{eq:RWF_lipschitzcon}
\end{flalign}
Moreover, by \cite[Lemma 1]{chen2015solving},
\begin{flalign}
	\frac{1}{m}\|\cA(\bM)\|_1\le c_2 \|\bM\|_F, \quad \quad \text{for all symmetric rank-2 matrices } \bM\in \bbR^{n\times n}, \label{eq:RWF_eventA}
\end{flalign}
holds with probability at least $1-C\exp(-c_1m)$ as long as $m>c_0 n$ for some constants $C, c_0, c_1, c_2>0$. Consequently, on the event that \eqref{eq:RWF_eventA} holds, we have
\begin{flalign*}
	&\left|\frac{1}{m}\sum_{i=1}^m \chi\left((\ba_i^T\bx)(\ba_i^T\bz)\right)-\frac{1}{m}\sum_{i=1}^m \chi\left((\ba_i^T\bx_0)(\ba_i^T\bz_0)\right)\right|\\
	&\le \frac{1}{m}\sum_{i=1}^m \left|\chi\left((\ba_i^T\bx)(\ba_i^T\bz)\right)-\chi\left((\ba_i^T\bx_0)(\ba_i^T\bz_0)\right)\right|\\
	&\le \frac{1}{\delta}\cdot \frac{1}{m}\|\cA(\bx\bz^T-\bx_0\bz_0^T)\|_1 \quad\quad \text{due to \eqref{eq:RWF_lipschitzcon}}\\
	&\le \frac{1}{\delta} \cdot c_2\|\bx\bz^T-\bx_0\bz_0^T\|_F \quad\quad \text{due to \eqref{eq:RWF_eventA}}\\
	&\le \frac{1}{\delta}\cdot c_2(\|\bx-\bx_0\|\cdot\|\bz\|+\|\bz-\bz_0\|\cdot\|\bx_0\|)\le 2c_3\epsilon/\delta.
\end{flalign*}

On the intersection of events that \eqref{eq:RWF_Proponnet} and \eqref{eq:RWF_eventA} hold, we have
\begin{flalign}
\frac{1}{m}\sum_{i=1}^m \chi\left((\ba_i^T\bx)(\ba_i^T\bz)\right) \le\left(0.045+\epsilon+2c_3\epsilon/\delta\right),
\end{flalign}
for all $\bx$ and $\bz$ with unit length and $\rho\ge 0.995$.  
 Since $\epsilon$ can be arbitrarily small, the proof is completed.

\subsection{Proof of Lemma \ref{lem:RWF_RCupper}} \label{app:lem:RWF_RCupper}

We first observe that for any $\gamma$,
\begin{flalign}
	\bone_{\{|\ba_i^T\bx|<|\ba_i^T\bh|\}}&\le \bone_{\{|\ba_i^T\bx|<\gamma\|\bx\|\}}+\bone_{\{|\ba_i^T\bh|\ge \gamma\|\bx\|\}}\le  \bone_{\{|\ba_i^T\bx|<\gamma\|\bx\|\}}+\bone_{\{|\ba_i^T\bh|\ge 20\gamma\|\bh\|\}} \label{eq:upperb}
\end{flalign}
where the last inequality is due to the assumption $\frac{\|\bh\|}{\|\bx\|}\le \frac{1}{20}$.

To establish the lemma, we set $\gamma=0.15$ and denote $\gamma':=20\gamma=3$.  We next respectively show  that
\begin{flalign}
	\frac{1}{m}\sum_{i=1}^m|\ba_i^T\bh|\bone_{\{|\ba_i^T\bx|<\gamma\|\bx\|\}}\le (0.11+\epsilon)\|\bh\|  \label{eq:summand1}
\end{flalign}
for all $\bx, \bh \in \bbR^n$,  and 
\begin{flalign}
	\frac{1}{m}\sum_{i=1}^m|\ba_i^T\bh|\bone_{\{|\ba_i^T\bh|>\gamma'\|\bh\|\}}\le (0.01+\epsilon)\|\bh\| \label{eq:summand2}
\end{flalign}
for all $\bh\in \bbR^n$.

We first prove \eqref{eq:summand1}. Without loss of generality, we assume that $\bh$ and $\bx$ have unit length. 
We introduce a Lipschitz function to approximate the indicator functions, which is defined as
\begin{flalign*}
	\chi_x(t):=\begin{cases}
1, &\text{if } |t|<\gamma;\\
\frac{1}{\delta}(\gamma-|t|)+1, & \text{if }  \gamma\le |t|\le \gamma+\delta;\\
0, & \text{else}.
\end{cases}
\end{flalign*}
 Then $\chi_x(t)$ is a Lipschitz function with constant $\frac{1}{\delta}$. We further have 
\begin{flalign}
	\bone_{\{|\ba_i^T\bx|<\gamma\}}\le \chi_x(\ba_i^T\bx) \le  \bone_{\{|\ba_i^T\bx|<\gamma+\delta\}}.
\end{flalign}
We first prove bounds for any fixed pair $\bh, \bx$, and then develop a uniform bound later on.

We next estimate the expectation of $|\ba_i^T\bh|\bone_{\{|\ba_i^T\bx|<\gamma+\delta\}}$, 
\begin{flalign}
	\mE [|\ba_i^T\bh|\bone_{\{|\ba_i^T\bx|<\gamma+\delta\}}]  =\iint_{-\infty}^\infty |\tau_1|\bone_{\{|\tau_2|<\gamma+\delta\}}\cdot f(\tau_1,\tau_2) d\tau_1 d\tau_2 ,
	\end{flalign}
where  $f(\tau_1,\tau_2)$  is the density of two jointly Gaussian random variables with correlation $\rho=\frac{\bh^T\bx}{\|\bh\|\|\bx\|}\neq \pm 1$. We then continue to derive
\begin{flalign}
	\mE[|\ba_i^T\bh|\bone_{\{|\ba_i^T\bx|<\gamma+\delta\}}] &= \frac{1}{2\pi\sqrt{1-\rho^2}}\int_{-\infty}^\infty  |\tau_1| \exp\left(-\frac{\tau_1^2}{2}\right) \cdot\int_{-(\gamma+\delta)}^{\gamma+\delta}\exp\left(-\frac{(\tau_2-\rho\tau_1)^2}{2(1-\rho^2)}\right) d\tau_2 d\tau_1 \label{eq:increasefun}\\
	&= \frac{1}{\sqrt{2}\pi}\int_{-\infty}^\infty  |\tau_1| \exp\left(-\frac{\tau_1^2}{2}\right) \cdot\int_{\frac{-\gamma-\delta-\rho\tau_1}{\sqrt{2(1-\rho^2)}}}^{\frac{\gamma+\delta-\rho\tau_1}{\sqrt{2(1-\rho^2)}}}\exp\left(-\tau^2\right) d\tau d\tau_1 \quad\quad\text{by changing variables}\nn\\
		&= \frac{1}{\sqrt{8\pi}}\int_{-\infty}^\infty |\tau_1| \exp\left(-\frac{\tau_1^2}{2}\right) \cdot \left(\erf\left(\frac{\gamma+\delta-\rho\tau_1}{\sqrt{2(1-\rho^2)}}\right)-\erf\left(\frac{-\gamma-\delta-\rho\tau_1}{\sqrt{2(1-\rho^2)}}\right) \right) d\tau_1 \label{eq:hardint}
\end{flalign}

For $|\rho|<1$, $\mE[|\ba_i^T\bh|\bone_{\{|\ba_i^T\bx|<\gamma+\delta\}}]$ is a continuous function of $\rho$. The last integral \eqref{eq:hardint} can be calculated numerically. Figure \ref{fig:hardintsupp} plots $\mE[|\ba_i^T\bh|\bone_{\{|\ba_i^T\bx|<\gamma+\delta\}}]$ for $\gamma=0.15$ and $\delta=0.01$ over $\rho\in (-1,1)$. Furthermore, \eqref{eq:increasefun} indicates that $\mE[|\ba_i^T\bh|\bone_{\{|\ba_i^T\bx|<\gamma+\delta\}}]$ is monotonically increasing with both $\theta$ and $\delta$. Thus, we obtain a universal bound
\begin{flalign}
\mE[|\ba_i^T\bh|\bone_{\{|\ba_i^T\bx|<\gamma+\delta\}}]\le 0.11\|\bh\| \quad \text{for } \gamma<0.15 \text{ and } \delta=0.01,
\end{flalign}
which further implies $\mE[|\ba_i^T\bh|\chi_x(\ba_i^T\bx)]\le 0.11\|\bh\|$ for $ \gamma<0.15$ and $\delta=0.01$.

\begin{figure}[th]
\centering 
\includegraphics[width=3.5in]{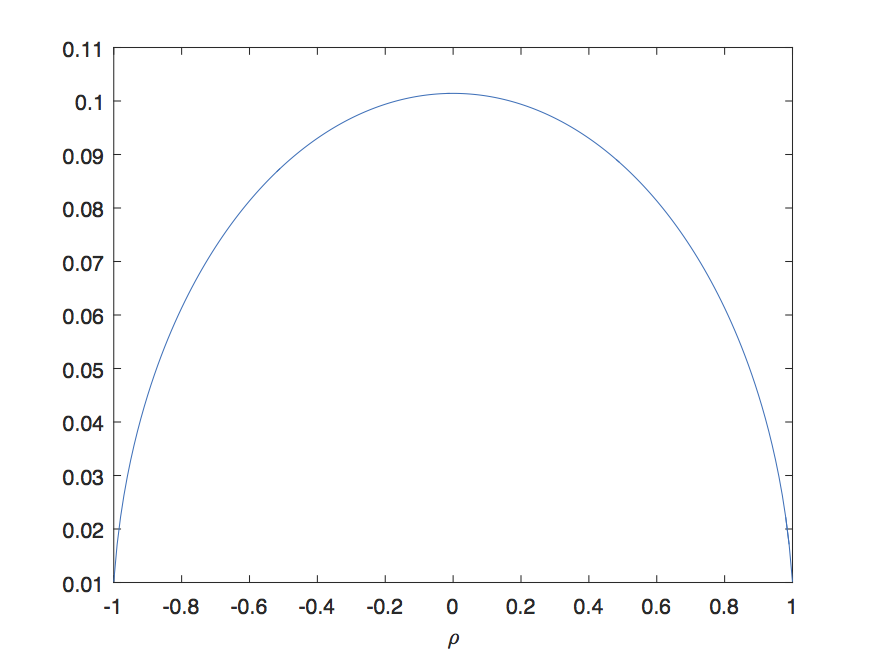}
\caption{$\mE[|\ba_i^T\bh|\bone_{\{|\ba_i^T\bx|<\gamma+\delta\}}]$ with respect to $\rho$}
\label{fig:hardintsupp}
\end{figure}

Furthermore, $|\ba_i^T\bh|\chi_x(\ba_i^T\bx)$'s are sub-Gaussian with sub-Gaussian norm $\cO(\|\bh\|)$.
By the Hoeffding type of sub-Gaussian tail bound \cite{Vershynin2012}, we have
\begin{flalign}
	\cP\left[	\frac{1}{m}\sum_{i=1}^m |\ba_i^T\bh|\chi_x(\ba_i^T\bx) >\left(0.11+\epsilon\right)\|\bh\| \right]<\exp(-cm\epsilon^2),
\end{flalign}
for some universal constant $c$.

We have proved so far that the claim holds for a fixed pair $\bh,\bx$. 
We next obtain a uniform bound over all $\bx$ and $\bh$ with unit length. 
Let $\mathcal{N}'_\epsilon$ be a $\epsilon$-net covering the unit sphere  in $\bbR^n$ and set
\begin{flalign*}
	\mathcal{N}_\epsilon=\{(\bx_0,\bh_0): (\bx_0,\bh_0)\in \mathcal{N}'_\epsilon\times \mathcal{N}'_\epsilon\}.
\end{flalign*}
One has cardinality bound (i.e., the upper bound on the covering number) $|\mathcal{N}_\epsilon|\le (1+2/\epsilon)^{2n}$. Then for any pair $(\bx,\bh)$ with $\|\bx\|=\|\bh\|=1$, there exists a pair $(\bx_0,\bh_0)\in \cN_{\epsilon}$ such that $\|\bx-\bx_0\|\le \epsilon$ and $\|\bh-\bh_0\|\le \epsilon$. Taking the union bound for all the points on the net, one can show
\begin{flalign}
	\frac{1}{m}\sum_{i=1}^m |\ba_i^T\bh_0|\chi_x\left(\ba_i^T\bx_0\right) \le 0.11+\epsilon, \quad \forall (\bx_0, \bh_0)\in \cN_{\epsilon} \label{eq:RWF_upperLemma_onnet}
\end{flalign}
holds with probability at least $1-(1+2/\epsilon)^{2n}\exp(-cm\epsilon^2)$.

Since $\chi_x(t)$ is Lipschitz with constant $1/\delta$, we have the following bound
\begin{flalign}
	\left|\chi_x\left(\ba_i^T\bx\right)-\chi_x\left(\ba_i^T\bx_0\right)\right|\le \frac{1}{\delta}\left|\ba_i^T(\bx-\bx_0)\right|. \label{eq:RWF_upperLemma_lipschitzcon}
\end{flalign}
Consequently, on the event that \eqref{eq:RWF_eventA} holds, we have
\begin{flalign*}
	&\left|\frac{1}{m}\sum_{i=1}^m |\ba_i^T\bh|\chi_x\left(\ba_i^T\bx\right)-\frac{1}{m}\sum_{i=1}^m |\ba_i^T\bh_0|\chi_x\left(\ba_i^T\bx_0\right)\right|\\
	&\le \frac{1}{m}\sum_{i=1}^m \left ||\ba_i^T\bh|\chi_x\left(\ba_i^T\bx\right)-|\ba_i^T\bh_0|\chi_x\left(\ba_i^T\bx_0\right)\right|\\
	&\le \frac{1}{m}\sum_{i=1}^m \left |\ba_i^T(\bh-\bh_0)\right|+\frac{1}{m}\sum_{i=1}^m \frac{1}{\delta}\left |\ba_i^T\bh_0\right|\cdot\left|\ba_i^T\bx-\ba_i^T\bx_0\right| \quad\quad \text{due to \eqref{eq:RWF_upperLemma_lipschitzcon}}\\
	&\le c'_2\|\bh-\bh_0\|+\frac{1}{\delta} \cdot c_2\|\bh_0(\bx-\bx_0)^T\|_F \quad\quad \text{due to \eqref{eq:RWF_eventA}}\\
	&\le c_3\epsilon/\delta.
\end{flalign*}

On the intersection of events that \eqref{eq:RWF_upperLemma_onnet} and \eqref{eq:RWF_eventA} hold, we have
\begin{flalign}
\frac{1}{m}\sum_{i=1}^m |\ba_i^T\bh|\chi_x\left(\ba_i^T\bx_0\right) \le\left(0.11+\epsilon+2c_3\epsilon/\delta\right), \label{eq:RWF_upperLemEq1}
\end{flalign}
for all $\bx$ and $\bh$ with unit length.

We next prove \eqref{eq:summand2}. Without loss of generality, we assume that $\bh$ has unit length. 
We introduce a Lipschitz function to approximate the indicator functions, which is defined as
\begin{flalign*}
	\chi_h(t):=\begin{cases}
|t|, &\text{if } |t|>\gamma';\\
\frac{1}{\delta}(|t|-\gamma')+\gamma', & \text{if }  \gamma'(1-\delta)\le |t|\le \gamma';\\
0, & \text{else}.
\end{cases}
\end{flalign*}
Then, $\chi_h(t)$ is a Lipschitz function with constant $\frac{1}{\delta}$. We further have 
\begin{flalign}
	|\ba_i^T\bh|\bone_{\{|\ba_i^T\bh|>\gamma'\|\bh\|\}}\le \chi_h(\ba_i^T\bh) \le  |\ba_i^T\bh|\bone_{\{|\ba_i^T\bh|>\gamma'(1-\delta)\|\bh\|\}}.
\end{flalign}
We first prove bounds for any fixed $\bh$, and then develop a uniform bound later on.

We next estimate the expectation of $|\ba_i^T\bh|\bone_{\{|\ba_i^T\bh|>\gamma'(1-\delta)\|\bh\|\}}$ as follows:
\begin{flalign}
	\mE [|\ba_i^T\bh|\bone_{\{|\ba_i^T\bh|>\gamma'(1-\delta)\|\bh\|\}}] &=\int_{-\infty}^\infty |\tau|\bone_{\{|\tau|>\gamma'(1-\delta)\}}\cdot f(\tau) d\tau, \nn\\
	&= 2\cdot \frac{1}{\sqrt{2\pi}}\int_{\gamma'(1-\delta)}^\infty  \tau \exp\left(-\frac{\tau^2}{2}\right)  d\tau\nn\\
	&= \sqrt{\frac{2}{\pi}}\exp(-\gamma'^2(1-\delta)^2/2)< 0.01 \quad \quad \text{for } \gamma'= 3, \delta= 0.01,
\end{flalign}
where $f(\tau)$ is the density of the standard Gaussian distrution.
We note that $\mE [|\ba_i^T\bh|\bone_{\{|\ba_i^T\bh|>\gamma'(1-\delta)\|\bh\|\}}] $ is monotonically increasing with $\delta$ and decreasing with $\gamma'$. Furthermore,  $\mE [\chi_h(\ba_i^T\bh)] \le 0.01\|\bh\|$ for $ \gamma'\ge 3$ and $\delta\le 0.01$.

Moreover, $\chi_h(\ba_i^T\bh)$ for all $i$ are sub-Gaussian with sub-Gaussian norm $\cO(\|\bh\|)$.
By the Hoeffding type sub-Gaussian tail bound \cite{Vershynin2012}, we have
\begin{flalign}
	\cP\left[	\frac{1}{m}\sum_{i=1}^m \chi_h(\ba_i^T\bh) >\left(0.01+\epsilon\right)\|\bh\| \right]<\exp(-cm\epsilon^2),
\end{flalign}
for some universal constant $c$.

We have proved so far that the claim holds for a fixed $\bh$. 
We next obtain a uniform bound over all  $\bh$ with unit length. 
Let $\mathcal{N}_\epsilon$ be an $\epsilon$-net covering the unit sphere  in $\bbR^n$. 
One has cardinality bound (i.e., the upper bound on the covering number) $|\mathcal{N}_\epsilon|\le (1+2/\epsilon)^{n}$. Then for any $\bh$ with unit length, there exists a  $\bh_0\in \cN_{\epsilon}$ such that $\|\bh-\bh_0\|\le \epsilon$. Taking the union bound for all the points on the net, one can show
\begin{flalign}
	\frac{1}{m}\sum_{i=1}^m \chi_h(\ba_i^T\bh_0) \le 0.01+\epsilon, \quad \forall \bh_0\in \cN_{\epsilon} \label{eq:RWF_upperLemma_onnet2}
\end{flalign}
holds with probability at least $1-(1+2/\epsilon)^{n}\exp(-cm\epsilon^2)$.

Consequently, we have
\begin{flalign*}
	&\left|\frac{1}{m}\sum_{i=1}^m \chi_h(\ba_i^T\bh)-\frac{1}{m}\sum_{i=1}^m \chi_h(\ba_i^T\bh_0)\right|\\
	&\le \frac{1}{m}\sum_{i=1}^m\left|\chi_h(\ba_i^T\bh)-\chi_h(\ba_i^T\bh_0)\right|\\
	&\le\frac{1}{\delta}\cdot \frac{1}{m}\sum_{i=1}^m \left |\ba_i^T(\bh-\bh_0)\right| \\
	&\le\frac{1}{\delta} c'_2\|\bh-\bh_0\|\le c_3\epsilon/\delta,
\end{flalign*}
where the second inequality is because  $\chi_h(t)$ is Lipschitz  continuous with constant $1/\delta$. 

On the intersection of events that \eqref{eq:RWF_upperLemma_onnet2} and \eqref{eq:RWF_eventA} hold, we have
\begin{flalign}
\frac{1}{m}\sum_{i=1}^m \chi_h(\ba_i^T\bh) \le\left(0.01+\epsilon+c_3\epsilon/\delta\right), \label{eq:RWF_upperLemEq2}
\end{flalign}
for all $\bh$ with unit length.  
 
Putting together \eqref{eq:RWF_upperLemEq1} and \eqref{eq:RWF_upperLemEq2}, and since $\epsilon$ can be arbitrarily small, the proof is completed.

\end{document}